\documentclass[10pt]{article}
\usepackage[utf8]{inputenc}
\usepackage{authblk}
\usepackage{amsfonts}
\usepackage{graphicx}
\usepackage{bm}
\usepackage{tikz}
\usepackage{enumitem}
\setlist[enumerate]{leftmargin=1.2cm}
\usepackage[authoryear]{natbib}
\usepackage{multirow}
\usepackage{amsmath,amsthm,amssymb}
\usepackage{algorithm,algorithmic}
\usepackage{hyperref}
\usepackage{makecell}
\hypersetup{
    colorlinks=true,
    citecolor=blue,
    urlcolor=blue,
    linkcolor=blue
}
\usepackage{rotating}
\usepackage{float}
\usepackage{dirtytalk}
\usepackage{lscape}
\usepackage{appendix}

\newtheorem{theorem}{Theorem}[section]
\newtheorem{lemma}{Lemma}[section]
\newtheorem{corollary}{Corollary}[theorem]
\newtheorem{prop}{Proposition}[section]

\theoremstyle{definition}
\newtheorem{remark}{\textbf{Remark}}[section]
\newtheorem{example}{Example}[section]

\usepackage[a4paper, total={6in, 9in}]{geometry}
\usepackage{subcaption}

\DeclareMathOperator*{\argmin}{arg\,min}
\DeclareMathOperator*{\vecc}{vec}

\title{Provably robust learning of regression neural networks \\using $\beta$-divergences}

\author[1]{\Large{Abhik Ghosh\footnote{Corresponding author: abhik.ghosh@isical.ac.in}}}
\author[1]{\Large{Suryasis Jana}}
\affil[1]{Indian Statistical Institute, Kolkata, India}

\begin{document}
\maketitle

\large
\begin{abstract}
Regression neural networks (NNs) are most commonly trained by minimizing  the mean squared prediction error,
which is highly sensitive to outliers and data contamination.
Existing robust training methods for regression NNs are often limited in scope and rely primarily on empirical validation, 
with only a few offering partial theoretical guarantees.
In this paper, we propose a new robust learning framework for regression NNs based on the $\beta$-divergence 
(also known as the density power divergence) which we call `rRNet'.  
It applies to a broad class of regression NNs, including models with non-smooth activation functions and error densities, 
and recovers the classical maximum likelihood learning as a special case. 
The rRNet is implemented via an alternating optimization scheme, 
for which we establish convergence guarantees to stationary points under mild, verifiable conditions.
The (local) robustness of rRNet is theoretically characterized through the influence functions of both the parameter estimates 
and the resulting  rRNet predictor, which are shown to be bounded for suitable choices of the tuning parameter $\beta$, 
depending on the error density. We further prove that rRNet attains the optimal 50\% asymptotic breakdown point at the assumed model 
for all $\beta\in(0, 1]$, providing a strong global robustness guarantee that is largely absent for existing NN learning methods.
Our theoretical results are complemented by simulation experiments and real-data analyses, 
illustrating practical advantages of rRNet over existing approaches in both function approximation problems 
and prediction tasks with noisy observations. 
\\

%This divergence depends on a tuning parameter $\beta$ controlling the balance between the robustness and efficiency of the resulting training procedure. 

\end{abstract}

% \tableofcontents

\section{Introduction}\label{intro}

Over the past decades, neural networks (NNs) have become central tools in machine learning and artificial intelligence, 
largely due to their excellent ability to learn a wide variety of complex data patterns. 
Classical results by \cite{hornik1989multilayer}, \cite{leshno1993multilayer} and others 
established that feedforward NNs with at least one hidden layer are indeed  `universal approximators', 
in the sense that they can approximate any continuous function on compact subsets of $\mathbb{R}^d$ arbitrarily well, 
given a sufficient number of neurons and suitable non-polynomial activation functions \citep{goodfellow2016deep}.
More recent works have shown that the ReLU (rectified linear unit) networks, with appropriate depth 
(number of hidden layers) and size (number of independent model parameters), can efficiently approximate fundamental arithmetic operations
(e.g., squaring, multiplication, etc.), and hence broad classes of analytic functions, including polynomials and trigonometric functions 
\citep{yarotsky2017error, schwab2019deep}. 
A comprehensive review of these approximation-theoretic results is provided in \cite{petersen2024mathematical}.
They highlight the flexibility of NN models for complex regression and function approximation, 
making them practically useful in settings where no prior structural information about 
the underlying input-output relationship is available \citep{bauer2019deep,schmidt2020nonparametric,shen2021robust}.

Despite their expressive power, practical performances of NNs depend critically on 
the learning algorithm used to estimate its model parameters. 
The most commonly used learning algorithms for training regression NNs  are based on the least squares (LS) estimation \citep{white1989learning},
which coincides with the maximum likelihood (ML) estimation under standard assumptions of 
additive independent and identically distributed (IID) Gaussian noise. 
However, these procedures are known to be adversely affected by even a small fraction of contaminating observations, 
leading to instability in both estimation and prediction.
With the increasing prevalence of noisy or contaminated data in modern applications,
this has necessitated growing interest in the development of robust learning algorithms for NN training.  
Early contributions in this direction include the least mean log square (LMLS) estimator of \cite{liano1996robust}, 
which assumes a heavy-tailed Cauchy distribution for the errors. While effective under strong tail assumptions, 
such approaches can suffer efficiency loss when the underlying noise distribution is light-tailed (e.g., Gaussian) 
except for a small fraction of outliers. 
Another widely used attempt to improve robustness of NN training without explicit distributional assumptions 
is to minimize the mean absolute error (MAE) loss, corresponding to the least absolute deviation regression \citep{goodfellow2016deep}.  
Although MAE is more resistant to heavy-tailed noises than the LS estimation, 
it provides limited robustness to leverage points and lacks strong robustness guarantees in complex nonlinear models. 
Trimming-based methods, such as the least trimmed squares (LTS) and least trimmed absolute deviations (LTA), 
were subsequently studied for NN regression by \cite{rusiecki2007robust} and \cite{rusiecki2013robust}, respectively. 
More generally, classical robust regression techniques based on M-estimation \citep{huber1981robust, hampel1986robust} 
have been adapted to NN training and empirically compared for multilayer perceptrons (MLPs) and radial basis function networks 
\citep{kalina2020regression, kalina2022combining}. 
Recent work has also explored alternative criteria such as $K$-divergence based estimators for NN regression \citep{sorek2024robust}. 

Nevertheless, the literature on noise-robust training of regression NNs remains comparatively limited 
and scattered relative to the severity of the problem. 
Most existing robust strategies for regression NNs remain highly context specific and architecture dependent, 
with theoretical analyses either absent or only partially available confined to narrow settings.
The adaptation of classical robustness measures,  such as influence functions and breakdown points, 
to the context of learning NN models remains limited both in scope and rigor, particularly for models with non-smooth architectures 
\citep[see, e.g.,][]{koh2017understanding,basu2020influence,bae2022if}.
In particular, convex M-estimator based learning procedures, including MAE-based training, 
have zero asymptotic breakdown point in regression problems with leverage points, 
and therefore do not provide global robustness guarantees \citep{werner2024global}.
Although trimming-based approaches can achieve higher breakdown points, their efficiency–robustness trade-offs varies across problems 
and a comprehensive theoretical treatment for general NN architectures is still lacking.

In this article, we develop a new robust learning framework for general regression NNs 
with rigorous theoretical guarantees of both local  and global robustness. 
For this purpose, we propose to use the popular \textit{$\beta$-divergences}, 
originally introduced by \cite{basu1998robust} under the name \textit{density power divergence} (DPD). 
Owing to its favorable properties and connections to information-theoretic and statistical principles, 
this divergence has  been extensively adopted in robust inference and machine learning \citep{basu2011statistical,basu2026}.
%it is commonly referred to as the $\beta$-divergence in ML  and as the DPD in robust statistical inference. 
Following the original terminology, we refer to it as the DPD throughout this paper, 
while denoting its tuning parameter by $\beta\in[0,1]$ (see Eq.~\eqref{dpd-def}).
The DPD is successfully applied in a wide range of settings,  with well-established asymptotic and robustness properties, 
including  regression problems \citep{ghosh2013robust,ghosh2016robust,jana2024robust},
% generalized linear models \citep{ghosh2016robust}, nonlinear regression \citep{jana2024robust}, as well as other machine learning problems  such as 
robust singular value decomposition \citep{roy2021rsvddpd} and  principal component analysis \citep{roy2024robust}. 
These make the DPD a natural statistically motivated candidate for constructing provably robust learning algorithms for NN models. 
A brief background on the DPD based estimators is provided in Appendix~\ref{bg} for completeness.

Here we focus on regression NNs with scalar outputs (formally defined in Section~\ref{RegNN-Setup}) for both  
function approximation and prediction of continuous responses under data possible contamination. 
Viewing general NN models with continuous outputs and additive errors as nonlinear regression models, 
where nonlinearity is introduced through activation functions in its hidden layers, 
we define minimum DPD estimators (MDPDEs) for the NN model parameters (weights and biases) by extending the idea from \cite{jana2024robust}. 
The resulting learning framework, which we formally define in Section \ref{rRnet} and refer to as the `{rRNet}', 
provides a statistically grounded alternative to LS/ML-based learning, explicitly designed to mitigate the adversarial effects of outliers 
and contaminated observations by suitably down-weighting their contributions in the estimation.
Importantly, our proposal is not restricted to smooth architectures or specific noise models, 
and remains applicable to most modern networks employing non-smooth activations, such as ReLU, 
and general (possibly non-smooth) log-concave error densities.
%Our proposal can accommodate a wide class of regression NNs, including possibly non-smooth activation functions (e.g., ReLU) and error densities.
%substantially extending existing robust NN learning approaches that are typically restricted to specific architectures or losses.
It include the classical ML based learning as a special case at $\beta=0$, 
while providing its robust generalizations for  $\beta>0$.

Although it may appear conceptually straightforward to extend the definition of MDPDEs from the existing literature to NN settings, 
their computation and rigorous theoretical analysis in general regression NNs are highly nontrivial, 
requiring substantial technical developments to accommodate non-smooth, non-convex and possibly overparameterized NN architectures
with inherent identifiability issues, which together constitute the core contributions of this paper.  
We propose a computationally tractable alternating optimization algorithm for estimating both NN parameters and the error scale in rRNet,
which are subsequently used to produce robust predictions from the fitted networks. Under mild verifiable conditions, 
we establish its monotone descent and convergence to stationary points for both smooth and non-smooth settings via Clarke subdifferential calculus.
Theoretical robustness of our proposal is guaranteed rigorously through both the influence function (IF) analysis of the rRNet functionals 
and their asymptotic breakdown point against most practically relevant contaminating distributions. 
We show that, for all $\beta >0$, the population IFs of both the parameter estimators and the induced regression predictor are bounded, 
establishing local B-robustness of rRNet.  These influence analyses have also been suitably extended for non-smooth cases 
through the smoothing-function approach, with an application to ReLU networks.
Moreover, the proposed rRNet is shown to achieve the maximum possible asymptotic breakdown point (50\%) at the assumed model for all $\beta \in (0, 1]$,
providing a strong global robustness guarantee, which is largely absent in the existing NN literature.
In contrast, we formally demonstrate the unbounded IFs of the classical ML based training (at $\beta=0$) under light-tailed noises
and zero asymptotic breakdown point of the resulting predictions, thereby quantifying its intrinsic non-robustness; 
these also cover the case of standard LS based learning as a special case with Gaussian error distribution.
Through extensive simulation studies on function approximation tasks and applications to real-world regression datasets 
with possibly contaminated observations, we demonstrate that rRNet consistently improves predictive stability of the regression NNs  
compared with existing robust NN training methods in most studied scenarios. 
These results demonstrate that the proved (asymptotic) theoretical advantages indeed translate into improved stability and predictive performance 
of rRNet in practical finite-sample applications.

The rest of the paper is organized as follows. 
Section \ref{proposed} introduces the model framework and the proposed rRNet, along with necessary notation and assumptions. 
Algorithmic convergence of rRNet is then established in Section \ref{optimization}.
Theoretical robustness guarantees are rigorously proved in Sections \ref{robustness} and \ref{BP}, 
covering the IFs and the asymptotic breakdown derivation, respectively. 
Results from our simulation studies and real data applications are presented in Section \ref{empirical},
where we also discuss a concrete implementation of rRNet, with the Gaussian noise assumption, that has been used in our empirical studies. 
Finally, some concluding remarks are given in Section \ref{conclusion}, 
with the broad scope and flexibility of rRNet framework being discussed in Appendix \ref{APP:assumptionms}. 
All technical proofs and additional empirical results are deferred to Appendices \ref{pf}-\ref{add-results}.

\section{The rRNet: Robust learning of regression neural networks} \label{proposed}

\subsection{Regression NN: Model setup and notation} \label{RegNN-Setup}

Let us consider a set of training samples $\mathcal{D}_n = \{(y_i, \bm{x}_i): i=1,\hdots,n\}$ of size $n$, 
where $y_i$ is a real-valued response (outcome) and $\bm{x}_i=(x_{i1},\hdots,x_{ip})^\top \in \mathcal{X} \subseteq \mathbb{R}^p$ 
is the associated covariate vector (input feature) for each $i$. We model the relationship between the outcome  
and the input covariates through a NN function $\mu(\boldsymbol{x}, \boldsymbol{\theta}) : \mathcal{X} \times \Theta \mapsto \mathbb{R}$, 
where $\boldsymbol{\theta}\in\Theta \subseteq\mathbb{R}^d$ is the $d$-dimensional vector of NN model parameters 
(often weights and biases) and $\Theta$ is the allowed parameter space. 

A commonly used regression NN model architecture is the fully connected MLP with $L$ hidden layers, 
formally defined through the recursive relations:
\begin{eqnarray} \label{mlp}
\left. \begin{array}{rl}
\mu(\bm{\bm{x},\theta}) =& 
%(1 ~\phi_L(\bm{z}_L^\top))\bm{w}^{out}, ~~~~
w_0^{out} + \sum\limits_{k=1}^{K_L} w_k^{out} \phi_L({z}_{Lk}),
\\[1em]
~~~\bm{z}_l  =& 
\bm{W}_l\phi_{l-1}(\bm{z}_{l-1}) + \bm{b}_l, 
~~l=1, \ldots, L,~~ 
%\\[.3em]
%~~~~~~~~~~~&\mbox{ with } ~
\bm{z}_0 = \phi_0(\bm{z}_0) = \boldsymbol{x}, 
\end{array}\right\}
\end{eqnarray}
where  $\bm{z}_l = (z_{l1}, \ldots, z_{lK_l})^\top$,  $\bm{W}_l = ((w_{jk}^{(l)}))$, and $\bm{b}_l = (b_{l1}, \ldots, b_{lK_l})^\top$, 
with	

$$
\begin{array}{rl}
	K_l, ~\phi_l :& \mbox{the number of nodes and activation function in the $l$-th hidden layer}, 
%	\\
%	\phi_l = & \mbox{activation function in the $l$-th hidden layer}, 
	\\
	z_{lk}, b_{lk} : & \mbox{net input and bias term to the $k$-th node in the $l$-th hidden layer, }
	\\
	w_{jk}^{(l)} :& \mbox{weight form the $j$-th node of $(l-1)$-th layer to the $k$-th node of $l$-th layer,}
%	\\
%	b_{lj} = & \mbox{bias to the $j$-th node in the $l$-th hidden layer, }
\end{array}
$$
for $j=1, \ldots, K_{l-1}$, $k=1, \ldots, K_l$, at $l=1, \ldots, L$,
and $\bm{w}^{out} = (w_0^{out}, w_1^{out},\ldots,w_{K_L}^{out})^\top$ be the vector of weights 
from the last ($L$-th) hidden layer to the output layer. 
Here, the activation functions ($\phi_l$) are assumed to operate element-wise on input vectors ($\bm{z}_{l-1}$). 
The full parameter vector is 
$\bm{\theta} = (\vecc^\top(\bm{W}_1),\bm{b}_1^\top, \ldots,\vecc^\top(\bm{W}_L),\bm{b}_{L}^\top, \bm{w}^{out\top})^\top$
with $\vecc(\bm{A})$ denoting the $mn$-vector constructed from the entries of an $m\times n$ matrix $\bm{A}$ stacked column-wise, 
and the resulting parameter dimension is 
$d = (p+1)K_1 + (K_1+1)K_2 + \hdots + (K_{L-1} + 1)K_L + (K_L + 1)$. 
The structure of such an example MLP network with $L=2$ hidden layers is illustrated in Figure \ref{fig:mlp}.

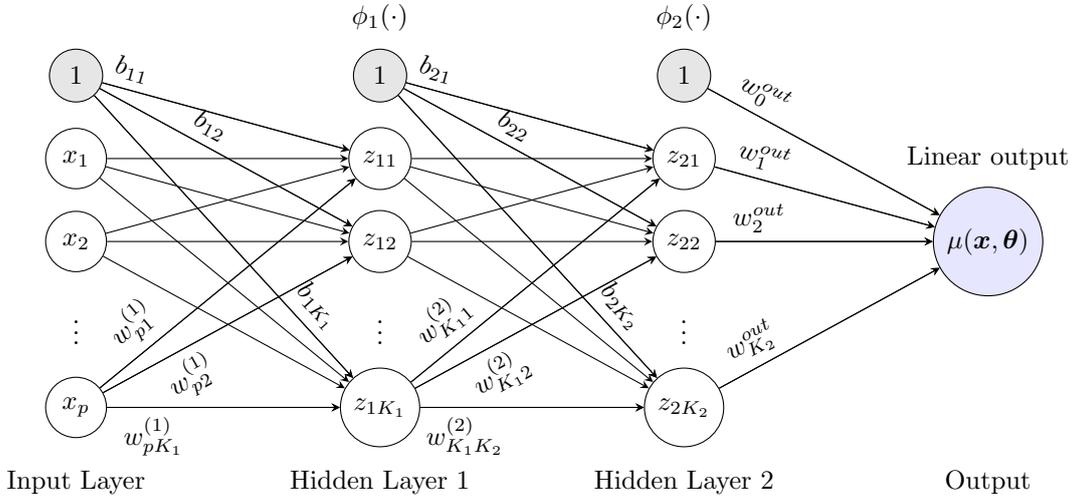
\begin{figure}[!h]
\centering
\begin{tikzpicture}[x=2cm, y=1.1cm, >=stealth]

%---------------- Input Layer ----------------%
\node at (0, -0.9) {Input Layer};
\node[draw, circle, minimum size=0.8cm] (x1) at (0,3) {$x_{1}$};
\node[draw, circle, minimum size=0.8cm] (x2) at (0,2) {$x_{2}$};
\node at (0,1) {$\vdots$};
\node[draw, circle, minimum size=0.8cm] (xp) at (0,0) {$x_{p}$};
% bias input
\node[draw, circle, minimum size=0.7cm, fill=gray!20] (b0) at (0,4) {$1$};

%---------------- Hidden Layer 1 ----------------%
\node at (2, -0.9) {Hidden Layer 1};
\node[draw, circle, minimum size=0.8cm] (h11) at (2,3) {$z_{11}$};
\node[draw, circle, minimum size=0.8cm] (h12) at (2,2) {$z_{12}$};
\node at (2,1) {$\vdots$};
\node[draw, circle, minimum size=0.8cm] (h1K) at (2,0) {$z_{1K_1}$};
% bias 1
\node[draw, circle, minimum size=0.7cm, fill=gray!20] (b1) at (2,4) {$1$};

%---------------- Hidden Layer 2 ----------------%
\node at (4, -.9) {Hidden Layer 2};
\node[draw, circle, minimum size=0.8cm] (h21) at (4,3) {$z_{21}$};
\node[draw, circle, minimum size=0.8cm] (h22) at (4,2) {$z_{22}$};
\node at (4,1) {$\vdots$};
\node[draw, circle, minimum size=0.8cm] (h2K) at (4,0) {$z_{2K_2}$};
% bias 2
\node[draw, circle, minimum size=0.7cm, fill=gray!20] (b2) at (4,4) {$1$};

%---------------- Hidden Layer L ----------------%
% \node at (6, -0.7) {Hidden Layer $L$};
% \node[draw, circle, minimum size=0.8cm] (hL1) at (6,3) {$z^{(L)}_{i1}$};
% \node[draw, circle, minimum size=0.8cm] (hL2) at (6,2) {$z^{(L)}_{i2}$};
% \node at (6,1) {$\vdots$};
% \node[draw, circle, minimum size=0.8cm] (hLK) at (6,0) {$z^{(L)}_{iK_L}$};

% % bias L
% \node[draw, circle, minimum size=0.7cm, fill=gray!20] (bL) at (6,4) {$1$};

%---------------- Output Layer ----------------%
\node at (6, -0.9) {Output};
\node[draw, circle, minimum size=1cm, fill=blue!10] (yhat) at (6,2) {$\mu(\bm{x},\bm{\theta})$};

%---------------- Connections ----------------%
% Input → Layer 1
\foreach \a in {x1, x2, xp, b0}{
    \foreach \b in {h11, h12, h1K}{
        \draw[->] (\a) -- (\b);
    }
}
% Layer 1 → Layer 2
\foreach \a in {h11, h12, h1K, b1}{
    \foreach \b in {h21, h22, h2K}{
        \draw[->] (\a) -- (\b);
    }
}
% % Layer 2 → Layer L
% \foreach \a in {h21, h22, h2K, b2}{
%     \foreach \b in {hL1, hL2, hLK}{
%         \draw[->] (\a) -- (\b);
%     }
% }
% Layer L → Output
\foreach \a in {h21, h22, h2K, b2}{
    \draw[->] (\a) -- (yhat);
}

% Labels describing activations
\node at (2,4.7) {$\phi_1(\cdot)$};
\node at (4,4.7) {$\phi_2(\cdot)$};
% \node at (6,4.7) {$\phi_L(\cdot)$};
\node at (6,3) {Linear output};

% --------------- Edge weights ---------------
\draw[->](b0) -- node[pos=0.1,above,sloped] {$b_{11}$} (h11);
\draw[->](b0) -- node[pos=0.4,above,sloped] {$b_{12}$} (h12);
\draw[->](b0) -- node[pos=0.8,above,sloped] {$b_{1K_1}$} (h1K);
\draw[->] (xp) -- node[pos=0.2,above,sloped] {$w^{(1)}_{p1}$} (h11);
\draw[->] (xp) -- node[pos=0.3,below,sloped] {$w^{(1)}_{p2}$} (h12);
\draw[->] (xp) -- node[pos=0.2,below,sloped] {$w^{(1)}_{pK_1}$} (h1K);

\draw[->](b1) -- node[pos=0.1,above,sloped] {$b_{21}$} (h21);
\draw[->](b1) -- node[pos=0.4,above,sloped] {$b_{22}$} (h22);
\draw[->](b1) -- node[pos=0.8,above,sloped] {$b_{2K_2}$} (h2K);
\draw[->](h1K) -- node[pos=0.2,above,sloped] {$w^{(2)}_{K_1 1}$} (h21);
\draw[->](h1K) -- node[pos=0.3,below,sloped] {$w^{(2)}_{K_1 2}$} (h22);
\draw[->](h1K) -- node[pos=0.2,below,sloped] {$w^{(2)}_{K_1 K_2}$} (h2K);

\draw[->] (b2)  -- node[pos=0.2,above,sloped] {$w^{out}_{0}$} (yhat);
\draw[->] (h21) -- node[pos=0.2,above,sloped] {$w^{out}_{1}$} (yhat);
\draw[->] (h22) -- node[pos=0.2,above,sloped] {$w^{out}_{2}$} (yhat);
\draw[->] (h2K) -- node[pos=0.2,above,sloped] {$w^{out}_{K_2}$} (yhat);

\end{tikzpicture}
\caption{The fully connected MLP with 2 hidden layers with the respective number of hidden nodes being $K_1, K_2$, 
	and the activation functions being $\phi_1, \phi_2$. The output layer has linear activation.}
\label{fig:mlp}
\end{figure}

Although the MLP serves as a canonical example, our theoretical developments apply to 
general regression NN functions satisfying the following minimal assumptions in (N0)--(N1).
The former is necessary to establish a well-defined population-level target for our proposed robust learning algorithm,
while the latter ensures its tractability over successive iterations. 
These are satisfied by a broad class of NN architectures including suitable MLPs; 
see Appendix \ref{APP:NN_assumptionms} for related discussions and examples. 

\begin{itemize}
	\item[(N0)] The regression NN function $\bm{x} \mapsto \mu(\bm{x}, \cdot)$ is measurable for all $\bm{\theta}\in\Theta$, 
	and the parameterization $\boldsymbol{\theta} \mapsto \mu(\cdot, \boldsymbol{\theta})$ 
	is identifiable in $\boldsymbol{\theta}$ up to the usual NN symmetries, i.e., 
	$$
%	\mbox{i.e., }~~~~~~~~~
\mu(\bm{x},\bm{\theta}_1) = \mu(\bm{x},\bm{\theta}_2)~~ a.s.~~~~~
	\Rightarrow ~~~\bm{\theta}_1 = g \cdot \bm{\theta}_2,~~~~~~~~~~~~~~~~~~~~~~~~
	$$
	for some $g$ in a known symmetry group $\mathcal{G}$ acting on the parameter space $\Theta$,
	which is measurable and non-empty. 

	\item[(N1)] For almost all $\bm{x}\in\mathcal{X}$, the NN model parametrization 
	$\boldsymbol{\theta} \mapsto \mu(\bm{x}, \boldsymbol{\theta})$ is locally Lipschitz continuous on $\Theta$.
%	and is either semi-algebraic or globally subanalytic. 
%	everywhere and differentiable almost everywhere in $\bm{\theta}\in\Theta$. 

\end{itemize}

In practice, the use of nonlinear activation functions (e.g., ReLU, sigmoid, etc.) in the NN model makes 
$\mu(\bm{x}, \bm{\theta})$ in \eqref{mlp} a nonlinear function in the parameters. 
Consequently, assuming random (IID) mean-zero additive noises, $\varepsilon_i$s, in the responses $y_i$s, 
we express the regression NN model as a general NLR model given by
\begin{equation}\label{mlp-mean-fn}
    y_i = \mu(\bm{x}_i, \bm{\theta}) + \varepsilon_i,\ i = 1,2,\ldots,n.
\end{equation}
Then, training this regression NN model \eqref{mlp-mean-fn} based on a given dataset $\mathcal{D}_n$ 
involves finding the optimal model weights (parametrized in $\bm{\theta}$) by minimizing a suitably chosen loss function. 
The most common and widely used loss function is the mean squared error (MSE) loss given by \citep{white1989learning} 
\begin{equation}\label{mse-loss}
\mathcal{L}_n(\bm{\theta}|\mathcal{D}_n) = \frac{1}{n} \sum_{i=1}^n (y_i - \mu(\bm{x}_i,\bm{\theta}))^2.
\end{equation}
This loss is often minimized using a suitable (stochastic) gradient based optimization method, with the gradients being computed via backpropagation, 
to obtain the resulting LS estimate (LSE) of $\bm{\theta}$. 
This LSE is then used to predict the response at any given input vector $\bm{x}$ 
and subsequently to estimate the noise variance \citep{goodfellow2016deep}.
Note that the LSEs are independent of the distribution of random noises, but coincide with the ML estimates for Gaussian error distribution. 
This standard NN training yields statistically optimal results under pure (clean) data
but is known to exhibit severe instability under even  a small amount  of data contamination. 

%
%As noted in the introduction, there have been attempts to robustly train the regression NN model in \eqref{mlp-mean-fn} 
%to obtain stable parameter estimates and predictions under different types of data contamination 
%by using alternative loss functions or heavy tailed error distributions. Here, in the following subsection,  
This work introduces a robust training algorithm for regression NNs using the DPD ($\beta$-divergences) 
which we call `rRNet'  --  the short form of \textit{(r)obust learning of (R)egression neural (Net)works}.
The associated DPD-loss function, defined formally in Eq.~\eqref{dpd-loss-gen} below, 
crucially depends on the model density of random errors to downweight the effects of contaminating observations,
characterized by their low probability of occurrence under the assumed model.  
In this respect, we make the following  assumption on the allowed error distributions within our rRNet framework.

\begin{itemize}
	\item[(A0)]  The random errors $\varepsilon_1, \ldots, \varepsilon_n$ are IID with density $\frac{1}{\sigma}f(\frac{\varepsilon}{\sigma})$,
	where $f$ is a continuous univariate density with mean $0$ and variance $1$ 
	that is strictly log-concave almost everywhere (a.e.) on its support $\mathbb{R}$. 
%	The error variance $\sigma^2$ is bounded away from zero, i.e., $\sigma \geq \sigma_0$ for some $\sigma_0>0$.
%	Additionally $\int f^{1+\beta} < \infty$ for the tuning parameter $\beta>0$ defining the rRNet.   
\end{itemize}

Note that, Assumption (A0) is satisfied by most practical choices of error densities under regression settings, 
such as Gaussian, Laplace, logistic, generalized Gaussian, Gumbel, etc.,
thereby ensuring the flexibility of  the proposed rRNet under all such possible noise/response distributions. 
For any such error density $f$, we then define its score function  $u$ as any measurable subgradient of $\log f$, 
i.e., $u(\varepsilon) \in \partial \log f(\varepsilon)$, the Clarke subdifferential of $\log f$ at $\varepsilon$
\citep{clarke1990optimization}. Under (A0), this score function $u$ is unique a.e. on $\mathbb{R}$, 
except only at kinks of $f$ which are measure-zero points. 
We refer to Appendix \ref{APP:NN_assumptionms} for further important properties of $f$ and $u$ under (A0).

Additionally, throughout the paper, we will use the following notation. 
We say a function is $\mathcal{C}^k$ ($k$-smooth) if it is $k$-times continuously differentiable, for any $k=1, 2, \ldots$. 
The first and second order partial derivatives of suitably differentiable functions 
with respect to any parameter vector $\bm{\omega}$ (which can be any part of $\bm{\theta}$, $\sigma$ or both)
are denoted by $\nabla_{\bm{\omega}}$ and $\nabla^2_{\bm{\omega}}$, respectively. For a possibly non-smooth function, 
its the Clarke subdifferential with respect to any parameter $\bm{\omega}$ is denoted by  $\partial_{\bm{\omega}}$. 
Also, $\bm{A}^+$ and $Ker(\bm{A})$ denotes the Moore–Penrose pseudo-inverse and the kernel (null-space) for any matrix $\bm{A}$,
while its usual inverse is denoted by $\bm{A}^{-1}$. Finally, we define 
\begin{eqnarray}
	&&C^{(\beta)}_{i,j} = \int s^i u^j(s) f^{1+\beta}(s) ds, ~~~~~ \mbox{ for }i, j = 0, 1, 2, \ldots, 
	~~\mbox{and any } \beta\geq 0.
	\nonumber
%	\\  
%	&&\bm{\dot{\mu}}_n(\bm{\theta}) = [\nabla_{\bm{\theta}}\mu(\bm{x}_1,\bm{\theta})~ 
%	\nabla_{\bm{\theta}}\mu(\bm{x}_2,\bm{\theta})~\cdots~ \nabla_{\bm{\theta}}\mu(\bm{x}_n,\bm{\theta})], ~~\mbox{ and }~~
%	\bm{J}_n(\bm{\theta}) = \frac{1}{n}\bm{\dot{\mu}}_n(\bm{\theta})^\top \bm{\dot{\mu}}_n(\bm{\theta}). \nonumber
\end{eqnarray}

\subsection{DPD-based robust learning via rRNet}\label{rRnet}

The DPD between two densities $g$ and $f$, defined with respect to a common dominating measure $\lambda$ 
(which is the Lebesgue measure here), is given by \citep{basu1998robust}
\begin{equation} \label{dpd-def}
    d_\beta(g,f) =
    \begin{cases}
        \int \left\{f^{1+\beta} - \left(1+\frac{1}{\beta}\right)f^\beta g + \frac{1}{\beta}g^{1+\beta}\right\} d\lambda, & \mbox{if}~~~ \beta>0,\\
        \int g \ln\left(\frac{g}{f}\right) d\lambda, & \mbox{if}~~~ \beta = 0,
    \end{cases}
\end{equation}
where $\beta\geq 0$ is a tuning parameter controlling the trade-off between robustness and efficiency. 
At $\beta=0$, we have $\lim_{\beta\downarrow 0}d_\beta(g,f) = d_0(g,f)$, which is the same as the classical Kullback-Leibler divergence (KLD). 
On the other hand, it gives the squared $L_2$ distance at $\beta=1$; see \cite{basu2011statistical,basu2026} for more details. 

Since the MLE is known to be a minimizer of the empirical KLD measure
between the data and the assumed model density, it is natural to seek for a robust extension of 
the MLE by minimizing any suitable alternative divergence measure instead of the KLD, 
and the DPD is shown to be very effective for this purpose in several statistical models \citep{basu2026}. 
Motivated by their success, our proposed learning algorithm \texttt{rRNet} estimates the NN model parameters $\boldsymbol{\theta}$
(along with the error scale $\sigma$) by minimizing the DPD-loss function given by
\begin{equation}\label{dpd-loss-gen}
\mathcal{L}_{n,\beta}(\bm{\eta}) =   \mathcal{L}_{n,\beta}(\bm{\eta}|\mathcal{D}_n) 
    = \frac{1}{n}\sum_{i=1}^n V_{\beta}(y_i,\bm{x}_i|\bm{\eta}), \mbox{ for all }\beta\geq 0,
\end{equation}
with respect to both $\bm{\theta}$ and $\sigma$ simultaneously, where $\boldsymbol{\eta} = (\bm{\theta}^\top, \sigma)^\top$ and 
\begin{equation*}
    V_{\beta}(y,\bm{x}|\bm{\eta}) = 
    \begin{cases}
        \frac{1}{\sigma^\beta}C^{(\beta)}_{0,0}  - \left(1 + \frac{1}{\beta}\right) \frac{1}{\sigma^\beta} 
        f^\beta\left(\frac{y - \mu(\bm{x},\bm{\theta})}{\sigma}\right) + \frac{1}{\beta}, & \mbox{if}~~ \beta>0,\\
        \ln \sigma - \ln f\left(\frac{y - \mu(\bm{x},\bm{\theta})}{\sigma}\right), & \mbox{if}~~ \beta=0.
    \end{cases}
\end{equation*}
The derivation of this DPD-loss function (\ref{dpd-loss-gen}) from (\ref{dpd-def})  follows directly from 
the definition of the MDPDE under a general non-homogeneous setup discussed in Appendix \ref{bg}.
We just need to note that, given the input features $\bm{x}_i$, each  $y_i$ is independent having model density 
$f_{i,\boldsymbol{\eta}}(y) = f_{\bm{\eta}}(y|\boldsymbol{x}_i) = \frac{1}{\sigma}f((y - \mu(\bm{x}_i,\bm{\theta}))/\sigma)$, 
for $i=1,\ldots, n$, under Assumption (A0).
Following the literature on DPD based inference, 
we refer to the resulting parameter estimate, a minimizer of \eqref{dpd-loss-gen} obtained under rRNet, 
as the MDPDE with tuning parameter $\beta\geq 0$ and denote it as 
$\widehat{\bm{\eta}}_{n,\beta} = (\widehat{\bm{\theta}}_{n,\beta}^\top, \widehat{\sigma}_{n,\beta})^\top$. 
The corresponding fitted regression function, $\widehat{\mu}_{n,\beta}(\bm{x}) = \mu(\bm{x}, \widehat{\bm{\theta}}_{n,\beta})$
at a given feature value $\bm{x}$, will be referred to as the `rRNet predictor',
while the whole trained NN model will be termed the `fitted rRNet'.
Throughout the paper, we will write both $\bm{\eta}$ and $(\bm{\theta}, \sigma)$ interchangeably 
for the arguments of the DPD-loss function as the context requires.

Note that, at $\beta=0$, the simplified DPD-loss function, obtained as $\beta\downarrow 0$, 
reduces to the negative log-likelihood. Thus, the MDPDE at $\beta=0$ is nothing but the MLE, 
which is also the LSE in the case of Gaussian error distribution
(see Example \ref{EX:Gaussian_error}).

\begin{example}[Gaussian error case]\label{EX:Gaussian_error}
If $f$ is the density of a standard normal distribution, i.e., 
when $\varepsilon_i$s are IID Gaussian with mean zero and variance $\sigma^2$, 
the DPD-loss function \eqref{dpd-loss-gen}, at any $\beta>0$,  turns out to be
    \begin{equation} \label{dpd-loss}
        \mathcal{L}_{n,\beta}(\bm{\eta}|\mathcal{D}_n) = \frac{1}{(\sigma\sqrt{2\pi})^\beta\sqrt{1+\beta}} - \frac{1+\beta}{\beta}\frac{1}{n(\sigma\sqrt{2\pi})^\beta}\sum_{i=1}^n e^{-\frac{\beta}{2\sigma^2}(y_i-\mu(\bm{x}_i, \bm{\theta}))^2} + \frac{1}{\beta}.
%         ~~ \mbox{for } \beta>0,
    \end{equation}
At $\beta=0$, it again coincides with the negative log-likelihood function given by 
$\mathcal{L}_{n,0}(\bm{\eta}|\mathcal{D}_n) = \ln(\sqrt{2\pi}\sigma)  + \frac{1}{2\sigma^2}\mathcal{L}_n(\bm{\theta}|\mathcal{D}_n)$;
minimizing this loss is thus equivalent the standard training of the regression NN using the MSE loss $\mathcal{L}_n$
and estimating $\sigma$ from the trained model. 
It may be noted that Gaussian error distribution satisfies the required Assumption (A0).
\qed
\end{example}

\begin{remark}\label{REM:1}
Although the DPD-loss function, and hence the resulting rRNet, is defined for all $\beta\geq 0$, 
and many of our theoretical results are proved for all $\beta>0$, 
the choice $\beta>1$ is never suggested in practice because of significant loss in efficiency under pure data.
See \cite{basu2011statistical,basu2026} for details on the role of $\beta$ in the MDPDEs computed under different parametric statistical models.  
Accordingly, we also restrict  ourselves to the choices $\beta \in[0,1]$ for all empirical studies here. 
%and also in some theoretical results.
\qed 
\end{remark}

% In all the numerical experiments presented in this paper, we have employed the above DPD-based NN learning algorithm with the assumption of Gaussian errors, i.e., used the loss function given in \eqref{dpd-loss} instead of the general form in \eqref{dpd-loss-gen}. Our implementation is available publicly at our GitHub repository\footnote{\url{https://github.com/Suryasis124/DPD-based-Robust-regression-neural-network-.git}}, which can be used by anyone to reproduce our empirical results and also to analyze their datasets robustly. 
Since it is not straightforward to minimize the DPD-loss with $\beta>0$ for a given NN architecture (not even for Gaussian error), 
we propose to solve this problem via an alternating optimization strategy --- 
minimizing $\mathcal{L}_{n,\beta}(\bm{\theta},\sigma|\mathcal{D}_n)$ cyclically with respect to $\bm{\theta}\in\Theta$ for a given $\sigma$ 
and with respect to $\sigma$ for a given  $\bm{\theta}$. 
For theoretical stability of the DPD-loss and its optimization dynamics, throughout the paper, 
we assume the error variance $\sigma^2$ to be bounded away from zero, i.e., $\sigma \geq \sigma_0$ for some $\sigma_0>0$.
In practice, $\sigma_0$ may be chosen to be a small positive number based on the underlying context. 
Accordingly the feasible parameter space for $\sigma$ is $[\sigma_0, \infty)$
while the same ($\Theta$) for $\bm{\theta}$ should be chosen based on the assumed NN architecture. 
Algorithm \ref{Alg-dpd-nn} below presents this basic schema pf the rRNet,
which is further detailed and theoretically justified in Section \ref{optimization}.
%a possible implementation used in our numerical experiments,  for the Gaussian error case is provided in Section \ref{implementation}.

\begin{center}
\begin{minipage}{0.85\textwidth} 
\begin{algorithm}[H]
\caption{The \texttt{rRNet} (basic schema)}
\label{Alg-dpd-nn}
\begin{algorithmic}[1]
	\makeatletter
	\setcounter{ALC@line}{-1}
	\makeatother
	
    \STATE \textbf{Input:} Training data $\mathcal{D}_n$, tuning parameter $\beta\geq 0$ and a tolerance limit $\varepsilon>0$. 
    \STATE \textit{Initialization:} Set $k = 0$. Fix suitable initial values $(\widehat{\bm{\theta}}^{(0)}, \widehat{\sigma}^{(0)})$. 
    \vspace{.6em}
    \item[] \hspace*{-1.2em} 	\textbf{do} \\[0.3em]
	\begin{ALC@g}
		\STATE Update the value of $\bm{\theta}$ as $\widehat{\bm{\theta}}^{(k+1)} = \argmin\limits_{\bm{\theta}\in\Theta} \mathcal{L}_{n,\beta}(\bm{\theta},\widehat{\sigma}^{(k)})$.
    	\STATE Update the value of $\sigma$ as $\widehat{\sigma}^{(k+1)} = \argmin\limits_{\sigma\geq\sigma_0} \mathcal{L}_{n,\beta}(\widehat{\bm{\theta}}^{(k+1)}, \sigma)$.
    %\item 4. If $\left| L_n^{\beta} \left( \widehat{\bm{\theta}}^{(k+1)}, \widehat{\sigma}^{(k+1)} \right) - L_n^{\beta} \left( \widehat{\bm{\theta}}^{(k)}, \widehat{\sigma}^{(k)} \right) \right| < \epsilon$ \\\\
    %\hspace{1cm} Stop;
    %\item ~~~~Else set $k \gets k+1$, go to step 2.
    	\STATE Set $k \gets k+1$. \\[0.5em] %, and repeat Steps 2 and 3 until the fitted model stabilizes (up to the $\varepsilon$-limit).
    \end{ALC@g}
    \item[] \hspace*{-1.2em} \textbf{while} the fitted model is not stabilized (up to the $\varepsilon$-limit).  \\[0.6em]
	\STATE \textbf{Return:} The rRNet estimates 
	$\widehat{\bm{\theta}}_{n,\beta} = \widehat{\bm{\theta}}^{(k+1)}$ and 	$\widehat{\sigma}_{n,\beta}= \widehat{\sigma}^{(k+1)}$.
\end{algorithmic}
\end{algorithm}
\end{minipage}
\end{center}

%\subsection{Statistical Target of the rRNet}
\subsection{Population-Level Target and Identifiability of rRNet}
\label{identifiability}

Before studying the optimization dynamics of the proposed rRNet,
we first establish that it indeed aims to estimate a meaningful and well-defined population-level target, 
%This is crucial because all subsequent robustness guarantees of rRNet are valid and interpretable 
%only when the underlying statistical target itself is well posed.
%This characterization of the population objective allows us to rigorously identify the statistical target of the rRNet, 
which  forms the basis for all subsequent optimization and robustness guarantees.

For this purpose, we note that the population-level DPD-loss functional for the regression NN model \eqref{mlp-mean-fn} is given by 
$\mathcal{L}_{\beta}^*(\bm{\eta}|G) = E_{G}\left[V_{\beta}(Y,\bm{X}|\bm{\eta})\right]$, for any $\beta\geq 0$,
where the expectation is taken with respect to the (true) joint distribution $G$ of the response variable ($Y$) 
and the feature vector ($\bm{X}$) that generated  the sample data $\mathcal{D}_n$.
Then, the associated minimum DPD functional (MDPDF) at a given $\beta\geq 0$, denoted by 
$\bm{T}_\beta(G) = (\bm{T}_{\beta}^{\bm{\theta}}(G)^\top, \bm{T}_{\beta}^\sigma(G))^\top$,
is defined as a minimizer of $\mathcal{L}_{\beta}^*(\bm{\theta},\sigma|G)$, 
and  the population-level rNet predictor   is given by  
$\mu_\beta^*(\boldsymbol{x}) =  \mu(\bm{x}, \bm{T}_\beta^{\bm{\theta}}(G))$ at $\boldsymbol{X}=\bm{x}$.

However, the above definitions are meaningful when features are assumed to be stochastic, and do not directly cover the fixed design cases. 
Since it is not essential to explicitly specify the distribution of features in our case, 
instead of $G$, we will work with the conditional joint distribution $\overline{G}_n$, given the observed feature values $\bm{x}_1, \ldots, \bm{x}_n$. 
We  can formally express $\overline{G}_n$ as  
$$
\overline{G}_n(y, \bm{x}) = \frac{1}{n}\sum_{i=1}^n \wedge_{\boldsymbol{x}_i}(\bm{x}) G_i(y), 
$$
where $\wedge_{\boldsymbol{x}_i}$ and $G_i$ denote the degenerate distribution function at the point $\bm{x}_i$ and  
the true conditional distribution of $Y$ given $\bm{X}=\bm{x}_i$ (under $G$), respectively, for each $i=1, \ldots, n$.
This construction unifies both fixed and random design settings:  
$\overline{G}_n(y, \bm{x})$ is the deterministic design-weighted distribution function under fixed design models
while it is a stochastic plug-in estimate of $G$ avoiding the specification of the distribution of $\bm{X}$ for random design cases. 
%Thus, it unifies both the set-ups allowing us to work only with the conditional distribution of the response given the covariates 
%even at the population level.  
Adopt this conditional framework,  we consider the population-level DPD-loss function evaluated at $\overline{G}_n$ as given by 
\begin{eqnarray}
\mathcal{L}_{n,\beta}^*(\bm{\eta}) = E_{\overline{G}_n}\left[V_{\beta}(Y,\bm{X}|\bm{\eta})\right] 
= \frac{1}{n} \sum_{i=1}^n E_{G_i}\left[V_{\beta}(Y,\bm{x}_i|\bm{\eta})\right],
\label{dpd-loss-pop}
\end{eqnarray}
and study the MDPDF of the model parameters at $\overline{G}_n$ defined as
\begin{eqnarray}
\bm{T}_\beta(\overline{G}_n) = (\bm{\theta}_{n,\beta}^{*\top}, \sigma_{n,\beta}^*)^\top 
= \arg\min_{(\bm{\theta}, \sigma)} \mathcal{L}_{n,\beta}^*(\bm{\theta},\sigma) 
= \arg\min_{\bm{\eta}} \frac{1}{n} \sum_{i=1}^n d_\beta(g_i, f_{i, \bm{\eta}}),
\label{EQ:MDPDF}
\end{eqnarray} 
where $g_i$ denotes the (true) data generating density of $G_i$ for each $i=1, \ldots, n$,
and we put $\bm{\theta}_{n,\beta}^* = \bm{T}_\beta^{\bm{\theta}}(\overline{G}_n)$
$\sigma_{n,\beta}^* = {T}_\beta^{\sigma}(\overline{G}_n)$ for simplicity of notation. 
Thus, the statistical target (population-level predictor) of our rRNet with tuning parameter $\beta\geq 0$ is given by
$\mu_{n,\beta}^*(\boldsymbol{x}) =  \mu(\bm{x}, \bm{\theta}_{n,\beta}^*)$.
It may be noted that, studying the proposed rRNet with respect to these estimates 
$(\bm{\theta}_{n,\beta}^*, \sigma_{n,\beta}^*)$ and the target $\mu_{n,\beta}^*$ is enough even under random design setups, 
since they  converge almost surely to $(\bm{T}_{\beta}^{\bm{\theta}}(G), \bm{T}_{\beta}^\sigma(G))$ and  $\mu_\beta^*$, respectively, 
under mild ergodicity condition on the distribution of $\bm{X}$.

Now, under Assumption (N0), the NN mapping  $\bm{\theta} \mapsto \mu(\cdot, \bm{\theta})$ is identifiable  up to a known symmetry group $\mathcal{G}$, 
and the family of model error densities is also identifiable in $\sigma$ under (A0). 
%These conditions are mild and are satisfied by a broad class of NN architectures 
%and error distributions used in practice (see Appendix \ref{APP:NN_assumptionms}). However, they 
Despite their generality and broad applicability (see Appendix \ref{APP:NN_assumptionms}), 
these assumptions are indeed sufficient to formalize the well-definedness of the population-level rRNet target, 
as presented in the following theorem.

\begin{theorem}[Identifiability of the rRNet objective]\label{THM:identifiability}
Suppose Assumptions (N0) and (A0) hold for a given $\beta \geq 0$.
Then, in either the fixed or random design case, the (conditional) population-level  DPD-loss  
$\mathcal{L}_{n,\beta}^*(\bm{\theta},\sigma)$, defined in \eqref{dpd-loss-pop}, satisfies
\begin{eqnarray*}
\mathcal{L}_{n,\beta}^*(\bm{\theta}_1,\sigma_1)  = \mathcal{L}_{n,\beta}^*(\bm{\theta}_2,\sigma_2) ~~~
\Rightarrow &&  \sigma_1 = \sigma_2, ~~~~ \mu(\bm{x}, \bm{\theta}_1) = \mu(\bm{x}, \bm{\theta}_2)~~~a.s., \\
~~~\mbox{ and hence }&& \bm{\theta}_1 = g \cdot \bm{\theta}_2 
~~\mbox{ for some } g \in \mathcal{G}.
\end{eqnarray*}
In particular, the population-level rRNet objective admits a unique minimizer in function space 
(the target rRNet predictor $\mu_{n,\beta}^*$), together with a unique MDPDF $\sigma_{n,\beta}^*$ of the error variance. 
However, the MDPDF $\bm{\theta}_{n,\beta}^*$ of the NN parameters $\boldsymbol{\theta}$ is unique 
only up to the symmetry group $\mathcal{G}$ of the network architecture.  
\end{theorem}

The proof of Theorem \ref{THM:identifiability} follows directly from the fact that the DPD is a genuine statistical divergence 
\citep{basu2011statistical}, along with Assumptions (A0) and (N0). 
%ensures identifiability of error distributions in $\sigma$. 
Consequently, the rRNet targets a uniquely defined regression function $\mu_{n, \beta}^*(\bm{x})$ and error scale, 
up to the unavoidable symmetry-induced non-identifiability of the NN parameters. 
Such symmetry groups $\mathcal{G}$ naturally arise from neuron permutations	and sign/scaling symmetries in multilayer NNs,
which are widely acknowledged in NN literature; see, e.g., \cite{goodfellow2016deep,ran2017parameter}
%With this population-level justification in place, we now turn to the optimization landscape and dynamics of rRNet learning.

In our subsequent robustness analysis of the rRNet predictor, we use a related concept of \textit{admissible feature domain},
characterizing input feature values whose parameter sensitivities (gradient directions) are entirely informed by the training data. 
Formally, we define the admissible feature domain of a give NN model $\mu$ (not necessarily smooth) 
based on the observed training data $\mathcal{D}_n$ as 
$$
\mathcal{A} = \left\{ \bm{x} \in \mathcal{X} :  \partial_{\bm{\theta}}\mu(\bm{x},\bm{\theta}) \subseteq 
\operatorname{Span}\left(\displaystyle\bigcup_{i=1}^n \partial_{\bm{\theta}}\mu(\bm{x}_i,\bm{\theta})\right)
%\mbox{Span}\{\nabla_{\bm{\theta}}\mu(\bm{x},\bm{\theta}), \ldots, \nabla_{\bm{\theta}}\mu(\bm{x}_n,\bm{\theta}) \}  
\right\} \subseteq\mathcal{X}.
$$
For 1-smooth NN models, this set $\mathcal{A}$ contains all feature values $\bm{x}\in \mathcal{X}$ for which 
$\nabla_{\bm{\theta}}\mu(\bm{x},\bm{\theta})$ belongs to the span of $\nabla_{\bm{\theta}}\mu(\bm{x}_1,\bm{\theta}), 
\ldots, \nabla_{\bm{\theta}}\mu(\bm{x}_n,\bm{\theta})$.
This concept is, in principle, related to classical parameter identifiability in statistical learning, 
where the ability to uniquely determine model parameters depends on the directions of 
sensitivity captured by the training data \citep{ran2017parameter}.

%\section{Optimization dynamics of the rRNet: Robustness and stability}\label{optimization}
\section{Optimization dynamics of the rRNet}\label{optimization}

%Having defined the proper population-level target for the rRNet, 
%let us now study its optimization dynamics when trained using Algorithm \ref{Alg-dpd-nn}. 
%Note that, updating $\sigma$ by minimizing the DPD-loss $\mathcal{L}_{n,\beta}$ in Step 3 is rather simpler,
%since the DPD-loss is locally convex around the (unique) target minimizer in $\sigma$, for any fixed $\bm{\theta}$, under Assumption (A0).
%But, the minimization of  $\mathcal{L}_{n,\beta}$ with respect to $\bm{\theta}$ in Step 2, even with fixed $\sigma$, 
%is nontrivial due to the non-convexity introduced through the assumed NN  architecture. 
%It is often not possible to find exact solution to this sub-optimization problem in Step 2,
%although the same is possible for Step 3.  

\subsection{Convergence guarantees of Algorithm \ref{Alg-dpd-nn}}

Alternating optimization provides a natural and widely used strategy for tackling complex, potentially non-convex optimization problems 
by decomposing them into a sequence of structured subproblems.
We first study the proposed alternating algorithm for the rRNet (Algorithm \ref{Alg-dpd-nn}) 
under the following sufficient conditions on the two inner sub-optimization problems (Steps 2--3);
possible implementations for these steps are studied subsequently in details. 
%Let us not restrict ourselves to everywhere differentiable loss-functions %$\mu(\cdot, \bm\theta)$ 
%and work with the framework of  Clarke subdifferential \citep{clarke1990optimization} 
%to accommodate all NN architectures satisfying (N0) and all error densities satisfying (A0).
%Accordingly, we make the following assumptions: 

\newpage
\begin{itemize}
	\item[(SO1)] The $\bm{\theta}$-update in Step 2 of Algorithm \ref{Alg-dpd-nn} satisfies 
	sufficient descent of the DPD-loss, i.e., there exists a constant $c >0$ such that 
\begin{eqnarray}
	\mathcal{L}_{n,\beta}\left(\widehat{\bm{\theta}}^{(k+1)},\widehat{\sigma}^{(k)}\right) 
	\leq \mathcal{L}_{n,\beta}\left(\widehat{\bm{\theta}}^{(k)},\widehat{\sigma}^{(k)}\right)
	 - c \left|\left|\widehat{\bm{\theta}}^{(k+1)} - \widehat{\bm{\theta}}^{(k)}\right|\right|^2, ~~~k\geq 0.
	 \label{EQ:SO1.1}
\end{eqnarray}
	Further, for each $k\geq 0$, these exist $c_k'>0$ and  $\bm{g}_{k} \in 
	\partial_{\bm{\theta}}\mathcal{L}_{n,\beta}(\widehat{\bm{\theta}}^{(k)}, \widehat{\sigma}^{(k)})$, 
	the Clarke subdifferential of $\mathcal{L}_{n, \beta}(\bm{\theta}, \widehat{\sigma}^{(k)})$ at $\widehat{\bm{\theta}}^{(k)}$,
	such that 
\begin{eqnarray}
		||\bm{g}_{k}|| \leq c_k' \left|\left|\widehat{\bm{\theta}}^{(k+1)} - \widehat{\bm{\theta}}^{(k)}\right|\right|.
	 \label{EQ:SO1.2}
\end{eqnarray}
%	(The expectation will be ignored for non-stochastic updation.) 

	\item[(SO2)] The $\sigma$-update in Step 3 of Algorithm \ref{Alg-dpd-nn} is exact in the sense that 
	$$
	\mathcal{L}_{n,\beta}(\widehat{\bm{\theta}}^{(k+1)}, \widehat{\sigma}^{(k+1)}) 
	\leq \mathcal{L}_{n,\beta}(\widehat{\bm{\theta}}^{(k+1)}, \sigma)
	~~~\mbox{ for all }~\sigma\geq \sigma_0.
	$$	
\end{itemize}

The following theorem then shows the convergence of the rRNet algorithm 
assuming the sub-optimization problems are solved suitably to ensure (SO1)--(SO2).

\begin{theorem}(Convergence of the rRNet)\label{THM:conv}
Under Assumptions (A0), (N0)--(N1) and (SO1)--(SO2), the sequence $\{(\widehat{\bm{\theta}}^{(k)},\widehat{\sigma}^{(k)})\}_{k\geq 1}$ 
generated by Algorithm \ref{Alg-dpd-nn}, for any given input data $\mathcal{D}_n$ and tuning parameter $\beta\geq 0$,  
satisfies the following properties: 
\begin{enumerate}
	\item[(a)] Monotone descent: For every $k\geq 1$, 
	$\mathcal{L}_{n,\beta}(\widehat{\bm{\theta}}^{(k+1)},\widehat{\sigma}^{(k+1)})
	\leq \mathcal{L}_{n,\beta}(\widehat{\bm{\theta}}^{(k)},\widehat{\sigma}^{(k)})$. 
	\item[(b)] Loss convergence: The sequence 
	$\{\mathcal{L}_{n,\beta}(\widehat{\bm{\theta}}^{(k)},\widehat{\sigma}^{(k)})\}_{k\geq 1}$ converges to a finite limit. 
%\end{enumerate}
%Additionally, if the error density $f$ is definable in an o-minimal structure, we have the followings: 
%\begin{enumerate}
%	\item[(c)] Finite-length property:  $\sum_{k=0}^\infty ||(\widehat{\bm{\theta}}^{(k+1)\top},\widehat{\sigma}^{(k+1)}) 
%	- (\widehat{\bm{\theta}}^{(k)\top},\widehat{\sigma}^{(k)}) || < \infty$. 
	\item[(c)] Stationarity: Every limit point $(\bm{\theta}_\infty, \sigma_\infty)$ is a stationary point of 
	$\mathcal{L}_{n,\beta}(\bm{\theta}, \sigma)$, i.e., 
	$$
	\bm{0} \in \partial_{\bm{\theta}}\mathcal{L}_{n,\beta}(\bm{\theta}_\infty, \sigma_\infty),
	~~~\mbox{and }~~~ 0 \in \partial_\sigma\mathcal{L}_{n,\beta}(\bm{\theta}_\infty, \sigma_\infty). 
%	\frac{\partial}{\partial\sigma}\mathcal{L}_{n,\beta}(\bm{\theta}_\infty, \sigma_\infty) = 0.  
	$$
%	Here, only the stationarity in $\sigma$-gradient assumes an additional condition 
%	that $f$ is $\mathcal{C}^1$, i.e., continuously differentiable (ensuring $\mathcal{L}_{n,\beta}$ is also $\mathcal{C}^1$ in $\sigma$). 
\end{enumerate}
\end{theorem}

Theorem \ref{THM:conv} ensures monotone descent of the DPD-loss 
and convergence of Algorithm \ref{Alg-dpd-nn} to stationary points  under minimal assumptions.
Although such guarantees are sufficient to justify the practical effectiveness of rRNet, 
one may further strengthen them by imposing additional, stricter structural assumptions.
For instance, assuming the Kurdyka-Lojasiewicz property for the DPD-loss would prevent oscillatory behavior
and ensure convergence of the entire parameter sequence (up to NN symmetries), along the line of  \cite{attouch2013convergence}.
Such assumptions, however, would unnecessarily restrict the scope of rRNet in practical applications 
and are therefore not pursued in this work.

When using stochastic updates for $\bm{\theta}$ in Step 2 of Algorithm \ref{Alg-dpd-nn}, 
e.g., using stochastic gradient descent methods, we must ensure that (SO1) holds in expectation of the DPD-loss.
Then, the convergence guarantees of Theorem \ref{THM:conv} continue to hold in expectation. 
Moreover, for fully differentiable NN models (having smooth activation functions) and error densities, 
the second part \eqref{EQ:SO1.2} of (SO1) holds automatically for any $\bm{\theta}$-update, 
and Part (c) of Theorem \ref{THM:conv} simplifies to 
$$
\frac{\partial}{\partial\bm{\theta}}\mathcal{L}_{n,\beta}(\bm{\theta}_\infty, \sigma_\infty)=\bm{0}_d 
~~\mbox{ and }~~
\frac{\partial}{\partial\sigma}\mathcal{L}_{n,\beta}(\bm{\theta}_\infty, \sigma_\infty) = 0.
$$

%\bigskip
\noindent\textbf{$\bm{\theta}$-$\sigma$ decoupling:} 
Practical performance of alternating minimization in Algorithm \ref{Alg-dpd-nn} crucially depends on 
the (second order) cross-partial derivatives of the population DPD-loss $\mathcal{L}_{n,\beta}^*$ with respect to $\bm{\theta}$ and $\sigma$,
which can be seen to be proportional to $C_{1,2}^{(\beta)}$.
However, if the error density $f$ satisfying (A0) is additionally symmetric about zero (e.g., Gaussian, Laplace, logistic, etc.), 
then $C_{1,2}^{(\beta)}=0$ for any $\beta \geq 0$ (Lemma \ref{LEM:A0}.v), exhibiting  $\bm{\theta}$-$\sigma$ decoupling. 
As noted in  \cite{bertsekas2016nonlinear}, such decoupling prevents the `zig-zagging' behavior 
common in ill-conditioned nonlinear problems, justifying our alternating optimization strategy in rRNet Algorithm \ref{Alg-dpd-nn}. 
This structural property ensures that the update steps for the NN weights $\bm{\theta}$ and the scale $\sigma$ 
do not interfere with each other asymptotically, leading to stable convergence of the rRNet 
even in the high-dimensional parameter (model weights) space.

\subsection{The sub-optimization problem with respect to $\boldsymbol{\theta}$}

%Let us now verify (SO1) for some common $\bm{\theta}$-update rules that can be used in the sub-optimization problem  of Step 2 in Algorithm \ref{Alg-dpd-nn}. 
The problem of minimizing $\mathcal{L}_{n,\beta}$ with respect to $\bm{\theta}$, even with a fixed $\sigma$, 
is clearly nonstandard because of  the non-convexity induced by the assumed NN  architecture in $\mu(\cdot, \bm{\theta})$. 
Thus, it is often not possible to find exact solution to this sub-optimization problem; 
we propose to solve it approximately to get an $\bm{\theta}$-update satisfying (SO1) 
which is sufficient for local convergence of our rRNet (by Theorem \ref{THM:conv}).

Note that, under Assumptions (A0) and (N1), the DPD-loss $\mathcal{L}_{n, \beta}$, 
is locally Lipschitz and Clarke sub-differentiable in $\bm{\theta}$ for any fixed $\sigma>0$ and $\beta\geq 0$.  
Its subdifferential is then obtained by the associated chain rule \citep{bolte2007clarke} as
\begin{equation}\label{EQ:theta_grad}
	\partial_{\bm{\theta}} \mathcal{L}_{n,\beta}\left(\bm{\theta},\sigma\right) = 
	\frac{1+\beta}{n\sigma^{\beta+1}} \sum\limits_{i=1}^n \psi_{1,\beta}\left(\frac{y_i - \mu(\bm{x}_i,\bm{\theta})}{\sigma}\right) \partial_{\bm{\theta}}\mu(\bm{x}_i,\bm{\theta}), 
\end{equation}
where $\psi_{1,\beta}(s) = u(s) f^{\beta}(s)$. 
%with $u = f'/f$ being the score function of the density $f$. 
Using it, we can update $\bm{\theta}$ in Step 2 of Algorithm \ref{Alg-dpd-nn} simply by a one-step sub-gradient descent: 
\begin{eqnarray}
\widehat{\bm{\theta}}^{(k+1)} =  \widehat{\bm{\theta}}^{(k)} - \alpha_k \bm{g}_k, ~~~
\bm{g}_k\in \partial_{\bm{\theta}}\mathcal{L}_{n,\beta}(\bm{\theta},\widehat{\sigma}^{(k)}), ~~k\geq 1,
\label{EQ:theta-update-gd1}
\end{eqnarray}
where $\alpha_k$ is a suitably chosen step-size (learning rate). 
That such a simple $\bm{\theta}$-update indeed satisfies the required Condition (SO1), 
ensuring the convergence of the overall rRNet algorithm, has been proved in the following proposition.

\begin{prop}\label{PROP:theta-updtae-dg1}
If the $\bm{\theta}$-update in Step 2 of Algorithm \ref{Alg-dpd-nn} is implemented by (\ref{EQ:theta-update-gd1}) for a suitably chosen $\alpha_k>0$,
then the conditions in (SO1) hold for any given input (tuning parameter $\beta\geq0$ and data $\mathcal{D}_n$) under Assumptions (A0) and (N1).
\end{prop}

Although it works theoretically, the simple  $\bm{\theta}$-update (\ref{EQ:theta-update-gd1}) for rRNet algorithm requires
the computation of the (sub-)gradient of the NN model $\mu(\bm{x}_i, \bm{\theta})$ for all observations, which is often quite expensive. 
So a suitable stochastic gradient descent or its momentum based versions (e.g., ADAM) is suggested in practice;
see Section \ref{implementation} for such a possible implementation for the Gaussian error.
These stochastic algorithms also satisfies (SO1), in expectation, under suitably assumptions on the loss function,
ensuring the overall convergence of the rRNet algorithm.   
For smooth NN models, their theoretical convergence guarantees follow directly from standard theory of optimization,
whereas for non-smooth NN models we may prove such results by using Clarke subdifferential calculus, as in \cite{shamir2013stochastic},
or through the smoothing function approach of \cite{giovannelli2025non}.     
However, to avoid distraction from the main focus on the robustness of rRNet, 
we defer such technical study of stochastic $\bm{\theta}$-updates for our future research.

\subsection{The sub-optimization problem with respect to $\sigma$}

The problem of optimizing the DPD-loss with respect to $\sigma$ alone is much simpler, 
since it is indeed a one-dimensional optimization problem that does not depend on the complexity of NN architecture. 
In fact, for any fixed $\bm{\theta}$ ($=\widehat{\bm{\theta}}^{(k+1)}$ in Step 3), 
Assumption (A0) ensures that the DPD-loss, viewed as a function of $\sigma$,  
is continuously differentiable a.e. and admits a unique minimizer. 
Moreover, the loss is strictly convex in $\log\sigma$ in a neighborhood of this minimizer, implying local strong convexity.
These can be verified easily by studying the sub-gradient of the DPD-loss with respect to $\sigma$, which is given by 
\begin{equation}\label{EQ:sigma_grad}
	{\partial_\sigma} \mathcal{L}_{n,\beta}\left(\bm{\theta},\sigma\right) = 
	\frac{1+\beta}{n\sigma^{1+\beta}} \sum\limits_{i=1}^n \psi_{2,\beta}\left(\frac{y_i - \mu(\bm{x}_i,\bm{\theta})}{\sigma}\right)  
	- \frac{\beta C^{(\beta)}_{0,0}}{\sigma^{1+\beta}}, %\int f^{1+\beta}d\lambda, 
\end{equation}
with $\psi_{2,\beta}(s) = (1 + s u(s)) f^{\beta}(s)$. % = f^{1+\beta}(s) + s \psi_{1,\beta}(s)$.
It becomes the unique (partial) derivative of the loss with respect to $\sigma$ 
whenever $f$ is everywhere differentiable (even if $\mu$ is non-smooth). 
Consequently, under (A0), the scale update in Step 3 of Algorithm \ref{Alg-dpd-nn} can be solved exactly using 
standard one-dimensional convex optimization methods, ensuring (SO2). 
In particular, we may employ Newton and quasi-Newton algorithms such as BFGS or L-BFGS, 
which are guaranteed to converge (locally) to the unique minimizer under available conditions \citep{sun2006optimization}.

For the most common case of Gaussian error discussed in Example \ref{EX:Gaussian_error},
one can alternatively solve this sub-optimization problem with respect to $\sigma$ by a simpler fixed-point iteration.
In this case, simplifying the partial derivative from (\ref{EQ:sigma_grad}), and equating it to zero, we get 
%to have the form
%\begin{equation}\label{EQ:sigma_grad_normal}
%\frac{\partial}{\partial\sigma} \mathcal{L}_{n,\beta}\left(\bm{\theta},\sigma\right) = 
%	\frac{\beta+1}{n\sigma^{\beta+1}} \sum\limits_{i=1}^n \psi_{1,\beta}\left(\frac{y_i - \mu(\bm{x}_i,\bm{\theta})}{\sigma}\right) \nabla_{\bm{\theta}}\mu(\bm{x}_i,\bm{\theta}), 
%\end{equation}
%which leads to the following fixed-point iteration:
$$
\sigma^2 = \frac{\sum\limits_{i=1}^n w_{i, \beta}(\bm{\eta}) (y_i - \mu(\bm{x}_i, \bm{\theta}))^2}{
%	e^{-\frac{\beta(y_i - \mu(\bm{x}_i, \bm{\theta}))^2}{2\sigma^2}}}{
	\sum\limits_{i=1}^n w_{i, \beta}(\bm{\eta}) - \frac{n\beta}{1+\beta}},
~~\mbox{ with }~ w_{i, \beta}(\bm{\eta}) = e^{-\frac{\beta(y_i - \mu(\bm{x}_i, \bm{\theta}))^2}{2\sigma^2}},
~~ i=1, \ldots, n.~~
$$
The above expression  can be used to construct the fixed-point update for $\sigma$, as previously used in the literature of MDPDEs \citep{ghosh2020}).
%This idea of updating noise variance has also been used successfully in the literature of DPD based inference \cite[see, e.g.,][]{ghosh2019robust}.
It follows from standard theory of univariate optimization \citep[see, e.g.][]{sun2006optimization} that 
the resulting sequence of $\sigma$-values converges globally to the unique minimizer,
and hence it can also be used safely in Step 3 of Algorithm \ref{Alg-dpd-nn} implementing rRNet.

Importantly, the above fixed-point formulation provides additional insights about 
the claimed robustness of our proposed rRNet, at least in terms of the estimate of $\sigma$ under Gaussian errors. 
The contribution of contaminating observations having high absolute residuals, $|y_i - \mu(\bm{x}_i, \bm{\theta})|$, 
gets exponentially downweighted at any $\beta>0$, thereby making the final estimate resistant against their adversarial effects. 
The extent of down-weighting via weights $w_{i, \beta}(\bm{\eta})$ further increases with increasing values of $\beta$, 
leading to greater robustness against stronger data contamination. 
For general NN and error models, such robustness behavior of the rRNet is achieved through the boundedness of the score-type functions 
$\psi_{1,\beta}$ and $\psi_{2,\beta}$, which will be formally justified in the following Sections \ref{robustness} and \ref{BP}.

\section{Local robustness guarantees: Influence functions}\label{robustness}

The influence function (IF) is a useful tool for measuring local bias-robustness (B-robustness) in statistical literature. 
\cite{koh2017understanding} examined the behavior of sample-based IFs in black-box deep learning models with possibly non-convex loss functions.
Recent works have demonstrated that these IFs are not always well-understood in such a context \citep{basu2020influence},
but they often provide a good approximation to the proximal Bregman response function, which successfully addresses many robustness concerns, 
e.g., identification of influential or mislabeled observations, in the context of NNs \citep{bae2022if}. 
However, all such attempts have emphasized that the sample-based IFs must be interpreted carefully in NN settings to get meaningful conclusions,
and there is no universal consensus for such interpretations.  

To avoid such issues associated with sample-based IFs, here we study the robustness of rRNet  through the IF of 
its population-level functionals, following its original statistical formulation \citep{hampel1986robust}. 
Formally, consider the MDPDF $\bm{T}_\beta(\overline{G}_n)$ corresponding to the parameter estimates obtained from the rRNet 
with tuning parameter $\beta\geq 0$ as defined in \eqref{EQ:MDPDF}, and a contaminated distribution 
$\overline{G}_{n,\epsilon} = (1-\epsilon)\overline{G}_n + \epsilon \wedge_{(t, \boldsymbol{x}_i)}$
where $\epsilon>0$ is the proportion contamination at the point $t\in\mathbb{R}$  in the response space for the $i$-th feature vector $\bm{x}_i$
for a pre-fixed $i\geq 1$. Then, the IF of the MDPDF $\bm{T}_\beta$ against such contamination is defined as
\begin{equation} \label{IF-def-gen}
IF_i(t, \bm{T}_\beta, \overline{G}_n) = \lim_{\epsilon\to 0} \frac{\bm{T}_\beta(\overline{G}_{n,\epsilon})-\bm{T}_\beta(\overline{G}_{n})}{\epsilon}   = \frac{\partial}{\partial\epsilon} \bm{T}_\beta(\overline{G}_{n,\epsilon})\bigg|_{\epsilon=0}.
\end{equation}
Clearly, the above $IF_i$ measures the standardized bias of the functional $\bm{T}_\beta$ caused by 
an infinitesimal contamination at the point $(t, \bm{x}_i)$. 
Hence, its nature over varying $t$ characterizes local B-robustness of $\bm{T}_\beta$ and, 
in turn, the same for the corresponding estimator (in an asymptotic sense). 
In particular, if the IF increases/decreases continuously as $t\rightarrow\pm\infty$, the absolute bias of the corresponding functional (estimator) 
increases indefinitely, indicating its severe non-robust even under an infinitesimal contamination at a distant point. 
A bounded IF in $t$ therefore indicates the robust nature of the functional. 
The maximum (or supremum) of the absolute IF over varying $t$, known as the \textit{gross-error sensitivity}, 
further measures the extent of such local B-robustness enabling us to compare various estimators in terms of their robustness.

Note that the definition in (\ref{IF-def-gen}) is valid under the assumption of (appropriate) differentiability of the functional $\bm{T}_\beta$,
which holds for the present case if $f$ is differentiable and also  $\mu(\bm{x}, \bm{\theta})$  is differentiable in $\bm{\theta}$, 
i.e., only for smooth activation functions and error densities. 
Below we first study the IF of the MDPDF for such cases and subsequently extend them for non-smooth NN models using the concepts of smoothing functions. 
We also study the local robustness of the unique rRNet target  $\mu_{n, \beta}^\ast$ by deriving its IF in both the cases.

%\subsection{IF of the MDPDF under smooth activation functions} \label{IF-EW}
\subsection{IFs under smooth activation functions and error densities} \label{IF-EW}

Let us first consider the cases of 2-smooth loss function which is ensured when both $f$ is $\mathcal{C}^2$ and 
$\mu(\bm{x}, \bm{\theta})$ is $\mathcal{C}^2$ in $\bm{\theta}\in\Theta$ for every $\bm{x}\in\mathcal{X}$,
i.e., the underlying activation functions are all $\mathcal{C}^2$.
%Let us consider the NN models for which $\mu(\bm{x}, \bm{\theta})$ is twice differentiable everywhere in $\bm{\theta}\in\Theta$, 
%for all $\bm{x}\in\mathcal{X}$, i.e., the underlying activation functions are twice differentiable, which 
Such 2-smooth NN models automatically satisfies (N1). 
Then, under Assumptions (A0) and (N0), a general formula for the IF of the MDPDF $\bm{T}_\beta$, defined in \eqref{EQ:MDPDF},  
at the true conditional distribution  $\overline{G}_n$  is given (implicitly) by the equation
\cite[see, e.g.,][]{basu2026} 
\begin{eqnarray}\label{EQ:IF-MDPDF-gen}
	\boldsymbol{\Psi}_n(\boldsymbol{\eta}) IF_i(t,\bm{T}_\beta,\overline{G}_n)  
	= - \frac{1}{n}\big\{\mathcal{I}_\beta(t,\bm{x}_i| \bm{\eta})  - E_{G_i}\left[\mathcal{I}_\beta(Y,\bm{x}_i|\bm{\eta}) \right]\big\},
\end{eqnarray}
where $\boldsymbol{\Psi}_n(\boldsymbol{\eta}) 
= n^{-1}\sum_{i=1}^n E_{G_i}\left[\nabla_{\bm{\eta}}\mathcal{I}_\beta(Y,\bm{x}_i|\bm{\eta}) \right]$ 
%with $\nabla_{\bm{\eta}}$ denoting the gradient with respect to $\bm{\eta}$, 
and,  for any $(y, \bm{x})\in\mathbb{R}\times\mathcal{X}$, 
\begin{eqnarray}
\mathcal{I}_\beta(y,\bm{x}|\bm{\eta}) = \left(\psi_{1,\beta}\left(\frac{y - \mu(\bm{x},\bm{\theta})}{\sigma}\right) [\nabla_{\bm{\theta}}\mu(\bm{x},\bm{\theta})]^\top, 
~ \psi_{2,\beta}\left(\frac{y - \mu(\bm{x},\bm{\theta})}{\sigma}\right) \right)^\top. 
\label{EQ:I}
\end{eqnarray}

Although it is relatively straightforward to obtain (\ref{EQ:IF-MDPDF-gen}) from the standard theory of IFs \cite[see, e.g.][]{hampel1986robust},
it is harder to simplify the expression of the resulting IF further for the present case of NN models.  
In the following theorem, we simplify it assuming that the conditional models are correctly specified so that $G_i = F_{i, \bm{\theta}}$,  
the distribution function of $f_{i, \bm{\theta}}$, for all $i=1, \ldots, n$. Here, we denote
%In this respect, we use the following notations:
\begin{eqnarray}
	&&\overline{F}_{n,\bm{\theta}}(y, \bm{x}) = \frac{1}{n}\sum_{i=1}^n \wedge_{\boldsymbol{x}_i}(\bm{x}) F_{i, \bm{\theta}}(y), ~~~y\in\mathbb{R}, ~\bm{x}\in\mathcal{X},
	\nonumber\\
%	&&C^{(\beta)}_{i,j} = \int s^i u^j(s) f^{1+\beta}(s) ds, ~~~~~ i, j = 0, 1, 2, \ldots, ~~\beta\geq 0,
%	\nonumber\\  
	\mbox{ and}&&
	\bm{\dot{\mu}}_n(\bm{\theta}) = \left[\nabla_{\bm{\theta}}\mu(\bm{x}_1,\bm{\theta})~ 
	\nabla_{\bm{\theta}}\mu(\bm{x}_2,\bm{\theta})~\cdots~ \nabla_{\bm{\theta}}\mu(\bm{x}_n,\bm{\theta})\right]^\top. 
%	~~\mbox{ and }~~\bm{J}_n(\bm{\theta}) = \frac{1}{n}\bm{\dot{\mu}}_n(\bm{\theta})^\top \bm{\dot{\mu}}_n(\bm{\theta}). 
\nonumber
\end{eqnarray}

%Accordingly, we introduce the following possible assumptions on the NN architecture:
%\begin{itemize}
%	\item[(N2a)] $\sup\limits_{\bm{x}\in\mathcal{X}} ||\nabla_{\bm{\theta}}\mu(\bm{x}, \bm{\theta}) || <\infty.$
%	\item[(N2b)] The matrix $\bm{J}_n(\bm{\theta})$ is non-singular. 
%\end{itemize}
%We may note that these assumptions are quite stricter than previous NN assumptions (N0)--(N1) and may not be satisfied in several applications. 
%Particularly, although (N2) may be satisfied for general NN modes under certain assumptions on features or activation functions, 
%(N3) can  only be satisfied in the cases of underparametrized NNs and fails of most deep networks. 
%So, we will first derive the general formula of the IF of the MDPDEs without these assumptions, 
%but subsequently discuss their possible simplifications under these assumptions

\begin{theorem}\label{THM:IF-MDPDF}
Consider the NN regression model given by (\ref{mlp-mean-fn}) where $f$ is $\mathcal{C}^2$ 
and $\mu(\bm{x}, \bm{\theta})$ is also $\mathcal{C}^2$ in $\bm{\theta}\in\Theta$ for every $\bm{x}\in\mathcal{X}$.
Fix any $\beta\geq 0$, and suppose that Assumptions (A0) and (N0) hold with $C^{(\beta)}_{1,2} = 0$.  
Put $\widetilde{C}^{(\beta)}_{2,2} = \left[C^{(\beta)}_{2,2} - \frac{1-\beta}{1+\beta} C^{(\beta)}_{0,0}\right]$.
Then, the IFs of the MDPDF $\bm{T}_\beta^{\bm{\theta}}$ of the NN model parameter $\bm{\theta}$ and 
the MDPDF ${T}_\beta^{\sigma}$ of the error scale $\sigma$ are independent at the model distribution $\overline{F}_{n,\bm{\theta}}$, 
and are respectively given by
\begin{eqnarray}
	 IF_i(t,\bm{T}_\beta^{\bm{\theta}}, \overline{F}_{n,\bm{\theta}})  
	&=& - \frac{\sigma}{C^{(\beta)}_{0,2}} \psi_{1,\beta}\left(\frac{t - \mu(\bm{x}_i,\bm{\theta})}{\sigma}\right) 
	\left[\bm{\dot{\mu}}_n(\bm{\theta})^\top \bm{\dot{\mu}}_n(\bm{\theta})\right]^{+}\nabla_{\bm{\theta}}\mu(\bm{x}_i,\bm{\theta}) + \bm{v}_n(\bm{\theta}),
	\nonumber\\
	&& ~~~~~~~~~~~~~~~~~~~~~~~
	\mbox{with }~~~\bm{v}_n(\bm{\theta})\in Ker(\bm{\dot{\mu}}_n(\bm{\theta})),
%	^\top \bm{\dot{\mu}}_n(\bm{\theta})),
	\label{EQ:IF-MDPDF-model-theta}
\end{eqnarray}
%and
\begin{eqnarray}
\mbox{and}~~~~~~
IF_i(t,{T}_\beta^{\sigma}, \overline{F}_{n,\bm{\theta}})  
= - \frac{\sigma}{n\widetilde{C}^{(\beta)}_{2,2}} \left[\psi_{2,\beta}\left(\frac{t - \mu(\bm{x}_i,\bm{\theta})}{\sigma}\right)  
	- \frac{\beta C^{(\beta)}_{0,0}}{1+\beta}\right].~~~~~~~~~~~~~~~~
	\label{EQ:IF-MDPDF-model-sigma}
\end{eqnarray}
\end{theorem}
%where $\widetilde{C}^{(\beta)}_{2,2} = \left[C^{(\beta)}_{2,2} - \frac{1-\beta}{1+\beta} C^{(\beta)}_{0,0}\right]$.

Note that, under the assumptions of Theorem \ref{THM:IF-MDPDF}, the IF of ${T}_\beta^{\sigma}$ (for error scale) is always unique
but that is not the case for the MDPDF of $\bm{\theta}$, as expected given (N0). 
For the simpler cases with non-singular $\left[\bm{\dot{\mu}}_n(\bm{\theta})^\top \bm{\dot{\mu}}_n(\bm{\theta})\right]$, 
the IF of $\bm{T}_\beta^{\bm{\theta}}$ is uniquely given by \eqref{EQ:IF-MDPDF-model-theta} with $\bm{v}_n(\bm{\theta}) = \bm{0}_d$ 
and the pseudo-inverse being the ordinary inverse.
In all other cases, however, this choice $\bm{v}_n(\bm{\theta}) = \bm{0}_d$  gives the unique minimum-norm solution 
for the IF of $\bm{T}_\beta^{\bm{\theta}}$. 
Notably, the IF for the unique rRNet target $\mu_{n, \beta}^\ast(\bm{x})$ at any \textit{admissible feature value}, $\bm{x}\in\mathcal{A}$, 
of the NN model is always unique, as shown in the following corollary. 

\begin{corollary}\label{COR:IF-predictor}
Under the assumptions of Theorem \ref{THM:IF-MDPDF}, for any $\beta\geq 0$ and any $\bm{x}\in \mathcal{A}$,
the IF of the rRNet predictor functional $\mu_{n, \beta}^\ast(\bm{x})$ at the model distribution $\overline{F}_{n,\bm{\theta}}$ 
is uniquely given by
\begin{eqnarray}
IF_i(t, \mu_{n, \beta}^\ast(\bm{x}), \overline{F}_{n,\bm{\theta}}) 
%= \lim_{\epsilon\to 0} \frac{\mu_{n, \beta, \epsilon}^\ast(\bm{x})-\mu_{n, \beta, 0}^\ast(\bm{x})}{\epsilon}   
=  - \frac{\sigma}{C^{(\beta)}_{0,2}} \psi_{1,\beta}\left(\frac{t - \mu(\bm{x}_i,\bm{\theta})}{\sigma}\right) 
\mathcal{H}_n(\bm{x}, \bm{\theta}),
%[\nabla_{\bm{\theta}}\mu(\bm{x},\bm{\theta})]^\top
%\left[\bm{\dot{\mu}}_n(\bm{\theta})^\top \bm{\dot{\mu}}_n(\bm{\theta})\right]^{+}\nabla_{\bm{\theta}}\mu(\bm{x}_i,\bm{\theta}),
%\nonumber
\label{EQ:IF-target}
\end{eqnarray}
where 
%$\mu_{n,\beta,\epsilon}^*(\boldsymbol{x}) =  \mu(\bm{x}, \bm{T}_\beta(\overline{G}_{n,\epsilon}))$ and 
$\mathcal{H}_n(\bm{x}, \bm{\theta}) = 
[\nabla_{\bm{\theta}}\mu(\bm{x},\bm{\theta})]^\top
\left[\bm{\dot{\mu}}_n(\bm{\theta})^\top \bm{\dot{\mu}}_n(\bm{\theta})\right]^{+}[\nabla_{\bm{\theta}}\mu(\bm{x}_i,\bm{\theta})]$.
For a feature value $\bm{x}\notin\mathcal{A}$, 
the IF of $\mu_{n, \beta}^\ast(\bm{x})$ at the model distribution $\overline{F}_{n,\bm{\theta}}$ is given by the 
expression in (\ref{EQ:IF-target}) plus a possibly non-unique scalar, independent of the contamination point $t$. 
\end{corollary}

Now, in order to examine the B-robustness of the MDPDFs obtained in our proposed rRNet, we need to examine the boundedness of their IFs. 
But, we may see that the general formulas of IF in (\ref{EQ:IF-MDPDF-gen}), 
or those given in Theorem \ref{THM:IF-MDPDF} and Corollary \ref{COR:IF-predictor} at the model distribution,
depends on the contaminating observation $(t, \bm{x}_i)$ only through $\mathcal{I}_\beta(t,\bm{x}_i|\bm{\eta})$, 
and hence their boundedness depends on the same of the functions $\psi_{1,\beta}$, $\psi_{2,\beta}$ 
and $\nabla_{\bm{\theta}}\mu(\bm{x},\bm{\theta})$. 
The resulting nature of the IFs under different contamination scheme is formally presented 
in the following corollary for $\beta>0$.

\begin{corollary}\label{COR:IF-bounded}
Under the assumptions of Theorem \ref{THM:IF-MDPDF}, for any $\beta>0$, we have 
$$
\sup_{t\in\mathbb{R}} ||IF_i(t,\bm{T}_\beta^{\bm{\theta}}, \overline{F}_{n,\bm{\theta}})||<\infty,
~~~
\sup_{t\in\mathbb{R}} |IF_i(t, {T}_\beta^{\sigma}, \overline{F}_{n,\bm{\theta}})|<\infty,
%\mbox{ for each } i\geq 1,
$$
for each $i\geq 1$, implying local B-robustness of the MDPDFs against response contamination. 
Additionally, if $\sup_{i\geq 1} ||\nabla_{\bm{\theta}}\mu(\bm{x}_i,\bm{\theta})|| <\infty$, then we have
$$
\sup_{t\in\mathbb{R}, ~i\geq 1} ||IF_i(t,\bm{T}_\beta^{\bm{\theta}}, \overline{F}_{n,\bm{\theta}})||<\infty,
~~~
\sup_{t\in\mathbb{R}, ~i\geq 1} |IF_i(t, {T}_\beta^{\sigma}, \overline{F}_{n,\bm{\theta}})|<\infty,
$$
%and hence $\sup\limits_{t\in\mathbb{R}, ~i\geq 1} |IF_i(t, \mu_{n, \beta}^\ast(\bm{x}), \overline{F}_{n,\bm{\theta}})|<\infty$ 
for all $\beta>0$, implying local B-robustness of the MDPDFs against both response contamination and varying feature values.
For both contamination types and any $\bm{x}\in\mathcal{X}$, under the same assumptions, we respectively have  
$$
\sup\limits_{t\in\mathbb{R}} |IF_i(t, \mu_{n, \beta}^\ast(\bm{x}), \overline{F}_{n,\bm{\theta}})|<\infty
~~~\mbox{ and }~~~
\sup\limits_{t\in\mathbb{R}, ~i\geq 1} |IF_i(t, \mu_{n, \beta}^\ast(\bm{x}), \overline{F}_{n,\bm{\theta}})|<\infty.
$$
%respectively, for any fixed $\bm{x}\in \mbox{Span}\{\bm{x}_1, \ldots, \bm{x}_n\}$. 
\end{corollary}

The results of the above corollary also holds when the IF is computed at any distribution $\overline{G}_n$, 
not necessarily the model $\overline{F}_{n, \bm{\theta}}$, but the proof will be more involved with complex expressions.
Let us now examine these IFs explicitly in a simple example. 

\begin{example}[Shallow NN with Gaussian error and sigmoid activation] \label{ex-3.1}
We now consider a shallow NN, an MLP with $L=1$ hidden layer, satisfying (N0) and 
take the underlying activation function to be sigmoid that makes $\mu$ twice differentiable in $\bm{\theta}$. 
Let the error density $f$ is Gaussian, which satisfies (A0) and the resulting DPD-loss function has the form 
 given in Example \ref{EX:Gaussian_error}.  
Then, we can further simplify the IFs of the MDPDFs at any $\beta\geq 0$ from Theorem \ref{THM:IF-MDPDF} 
by noting that $u(s) = -s$, and hence $C^{(\beta)}_{0,0} = (2\pi)^{-\frac{\beta}{2}}(1+\beta)^{-\frac{1}{2}}$, 
$C^{(\beta)}_{1,2} = 0$, $C^{(\beta)}_{0,2} = C^{(\beta)}_{0,0}/(1+\beta)$, 
$C^{(\beta)}_{2,2} = 3C^{(\beta)}_{0,0}/(1+\beta)^2$, and 
%$C^{(\beta)}_{2,2} = 3C^{(\beta)}_{0,2} = 3(2\pi)^{-\frac{\beta}{2}}(1+\beta)^{-\frac{1}{2}}$, and 
$$
\psi_{1,\beta}(s) = -{(2\pi)^{-\frac{\beta}{2}}} s e^{-\frac{\beta s^2}{2}},
~~~~
\psi_{2,\beta}(s) = {(2\pi)^{-\frac{\beta}{2}}}(1-s^2)e^{-\frac{\beta s^2}{2}}.
$$
Clearly, the functions $\psi_{1,\beta}(s)$ and $\psi_{2,\beta}(s)$, and hence the IFs of the MDPDFs related to our rRNet,
are both bounded implying local B-robustness at any $\beta>0$. 
But, they are clearly unbounded (linear and quadratic functions, respectively) at $\beta=0$.

In order to explicitly compute and plot the IFs, we need to compute $\nabla_{\bm{\theta}}\mu$. 
Let us illustrate it, with explicit computation, for a simple case with one input and one node in the hidden layer;
more general NN models can be handled numerically using backpropagation.
In this specific case, we have $\mu(x, \bm{\theta}) = w_0^{out} + w_1^{out}\phi(w_{0} + w_{1}x)$, with $\bm{\theta}=(w_{0},w_{1},w_0^{out},w_1^{out})^\top$ and $\phi(x) = 1/(1+e^{-x})$, so that we can explicitly compute 
\[
\nabla_{\bm{\theta}}\mu(x,\bm{\theta}) = \begin{bmatrix}
		w_1^{out}\phi'(w_{0} + w_{1}x) & w_1^{out}\phi'(w_{0} + w_{1}x)x & 1 & \phi(w_{0} + w_{1}x)
	\end{bmatrix}^\top.
\]
Using all these, we numerically compute the IFs of the MDPDFs of the parameters and that of the rRNet predictor, 
for various $i\geq 1$ and $\beta\geq 0$, taking $x_i$s to be $n=50$ equispaced points within $[-10,20]$ and 
the true parameter values to be $\bm{\theta} = (1,1,2,1.5)^\top$ and $\sigma=0.1$. 
Figure \ref{fig:IF-NN-2} presents them  for $i=2$, while the same obtained for $i=49$ 
is provided in Figure \ref{fig:IF-NN-49} in Appendix \ref{add-results}. 
All these plots are clearly consistent with the boundedness and robustness implications of the corresponding IFs 
as predicted by the general theory above. 
%As predicted from the general expression, all IFs are bounded for all $\beta > 0$ and 
They have greater redecending nature at higher values of $\beta>0$, indicating greater robustness for increasing $\beta$. 
%However, they are unbounded at $\beta = 0$, as expected, since the MDPDE with $\beta=0$ coincides with the non-robust MLE. 
\qed
\end{example}

\begin{figure}[!h]
	\centering
	\includegraphics[width=\textwidth]{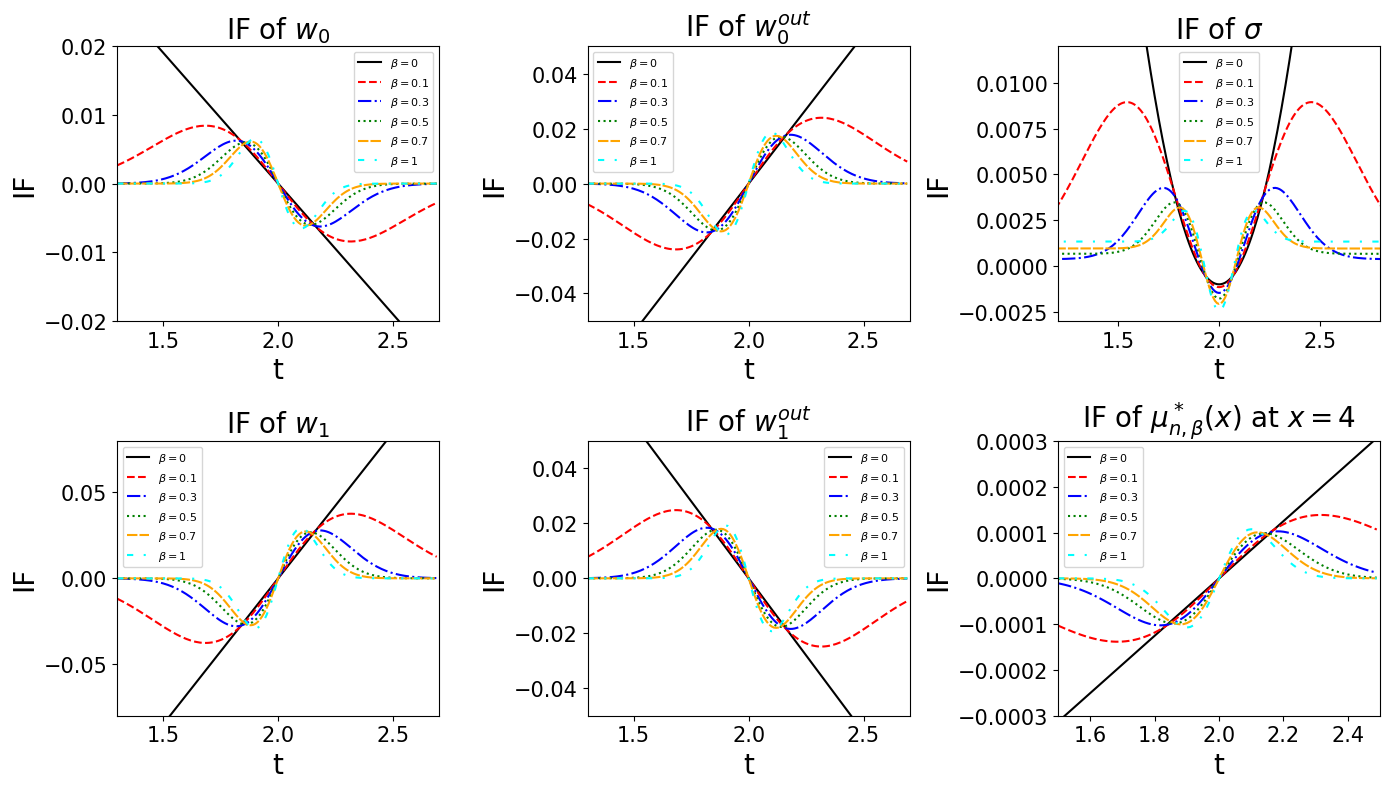}
	\caption{IFs of the MDPDFs and the rRNet predictors for a simple MLP, with sigmoid activation and Gaussian error, 
		under contamination in the 2nd observation [The case $\beta=0$ represents the standard LSE based training]}
	\label{fig:IF-NN-2}
\end{figure}

In the above example, we have seen that the IFs of the parameter estimates and the prediction obtained by the ML based training 
(rRNet at $\beta=0$) under Gaussian error is unbounded indicating their non-robust nature. 
The following corollary extend this for general error densities satisfying (A0); 
the proof is straightforward from Lemma \ref{LEM:A0} by noting that $\psi_{1,0}(s) = u(s)$ and $\psi_{2,0}(s) = 1+su(s)$.

\begin{corollary}\label{COR:IF-bounded0}
Under the assumptions of Theorem \ref{THM:IF-MDPDF}, the MDPDF at $\beta=0$ (i.e., the ML functional) satisfies the following 
at any $i\geq 1$: 
\begin{itemize}
	\item[a)]  $\sup_{t\in\mathbb{R}} ||IF_i(t,\bm{T}_0^{\bm{\theta}}, \overline{F}_{n,\bm{\theta}})||<\infty$
	and $\sup_{t\in\mathbb{R}} |IF_i(t, \mu_{n, \beta}^\ast(\bm{x}), \overline{F}_{n,\bm{\theta}})|<\infty$
	if and only if $f$ has (asymptotically) exact exponential tails, i.e., there exist constants $0<c_1\leq c_2<\infty$, $0<a<\infty$ 
	and $M>0$ such that 
	$$
	c_1 e^{-a|\varepsilon|} \leq f(\varepsilon) \leq c_2 e^{-a|\varepsilon|}~~~\mbox{ for all }|\varepsilon|\geq M.
	$$
	\item[b)]  $\sup_{t\in\mathbb{R}} |IF_i(t,{T}_\beta^{\sigma}, \overline{F}_{n,\bm{\theta}})| = \infty$
	for all $f$ satisfying (A0).
\end{itemize}
\end{corollary}

In particular, the above corollary implies that the IFs of the ML functionals of both $\bm{\theta}$ and $\sigma$,
as well as that of the NN prediction, are unbounded for error densities having light-tails. 
So, for regression NN with IID Gaussian errors (Example \ref{EX:Gaussian_error}), 
the usual MLE based training, equivalently the standard LSE based NN training, fails to  yield locally B-robust 
parameter estimates under infinitesimal contamination in responses, even when the feature values are kept fixed. 
However, the ML based training of regression NN under logistic error distribution (with fixed scale)
leads to locally robust estimates of the NN model parameter $\bm{\theta}$.

\subsection{Extending the IF for non-smooth models}

There are certain popular activation functions, such as ReLU ($\phi(x) = \max(0,x)$), 
that are not differentiable everywhere in $\mathbb{R}$, which causes the NN model function $\mu$ 
to be non-differentiable at certain points. The same may also hold for some error densities $f$, e.g., Laplace. 
However, Assumptions (A0) and (N1) ensure that both $f$ and $\mu$ are differentiable a.e., 
so that the DPD-loss $\mathcal{L}_{n, \beta}$ is non-differentiable only at finitely many points having Lebesgue measure zero. 
For such cases, we can use the smoothing function approach to define and study the IFs for the MDPDFs and the target rRNet predictors. 
The concept of smoothing function was, arguably, first introduced in the context of NNs by \cite{chen1995smoothing} 
who developed the \textit{softplus} activation as a smooth approximation to the ReLU (\textit{plus}) function. 
Subsequently, this idea is extensively used in the optimization of non-smooth functions; 
see, e.g., the recent work by \cite{giovannelli2025non}.
In statistical contexts, this has been used by \cite{AvellaMedina2017} and \cite{tamine2001smoothed}
while studying the IFs of the penalized M-estimators with non-smooth penalties and non-parametric estimators, respectively.

For the present context of NN regression, let us start by verifying the existence of such smoothing functions for both $f$ and $\mu$. 
Conveniently, this is automatically satisfied by any error density $f$ satisfying (A0) as shown in the following lemma. 

\begin{lemma}\label{LEM:S0}
Let $f$ satisfies Assumption (A0). Then there exists a sequence of probability densities $\{f_m\}_{m\geq 1}$ such that 
each $f_m$ is $\mathcal{C}^2$ and strictly log-concave on $\mathbb{R}$ and 
$f_m$ converges to $f$ uniformly on $\mathbb{R}$ as $m \to \infty$.
Additionally, if $f$ is symmetric about $0$, then $f_m$ may be constructed to be symmetric about $0$ for each $m\geq 1$.
\end{lemma}

However, our previous assumptions on the NN model $\mu$, namely (N0) and (N1), do not ensure the existence of such smoothing sequence for $\mu$,
and we need additional restrictions on the assumed NN architecture to ensure it. 
Here, let us directly assume the existence of required smoothing sequence in (N2) given below; 
couple of practically  useful sufficient conditions are provided in Appendix \ref{APP:NN_assumptionms}.

\begin{itemize}
%	\item[(S1)] There exists a sequence of densities $\{f_m\}_{m\geq 1}$ such that 
%	$f_m$ is $\mathcal{C}^2$ and satisfies (A0) for each $m\geq 1$, 
%	and $f_m$ converges uniformly (on $\mathbb{R}$) to $f$ as $m\rightarrow\infty$. 

	\item[(N2)] There exists a sequence of NN models $\{\mu_m(\bm{x}, \bm{\theta})\}_{m\geq 1}$ such that, for each $m\geq 1$,  
	$\mu_m$ satisfies (N0)--(N1) and is $\mathcal{C}^2$ in $\bm{\theta}\in\Theta$ for every $\bm{x}$, 
	and\\ $\sup\limits_{\bm{\theta}\in \Theta, ~i\geq 1} |\mu_m(\bm{x}_i, \bm{\theta}) - \mu(\bm{x}_i, \bm{\theta})| \rightarrow 0$ 
	as $m\rightarrow\infty$. 
\end{itemize}

Now, let $\mathcal{L}_{m, n,\beta}^\ast(\bm{\eta})$ denotes the population-level DPD-loss, 
as defined in  (\ref{dpd-loss-pop}) but with $f$ and $\mu$ being replaced by $f_m$ and $\mu_m$, respectively, for each $m\geq 1$, 
while its minimizer is denoted by $\bm{T}_{m,\beta}(\overline{G}_n)$; also put  
$\mu_{m,n,\beta}^*(\bm{x}) = \mu_m(\bm{x}, \bm{T}_{m,\beta}^{\bm{\theta}}(\overline{G}_n))$ for all $m\geq1$.
The following lemma present the convergences of these quantities to the respective targets as $m\rightarrow\infty$. 

\begin{lemma}\label{LEM:smoothing_func}
Under (A0) and (N0)--(N2), if the joint parameter space of $\bm{\eta}$ is taken to be compact, 
then we have the following convergence results as $m\rightarrow\infty$, for any fixed $n\geq 1$ and $\beta\geq 0$.
%$$
%\sup_{\bm{\eta}} \left|\mathcal{L}_{m, n,\beta}^\ast(\bm{\eta}) - \mathcal{L}_{n,\beta}^\ast(\bm{\eta})\right|\rightarrow 0,
%~~~~ 
%\bm{T}_{m,\beta}(\overline{G}_n) \rightarrow \bm{T}_{\beta}(\overline{G}_n),
%~~~~ 
%\sup_{i\geq 1} |\mu_{m,n,\beta}^*(\bm{x}_i) - \mu_{n,\beta}^*(\bm{x}_i)| \rightarrow 0.
%$$
\begin{itemize}
	\item[a)] $\sup_{\bm{\eta}} \left|\mathcal{L}_{m, n,\beta}^\ast(\bm{\eta}) - \mathcal{L}_{n,\beta}^\ast(\bm{\eta})\right|
	\rightarrow 0$, and hence $\bm{T}_{m,\beta}(\overline{G}_n) \rightarrow \bm{T}_{\beta}(\overline{G}_n)$. 
	\item[b)] $\sup_{i\geq 1} |\mu_{m,n,\beta}^*(\bm{x}_i) - \mu_{n,\beta}^*(\bm{x}_i)| \rightarrow 0$.
\end{itemize}  
\end{lemma}

Since $\mathcal{L}_{m, n,\beta}^\ast(\bm{\eta})$ is $\mathcal{C}^2$ by construction, 
the IFs of $\bm{T}_{m,\beta}(\overline{G}_n)$ and $\mu_{m,n,\beta}^*(\bm{x})$ can be obtained directly 
from the preceding subsection, particularly from Theorem \ref{THM:IF-MDPDF} there, for each $m\geq 1$. 
Based on Lemma \ref{LEM:smoothing_func}, it is then intuitive to define the IFs of our MDPDFs $\bm{T}_{\beta}(\overline{G}_n)$ 
and $\mu_{m,n,\beta}^*(\bm{x})$ with possibly non-smooth models by the limit of the IFs of 
$\bm{T}_{m,\beta}(\overline{G}_n)$ and $\mu_{m,n,\beta}^*(\bm{x})$, respectively, as $m \rightarrow \infty$, 
whenever these limits exist. The resulting form of the IFs are given in the following theorem.

\begin{theorem}\label{THM:IF-MDPDF-nonsmooth}
Suppose that Assumptions (A0) and (N0)--(N2) hold for the NN regression model (\ref{mlp-mean-fn}) 
with possibly non-smooth $f$ and $\mu$. Then, for any $\beta\geq 0$ and $\bm{x}\in \mathcal{A}$ (admissible feature domain), 
the IFs of the MDPDFs $\bm{T}_\beta^{\bm{\theta}}$, ${T}_\beta^{\sigma}$ and the rRNet target predictor $\mu_{n, \beta}^\ast(\bm{x})$
at the model distribution $\overline{F}_{n,\bm{\theta}}$ are given by 
(\ref{EQ:IF-MDPDF-model-theta}), (\ref{EQ:IF-MDPDF-model-sigma}), and (\ref{EQ:IF-target}), respectively, 
with $u$ within $\psi_{1,\beta}$ and $\psi_{2,\beta}$ being a measurable element of $\partial \ln f$, and
$\nabla_{\bm{\theta}}\mu$ being replaced by any measurable element of $\partial_{\bm{\theta}}\mu$ so that 
$$\bm{v}_n(\bm{\theta}) \in \bigcap_{i=1}^n \bigcap_{\bm{g}_i \in \partial \mu(\bm{x}_i, \bm{\theta})}Ker(\bm{g}_i).$$  
\end{theorem}

It follows from the above theorem that the local robustness properties of rRNet are preserved under non-smooth activation such as ReLU;
they are exactly the same as discussed in Corollaries \ref{COR:IF-bounded} and \ref{COR:IF-bounded0} for $\mathcal{C}^2$ models. 
In particular, the IFs of the MDPDFs and the rRNet predictor are bounded at all $\beta>0$ implying local B-robustness of our proposal; 
but they are unbounded (indicating non-robust nature) at $\beta=0$, unless $f$ has exact exponential tails 
in which case the IFs of $\bm{\theta}$ and the predictor functional  become bounded even at $\beta=0$ (the ML based training). 
Since the ML estimate under Laplace error (with known $\sigma$) is indeed the MAE of $\bm{\theta}$, 
an application of Theorem \ref{THM:IF-MDPDF-nonsmooth} with Laplace error density yields the IF formula  for the MAE, 
justifying its local robustness (as Laplace density has exponential tails).

\begin{example}[Shallow NN with Gaussian error and ReLU activation] \label{ex-3.2}
We again consider the example of simple shallow NN model with Gaussian error, as in Example \ref{ex-3.1}, 
but now with the ReLU activation function $\phi(x) = \phi_R(x) = \max(0,x)$, which is non-differentiable at $x = 0$. 
Hence, the NN model function $\mu(x_i, \bm{\theta}) = w_0^{out} + w_1^{out}\phi_R(w_{0} + w_{1}x_i)$ 
is not differentiable everywhere in $\bm{\theta}\in\Theta$. So, we consider the sequence of smoothing functions 
$\phi_m(x) = \frac{1}{m}\log(1+e^{mx})$ that converges uniformly to $\phi_R(x)$, as $m\to\infty$, for all $x\in\mathbb{R}$. 
The function $\phi_1(x)$ the standard softplus function \citep{zheng2015improving}. 
With these activation functions, the resulting sequence of NN models 
$\mu_m(x_i, \bm{\theta}) = w_0^{out} + w_1^{out}\phi_m(w_{0} + w_{1}x_i)$ satisfies Assumption (N2).
So, the IF of the MDPDFs of the NN model parameters with ReLU activation $\phi_R$ is given by 
the limits of the same obtained with activation function $\phi_m$. They have the exact same form as in Example \ref{ex-3.1}
with the same $\psi_{1,\beta}$ and $\psi_{2,\beta}$ (for Gaussian error) but with $\mu$ being computed using $\phi_R$ 
and $\nabla_{\bm{\theta}}\mu(x_i,\bm{\theta})$ being replaced by %(for each $i=1, \ldots, n$)
$$
\overline{\mu}_i(\bm{\theta}) := \begin{bmatrix}
        w_0^{out} H(w_{0} + w_{1}x_i) & w_1^{out} H(w_{0} + w_{1}x_i)x_i & 1 & \phi_R(w_{0} + w_{1}x_i)
    \end{bmatrix}^\top \in \partial_{\bm{\theta}}\mu({x}_i,\bm{\theta}),
$$
for all $i=1, \ldots, n$, where $H(x) = \lim\limits_{m\rightarrow\infty} \phi_m'(x) = \frac{1}{2}(sign(x)+1)$ is the Heaviside step function.  

Now, considering exactly the same values of features and true parameters as in Example \ref{ex-3.1}, 
we plot the resulting IFs of the MDPDFs and the rRNet predictor at $i = 2$ in Figure \ref{fig:IF-relu-50-2}; 
the same for $i=49$ are provided in Figure \ref{fig:IF-relu-10-49} (Appendix \ref{add-results}). 
Similar to the previous example, the limiting IFs are again bounded for $\beta > 0$ 
and are redescending in nature, indicating the increasing robustness for higher values of $\beta$
even for the ReLU activation function. As before, the unboundedness of the IFs at $\beta = 0$ indicates the lack of robustness of the ML based training,
and equivalently the standard LSE based training of ReLU networks. 
\end{example}

\begin{figure}[!h]
    \centering
    \includegraphics[width=\textwidth]{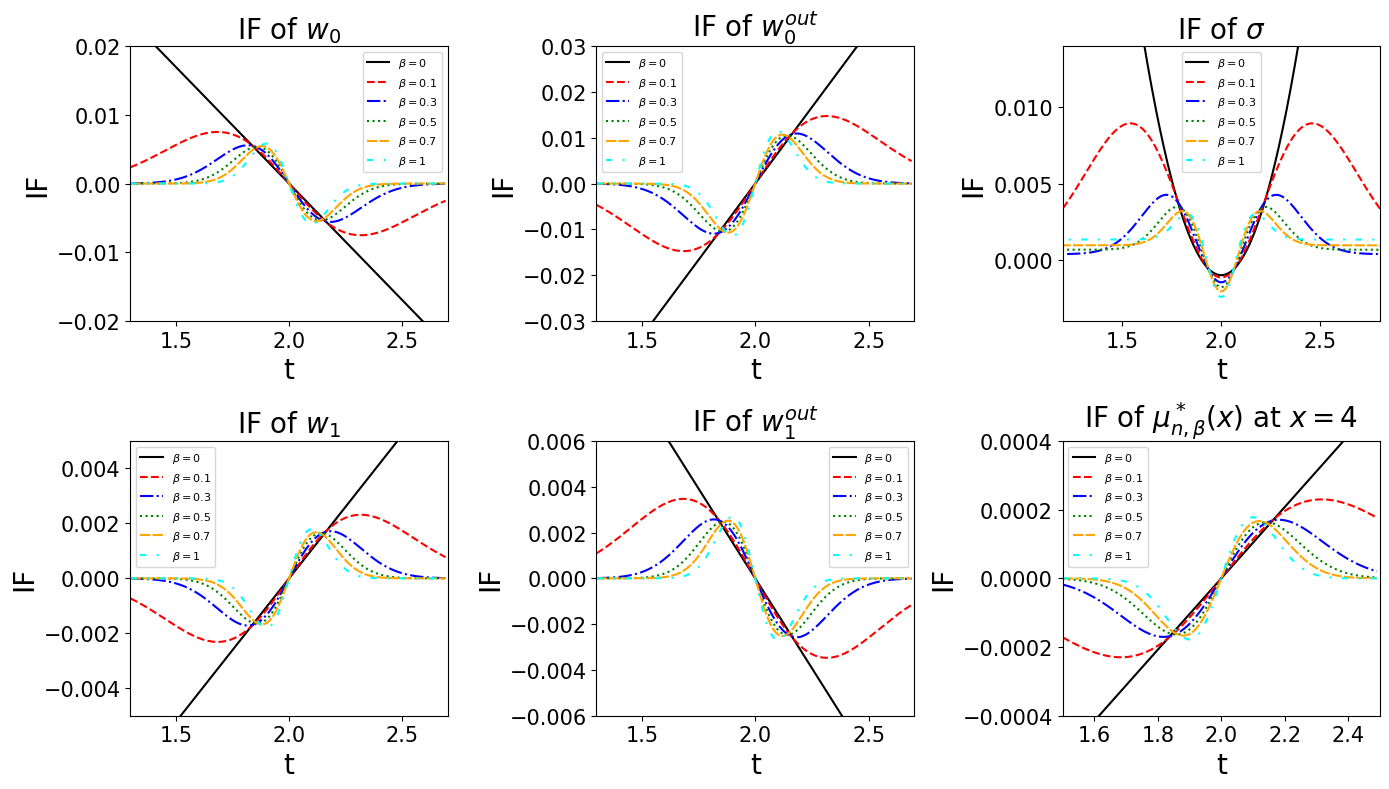}
    \caption{IFs of the MDPDFs and the rRNet predictors for a simple MLP, with ReLU activation and Gaussian error, 
    	under contamination in the 2nd observation [The case $\beta=0$ represents the standard LSE based training]}
    \label{fig:IF-relu-50-2}
\end{figure}

\section{Global robustness guarantees: Breakdown Analysis}\label{BP}

Let us now study global robustness of the proposed rRNet by deriving the asymptotic breakdown points (ABP) of associated functionals. 
%We show that, under a broad and practically relevant class of contaminations, the ABP of the rRNet functions 
%can be as high as 50\%, so that they do not \textit{break down} even under contamination proportion close to 0.5. 
Recall that the MDPDF of weight parameters $\bm{\theta}$ is unique only up to the NN symmetries, 
and so we should formally define the breakdown notion in the present context of learning regression NNs 
through the unique rRNet target (prediction) functional  $\mu_{n, \beta}^\ast$ (instead of the MDPDF of $\bm{\theta}$) 
along with the unique MDPDF of the error scale $\sigma$. Formally, we call a sequence of MDPDFs 
$\bm{\eta}_m =(\bm{\theta}_m^\top, \sigma_m)^\top$, $m\geq 1$, to \textit{break down}  if 
(since $\sigma$ is assumed to be bounded below)
\begin{eqnarray}
\mbox{ either }~ \max\limits_{i\geq 1}|\mu(\bm{x}_i, \bm{\theta}_m)| \rightarrow\infty, 
~~ \mbox{ or }~~  \sigma_m \rightarrow \infty, ~~~\mbox{as }~m\rightarrow\infty.
\label{EQ:BP_def0}
\end{eqnarray}

In the following analyses, we assume that the conditional model is specified correctly, i.e., $G_i = F_{i, \bm{\eta}_*}$ 
for some $\bm\eta_*=(\bm{\theta}_*^\top, \sigma_*)^\top$ in the interior of $\Theta \times [\sigma_0, \infty)$ and all $i\geq 1$.
But the data comes from a (gross-error) contaminated distribution having conditional densities of $Y$ given $\bm{x}_i$ 
as $G_{i,m, \epsilon} = (1-\epsilon) F_{i, \bm{\eta}_*} + \epsilon K_{i,m}$ for each $i\geq 1$ and $m\geq 1$, 
where $\{K_{i,m}\}_{m\geq 1}$ is a sequence of contaminating distributions at each $i\geq 1$ and $\epsilon$ is the contamination proportion. 
We then put $\overline{G}_{m, n, \epsilon}(y, \bm{x}) = \frac{1}{n}\sum_{i=1}^n \wedge_{\boldsymbol{x}_i}(\bm{x}) G_{i,m, \epsilon}(y)$ 
and $\bm{\eta}_{m, \beta} =(\bm{\theta}_{m, \beta}^\top, \sigma_{m, \beta})^\top = \bm{T}_\beta(\overline{G}_{n,m, \epsilon})$ 
for each $m, n \geq 1$ and $\beta\geq 0$. 
Then, following \cite{hampel1986robust} and \cite{jana2025bp},
we formally define the asymptotic breakdown point of rRNet functionals  with tuning parameter $\beta \geq 0$ as
\begin{eqnarray}
\epsilon_\beta^\ast &=& \sup\bigg\{ \epsilon\in[0, 0.5] : \limsup\limits_{n\rightarrow\infty} \limsup\limits_{m\rightarrow\infty}
\max \left\{\max\limits_{1\leq i\leq n}|\mu(\bm{x}_i, \bm{\theta}_{m, \beta})|, \sigma_{m, \beta} \right\} <\infty,  
\nonumber\\
&&~~~~~~~ ~~~~~~~~~~~\mbox{for all possible contaminating sequences } \{K_{i,m}\}_{m\geq 1} \bigg\}.~~~~
\label{EQ:BP-def}
\end{eqnarray}

The above definition of ABP is applicable for any NN learning algorithm whose prediction functional is uniquely defined, 
but computing it explicitly is not straightforward in general.  
Here, we  derive the breakdown point of rRNet against a wide class of practically relevant contaminating sequences,
specified by the following assumption.  

\begin{itemize}
	\item[(A1)] For each $i\geq 1$ and $m\geq 1$, the contaminating distribution $K_{i,m}$ admits a density $k_{i,m}$ 
	with respect to the Lebesgue measure such that 
	\begin{itemize}
		\item[(i)] $\lim\limits_{m\rightarrow\infty} \frac{1}{n}\sum\limits_{i=1}^n \int k_{i,m}^{1+\beta}(y)dy <\infty$ for each $n\geq 1$, and 
		\item[(ii)] $\lim\limits_{m\rightarrow\infty} \sup\limits_{i\ge1} \int_B k_{i,m}(y)\,dy = 0 $ 
		for every compact subset $B \subset \mathbb{R}$.
	\end{itemize}
%	\item[(B0)] For each $n\geq 1$, we have $\lim\limits_{m\rightarrow\infty} n^{-1}\sum\limits_{i=1}^n \int k_{i,m}^{1+\beta} <\infty$.  
%	\item[(B1)] For any compact $S\subset \Theta\times [\sigma_0, \infty)$, we have 
%	$\int \min\{ f_{i, \bm{\eta}}, k_{i,m} \} \rightarrow 0 $ as $m\rightarrow\infty$ uniformly for all $\bm{\eta}\in S$ and $i\geq 1$. 
%	\item[(B2)] For any sequence $\{\bm{\eta}_m\}_{m\geq 1}$ satisfying \eqref{EQ:BP_def0}, we have 
%	$\int \min\{ f_{i, \bm{\eta}_0}, f_{i, \bm{\eta}_m}\} \rightarrow 0 $ as $m\rightarrow\infty$ uniformly for all $i\geq 1$, and  
%	$\limsup\limits_{m\rightarrow\infty} \int f_{i, \bm{\eta}_m}^\beta k_{i,m} < \frac{C_k}{\sigma_0}$ for all $i \geq 1$
%	and some finite constant $C_k>0$ (independent of $i$). 
\end{itemize}

In the above assumptions, Part (i) ensures that the DPD-loss function 
and hence the rRNet functionals remain well-defined even under contaminated distributions. 
The second part (ii) of (A1)  makes the contaminating sequence of distributions practically meaningful,
by requiring them to be asymptotically negligible on any compact subset of the sample space, thereby concentrating their mass at infinity.
Under our base Assumption {\rm(A0)}, it further implies asymptotic mutual singularity between the assumed (conditional)  model and
contaminating distributions which is commonly assumed in the literature of minimum divergence estimators.
Under these assumptions, we now have the following theorem specifying  high breakdown of rRNet functionals.

\begin{theorem}\label{THM:BP}
Under Assumptions (A0), (A1) and (N0), the ABP $\epsilon_\beta^\ast$ of the rRNet functionals 
at the assumed (conditional) model distributions is given by  
\begin{eqnarray}
\epsilon_\beta^\ast = \left\{ \begin{array}{ll}
\frac{1}{2} & \mbox{ if } 0<\beta \leq 1, 
\\\\
0      		& \mbox{ if } \beta =0. 
\end{array}\right.  
\nonumber
\end{eqnarray}
\end{theorem}

The above theorem proves high global robustness of the proposed rRNet at any useful $\beta\in(0, 1]$; 
they can withstand even up to 50\% contamination in the training data. 
This is one of the strongest global robustness guarantees currently available for regression NNs.
The above theorem also formally proves that the standard LSE or MLE based training of regression NN 
breaks down even under a slightest of contamination proportion  as evident from their zero breakdown point.

We may note that Assumption (A1) is quite flexible and holds for a wide range of contaminating distributions.
In particular, if each $k_{i,m}$ belongs to the assumed model family of error densities satisfying (A0),
then Assumption (A1) holds whenever their location or scale parameter diverges to infinity.
In other words, contamination generated by pushing a log-concave location–scale model to infinity automatically satisfies (A1).
However, this assumption is not restricted to log-concave contamination only and applies equally to a broad class of non-log-concave densities.
This includes Student-$t$ distributions with sufficiently many degrees of freedom, polynomially decaying Pareto-type densities, 
finite mixtures, etc.,  provided their location or scale parameters diverge to ensure that 
the contaminating  mass escapes every compact subset of the sample space.
It also allows compactly supported contaminations whose mass drifts to infinity, such as 
$k_{i,m}(y)=\mathbf{1}_{[m,m+1]}(y)$, which are often useful as models of worst-case adversarial contamination.
Consequently, Theorem \ref{THM:BP} remains valid, proving global robustness of the rRNet, 
against almost all possible practically relevant (diverging) contamination mechanism.

%\section{Function-Space Robustness for Overparameterized Networks}
%\section{Function-Space Robustness: NTK Dynamics and Implicit Bias}

%\newpage
\section{Empirical illustrations}\label{empirical}

\subsection{An implementation of the rRNet for Gaussian noise}\label{implementation}

In all our numerical experiments, we implement the proposed rRNet under the assumption of Gaussian errors,
through a concrete instantiation of the general rRNet framework of Algorithm~\ref{Alg-dpd-nn} as described below.
%The details of this implementation is described below, while the full algorithm is presented in Algorithm \ref{alg:nested_adam}.
All associated python codes are shared (on request) through the GitHub repository \texttt{Robust-NN-learning}\footnote{\url{https://github.com/Suryasis124/Robust-NN-learning.git}}, 
to ensure reproducibility and independent validation of all the empirical results presented here.
We acknowledge that, while alternative optimization strategies or update schedules may be employed, 
the implementation described here is found to be stable and effective in all our experiments. 

Recall that, for Gaussian errors, the general DPD-loss in \eqref{dpd-loss-gen} simplifies to the form given in \eqref{dpd-loss}, 
with the associated $\psi_{1,\beta}$ and $\psi_{2,\beta}$ being specified in Example~\ref{ex-3.1}. 
In this setting, the error score function $u$ is uniquely defined, leading to  a unique gradient with respect to the scale parameter $\sigma$. 
Consequently, in Step~3 of Algorithm~\ref{Alg-dpd-nn}, we update $\sigma$ using the L-BFGS-B algorithm, 
a limited-memory quasi-Newton method designed for bound-constrained optimization  known for its superlinear convergence properties 
\citep{byrd1995limited}. Since the Gaussian density satisfies the Poincaré inequality (Lemma~\ref{LEM:A0}(viii)), 
the resulting DPD-loss exhibits a stable and non-exploding gradient surface with respect to $\sigma$, 
allowing L-BFGS-B to converge efficiently even without imposing an explicit upper bound on the scale parameter.

The optimization with respect to the network parameters $\bm{\theta}$ in Step~2 of Algorithm~\ref{Alg-dpd-nn}, however,  
requires special care when the regression function $\mu$ involves non-smooth activation functions. 
To address this, we employ the \texttt{TensorFlow} library \citep{abadi2015tensorflow} of Python, 
which computes a measurable selection from the subdifferential $\partial_{\bm{\theta}}\mu(\cdot,\bm{\theta})$ 
via automatic differentiation \citep{bolte2021conservative}. 
For smooth $\mu$, this coincides with the exact gradient, while for non-smooth architectures 
it yields a valid subgradient suitable for stochastic optimization. 
Following standard practice in large-scale NN training, 
we use the ADAM optimizer \citep{kingma2014adam} to solve the inner minimization problem with respect to $\bm{\theta}$. 
To reduce computational cost, ADAM is run for a fixed number of epochs $E=100$, 
with $\lceil n/32 \rceil$ stochastic mini-batches of size 32 per epoch.

%\begin{center}
%	\begin{minipage}{0.95\textwidth} 
		\begin{algorithm}[!h]
			\caption{Implementation of \texttt{rRNet} for Gaussian error (AutoDiff ADAM + L-BFGS-B)}
			\label{alg:nested_adam}
			\begin{algorithmic}[1]
				\makeatletter
				\setcounter{ALC@line}{-1}
				\makeatother
				
				\STATE \textbf{Input:} Training data $\mathcal{D}_n$, tuning parameter $\beta\geq 0$ and a tolerance limit $\varepsilon>0$. 
				\STATE \textit{Initialization:} Set $k = 0$. Fix initial values of NN weights $\widehat{\bm{\theta}}^{(0)}$ and scale $\widehat{\sigma}^{(0)}>0$. 
				\vspace{.6em}
				\item[] \hspace*{-1.2em} \textbf{do} \\[0.1em]
				\begin{ALC@g}
					\STATE Update the weights $\bm{\theta}$, with fixed $\sigma = \widehat{\sigma}^{(k)}$, using AutoDiff based ADAM:
				\begin{ALC@g}					
					\item[] Initialize $\widehat{\bm{\theta}} = \widehat{\bm{\theta}}^{(k)}$, $\bm{m}_0 = \bm{0}$, $\bm{v}_0 = \bm{0}$.
				\item[] \textbf{for} epoch $e = 1, \dots, E$ \textbf{do}
				\begin{ALC@g}
						\item[] Partition $\mathcal{D}_n$ into $J = \lceil n/32 \rceil$ mini-batches $\{B_1, \dots, B_J\}$.
						\item[] \textbf{for} each mini-batch $B_j$ ($j=1, \dots, J$) \textbf{do}
						\begin{ALC@g}
							\item[] \hspace{1em} $\bm{d}_i \leftarrow \text{AutoDiff}_{\bm{\theta}} [\mu(\bm{x}_i, \widehat{\bm{\theta}})]$, a selection from $\partial_{\bm{\theta}}  \mu(\bm{x}_i, \widehat{\bm{\theta}})$, 
							for $i \in B_j$.
							\item[] \hspace{1em}  
							$\bm{g}_j \leftarrow \frac{1+\beta}{32\sigma^{\beta+1}} \sum_{i \in B_j}\psi_{1,\beta}\left(\frac{y_i - \mu(\bm{x}_i,\widehat{\bm{\theta}})}{\sigma}\right) \bm{d}_i$.
							\item[] \hspace{1em} $\bm{m}_j \leftarrow \beta_1 \bm{m}_{j-1} + (1-\beta_1) \bm{g}_j$
							\item[] \hspace{1em} $\bm{v}_j \leftarrow \beta_2 \bm{v}_{j-1} + (1-\beta_2) \bm{g}_j^2$
							\item[] \hspace{1em} $\hat{\bm{m}}_j \leftarrow \bm{m}_j / (1-\beta_1^j), \quad \hat{\bm{v}}_j \leftarrow \bm{v}_j / (1-\beta_2^j)$
							\item[] \hspace{1em} $\widehat{\bm{\theta}} \leftarrow \widehat{\bm{\theta}} - \alpha \cdot \hat{\bm{m}}_j / (\sqrt{\hat{\bm{v}}_j} + \epsilon)$
						\end{ALC@g}
						\item[] \textbf{end for}
					\end{ALC@g}
					\item[] \textbf{end for}
				\item[]  Updated weights: $\widehat{\bm{\theta}}^{(k+1)} \leftarrow \widehat{\bm{\theta}}$ 
				\COMMENT{final value after $E$ epochs}.
					\end{ALC@g}
				\vspace{.6em}

		\STATE Update the scale $\sigma$, with fixed $\bm{\theta} = \widehat{\bm{\theta}}^{(k+1)}$, using L-BFGS-B:
				\begin{ALC@g}
					\item[] Initialize $\sigma_0 = \widehat{\sigma}^{(k)}$ and initial Hessian inverse $B_0 = 1$.
					\item[] \textbf{for} $j = 0$ to $J_{max} (=15000)$ \textbf{do}
				\begin{ALC@g}
					\item[] $g_j \leftarrow 	\frac{1+\beta}{n\sigma_j^{1+\beta}} \sum\limits_{i=1}^n 
					\psi_{2,\beta}\left(\frac{y_i - \mu(\bm{x}_i,\bm{\theta})}{\sigma_j}\right)
					- \frac{\beta C^{(\beta)}_{0,0}}{\sigma_j^{1+\beta}}$. 					
					\item[] \textbf{if} {$|g_j| < gtol=10^{-5}$} \textbf{then} **break** \textbf{end if}
					\item[] $s_j \leftarrow \sigma_{j} - \sigma_{j-1}, \quad y_j \leftarrow g_j - g_{j-1}$
					\item[] $B_{j} \leftarrow (1 - \frac{s_j y_j}{y_j^2}) B_{j-1} (1 - \frac{s_j y_j}{y_j^2}) + \frac{s_j^2}{y_j^2}$ 
					\item[] $\sigma_{j+1} \leftarrow \sigma_j - B_j g_j$ 
					\item[] \textbf{end for}
				\end{ALC@g}
					\item[] Updated scale: $\widehat{\sigma}^{(k+1)} \leftarrow \sigma_{j+1}$
				\end{ALC@g}
			\STATE Set $k \gets k+1$. \\[.5em]
			\end{ALC@g}
	\item[] \hspace*{-1.2em} \textbf{while} the reduction in loss is significant, i.e., 
				$ \mathcal{L}_{n,\beta} ( \widehat{\bm{\theta}}^{(k)}, \widehat{\sigma}^{(k)}) < 
				\mathcal{L}_{n,\beta}( \widehat{\bm{\theta}}^{(k-1)}, \widehat{\sigma}^{(k-1)} ) - \epsilon$.  \\[0.6em]
				\STATE \textbf{Return:} The rRNet estimates 
				$\widehat{\bm{\theta}}_{n,\beta} = \widehat{\bm{\theta}}^{(k+1)}$ and 	$\widehat{\sigma}_{n,\beta}= \widehat{\sigma}^{(k+1)}$.
			\end{algorithmic}
		\end{algorithm}
%	\end{minipage}
%\end{center}

The resulting nested optimization scheme, combining ADAM-based $\bm{\theta}$-update 
and L-BFGS-B updates for $\sigma$, yields a fully specified rRNet implementation presented in Algorithm~\ref{alg:nested_adam}. 
Convergence of the outer iteration is monitored through successive reduction of the DPD loss. 
This implementation still depends on several algorithmic hyperparameters associated with the inner optimizers. 
In all our experiments, we retain the default ADAM hyperparameters recommended in \cite{kingma2014adam}, 
namely step-size $\alpha = 0.001$, exponential decay rates $\beta_1 = 0.9$ and $\beta_2 = 0.999$, 
and numerical tolerance $\epsilon = 10^{-8}$. 
For the $\sigma$-update, we impose a lower bound $\sigma_0 = 0.001$ 
and use a function tolerance of $ftol = 2.22 \times 10^{-9}$, 
corresponding to the default setting of the \texttt{scipy.optimize.minimize} interface in Python.

One final aspect of the rRNet implementation is to use suitable initialization of $\bm{\theta}$ and $\sigma$ 
in Step 1 of Algorithm \ref{Alg-dpd-nn} or \ref{alg:nested_adam}.
To ensure stable gradient propagation, preventing vanishing or exploding gradients during the initial stages of the rRNet training, 
the network weights $\bm{\theta}$ are initialized using the Glorot uniform initializer \citep{glorot2010understanding}. 
%which scales the initial weight variance based on the layer dimensions. 
For initializing the scale parameter, we suggest a robust estimate based on the median absolute deviation of the initial residuals, 
$r_i = y_i - \mu(\bm{x}_i,\bm{\theta}^{(0)})$ for $i=1,\ldots,n$, given by
$$
\widehat{\sigma}^{(0)} = 1.4826~ \underset{{1\leq i\leq n}}{\text{Median}} 
\left| r_i - \underset{{1\leq j\leq n}}{\text{Median}}~ r_j\right|.
$$

Although Algorithm~\ref{alg:nested_adam} is designed to iterate until the reduction in the DPD-loss falls below a prescribed tolerance, 
in practice, running the full outer loop of rRNet to convergence may be computationally demanding in large-scale problems. 
In such cases, it is often sufficient to terminate the algorithm after a predetermined number of iterations, chosen using standard validation criteria. 
For instance, one may monitor the trimmed mean squared error (TMSE) of prediction on a validation set 
and stop the iterations once the TMSE exhibits negligible relative change across successive outer updates.

\subsection{Performance in function approximations: Simulation studies}\label{simulation}

%\subsubsection{Experimental setups} \label{sim-setups}

A major application of regression NNs is the approximation of complex functions from sampled observations.
To illustrate the performance of our rRNet in this respect, we model the following real-valued functions 
using suitable regression NNs trained on artificially simulated, possibly contaminated, datasets.   

\begin{enumerate}[label=$\bullet$~F-\arabic*:, ref=Function~\arabic*]
    \item\label{itm:f1} $\varphi_1(x) = |x|^{2/3},\ x\in\mathbb{R}$ (triangular-shaped function).

    \item\label{itm:f2} $\varphi_2(x) = \sin(x)/x,\ x\in\mathbb{R}$ (a sinusoidal-type curve).

	\item\label{itm:f6} $\varphi_3(x) = \sqrt{x(1-x)}~\sin\!\left(2.2\pi/(x + 0.15)\right), \ x \in [0,1]$ 
	(sinusoidal-type Doppler).

    \item\label{itm:f3} $\varphi_4(x_1, x_2) = x_1e^{-(x_1^2+x_2^2)},\ x_1,x_2\in\mathbb{R}$ (a simple 2-dimensional function).
    
    \item\label{itm:f4} A function on two variables, $\varphi_5(x_1,x_2)$, which is implicitly defined by the relationships 
    $x_1 = \sin(\varphi_5(x_1,x_2))$ and $x_2 = \cos(\varphi_5(x_1,x_2))$ (a 2D spiral).

    \item\label{itm:f5} $\varphi_6(\bm{x}) = \sum_{l=1}^3 e^{-\alpha_l ||\bm{x}-m_l||^2}, \ \bm{x}\in\mathbb{R}^2$,
    with $\alpha_1=1$, $\alpha_2=\alpha_3=2$, $\bm{m}_1=(0,0.75)^\top$, $\bm{m}_2=(0.5,-0.5)^\top$, $\bm{m}_3 = (-0.75,0)^\top$
    (sum of bivariate Gaussian kernels).
 
    \item\label{itm:f7} $\varphi_7(\bm{x}) = x_1 + \tan x_2 + x_3^3 + \ln (x_4+0.1) + 3x_5 + x_6 + \sqrt{x_7+0.1},\ \bm{x}\in [0,1]^7$~ 
    (a complex 7-variate function).
\end{enumerate}

These functions have previously been investigated while studying performances of existing NN training algorithms. 
In particular, $\varphi_4$ and $\varphi_5$ were considered by \cite{rusiecki2013robust}, while $\varphi_3$, $\varphi_6$ and $\varphi_7$ 
were studied by \cite{shen2021robust}, \cite{sorek2024robust} and \cite{bauer2019deep}, respectively.
See Figure \ref{fig:sim-data-instance} for visual representations of the structure of the univariate and bivariate functions 
($\varphi_1$ to $\varphi_6$), along with the sample data and outliers.\\

\begin{figure}[!h]
	\centering
	\begin{subfigure}[b]{0.3\textwidth}
		\centering
		\includegraphics[width=\textwidth]{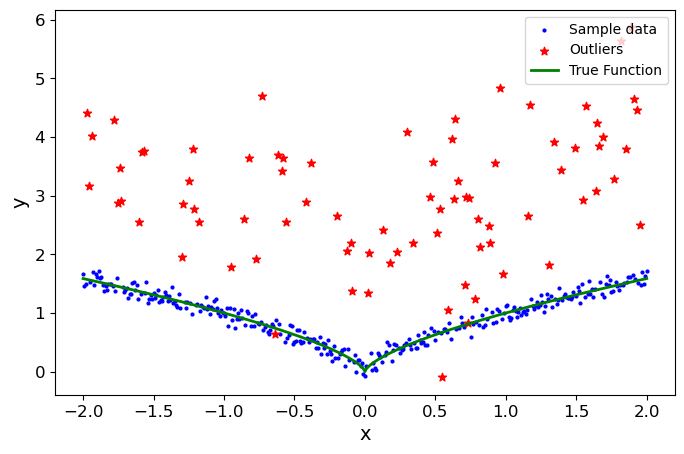}
		\caption{\ref{itm:f1} ($\varphi_1$)}
		\label{fig:fa-1}
	\end{subfigure}
	\hfill
	\begin{subfigure}[b]{0.3\textwidth}
		\centering
		\includegraphics[width=\textwidth]{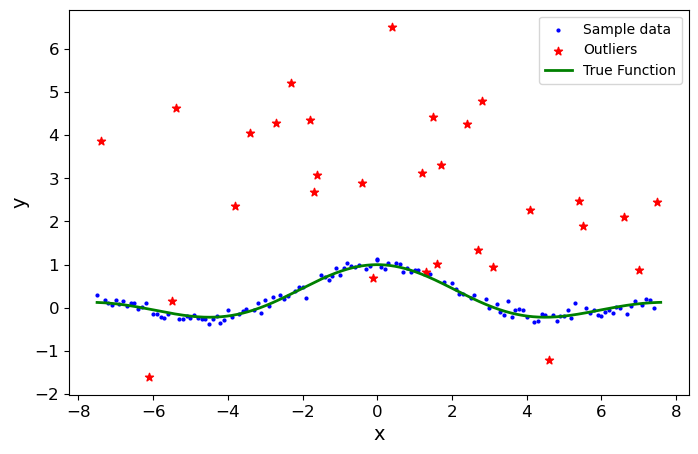}
		\caption{\ref{itm:f2} ($\varphi_2$)}
		\label{fig:fa-2}
	\end{subfigure}
	\hfill
	\begin{subfigure}[b]{0.3\textwidth}
		\centering
		\includegraphics[width=\textwidth]{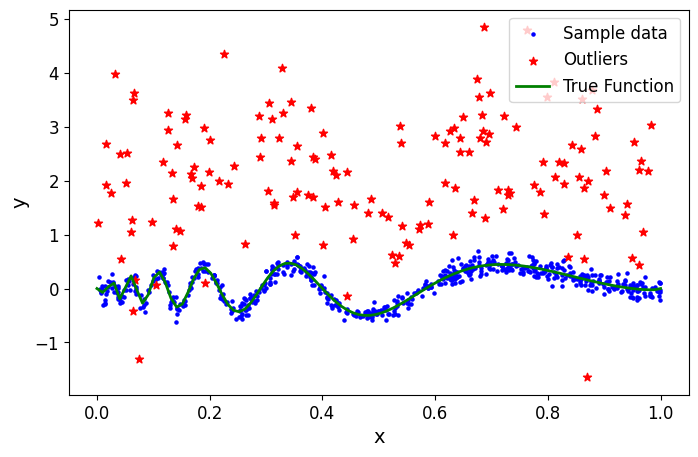}
		\caption{\ref{itm:f6} ($\varphi_3$)}
		\label{fig:fa-6}
	\end{subfigure}
	%	\hfill
	%	\begin{subfigure}[b]{0.3\textwidth}
		%		\centering
		%		\includegraphics[width=\textwidth]{NN-dopller-image.png}
		%		\caption{\ref{itm:f7} ($\varphi_2$)}
		%		\label{fig:fa-8}
		%	\end{subfigure}
	\hfill
	\begin{subfigure}[b]{0.4\textwidth}
		\centering
		\includegraphics[width=\textwidth]{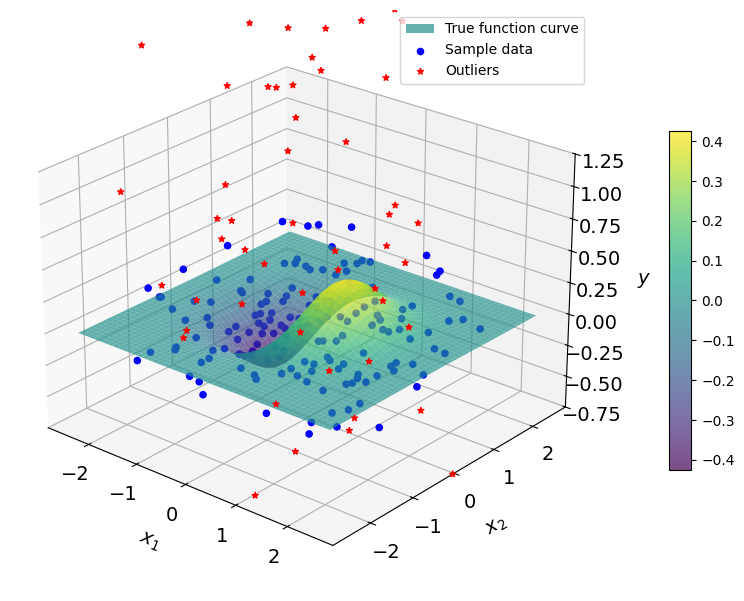}
		\caption{\ref{itm:f3} ($\varphi_4$)}
		\label{fig:fa-3}
	\end{subfigure}
	%	\hfill
	\begin{subfigure}[b]{0.4\textwidth}
		\centering
		\includegraphics[width=.95\textwidth]{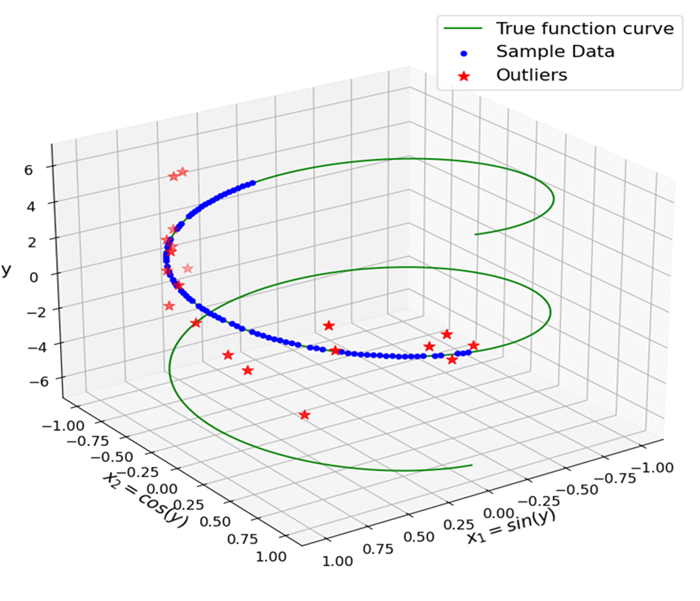}
		\caption{\ref{itm:f4} ($\varphi_5$)}
		\label{fig:fa-4}
	\end{subfigure}
	%	\hfill
	\begin{subfigure}[b]{0.4\textwidth}
		\centering
		\includegraphics[width=\textwidth]{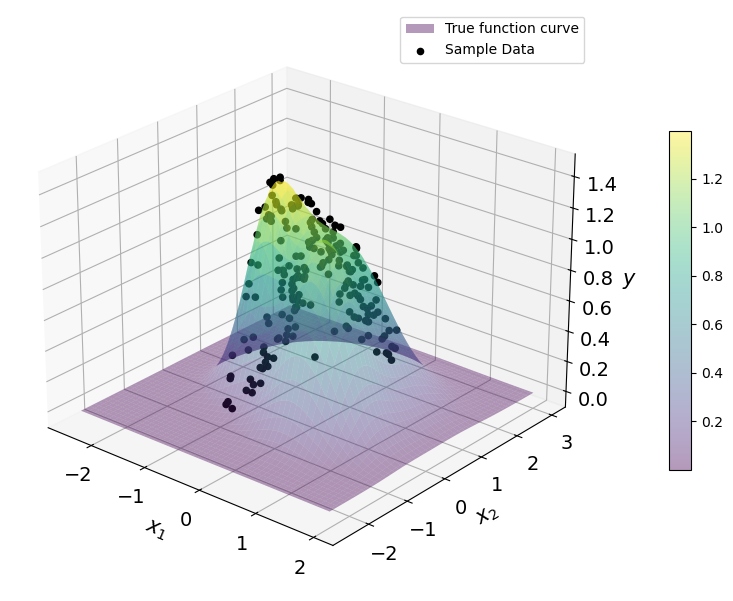}
		\caption{\ref{itm:f5} ($\varphi_6$) with sample data}
		\label{fig:fa-5a}
	\end{subfigure}
	%	\hfill
	\begin{subfigure}[b]{0.4\textwidth}
		\centering
		\includegraphics[width=\textwidth]{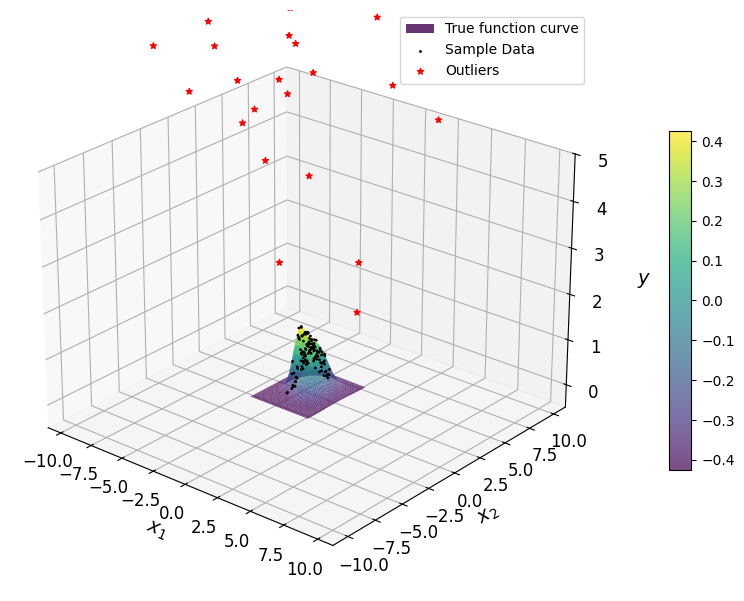}
		\caption{\ref{itm:f5} ($\varphi_6$) with data and outliers}
		\label{fig:fa-5b}
	\end{subfigure}
	\caption{Plots of the Functions 1-6 ($\varphi_1,\ldots,\varphi_6$), 
		along with an instance of simulated dataset with illustrative contamination}
	\label{fig:sim-data-instance}
\end{figure}

%\newpage
\noindent\textbf{Data generation mechanisms}:\\
For an identified function $\varphi\in \{\varphi_1, \ldots,\varphi_7\}$,
we have simulated 1000 artificial datasets, each consisting of $n$ observations $(\bm{x}_i, y_i),\ i=1,\hdots,n$, generated from the model
\begin{equation}\label{fa-model}
    y_i = \varphi(\bm{x}_i) + \epsilon_i, ~~i=1,\hdots,n,
\end{equation}
where $\epsilon_i$s are IID random errors with distribution $\mathcal{N}(0, \sigma^2)$. 
For each function, except \ref{itm:f4}, we first randomly sample the values of $\bm{x}_i$s and $\epsilon_i$s, 
and then generate $y_i$s using \eqref{fa-model}. For \ref{itm:f4}, however, first the values of $\varphi_5(\cdot)$ are generated, 
and then $\bm{x}_i$ are generated using the implicit functional relationships, 
while $y_i$s are constructed by adding independently simulated $\epsilon_i$s to the $\varphi_5(\cdot)$ values.
Subsequently, each dataset is contaminated artificially by randomly corrupting a particular proportion 
(say $\delta$, where $\delta=0,0.1,0.2,0.3$) of these observations by replacing the associated error components ($\epsilon_i$s) 
by new values generated independently from appropriate contaminating distributions. 
The assumed distribution and parameter values for these data generation processes for each function are summarized in Table \ref{tab:dgps}; 
note that the contamination mechanism for $\varphi_6$ is extremely severe,  relative to the nominal scale of the actual sample data 
(contrast Figures  \ref{fig:fa-5a}-\ref{fig:fa-5b}).\\

\begin{table}[!h]
\centering
\caption{Summary of data-generating mechanisms used in the simulation study. $n$ is the sample size and $\sigma$ is the error standard deviation considered in each case.}
\label{tab:dgps}
\begin{tabular}{c|c|p{5cm}|c|p{5.5cm}}
\hline
Function & $n$ & Covariate ($\bm{x}$) generation & $\sigma$ & Contamination scheme \\
\hline

$\varphi_1$ & 401 &
Equidistant values in $[-2,2]$ with step size $0.01$. &
0.1 &
Randomly draw $\epsilon_i$s from $\mathcal{N}(2,1)$. \\ \hline

$\varphi_2$ & 151 &
Equidistant values in $[-7.5,7.5]$ with step size $0.1$. &
0.1 &
Randomly draw $\epsilon_i$s from $\mathcal{N}(2,4)$. \\\hline

$\varphi_3$ & 800 &
Randomly generated from Uniform$(0,1)$ &
0.1 &
Randomly draw $\epsilon_i$s from $\mathcal{N}(2,1)$. \\ \hline

$\varphi_4$ & 256 &
Ordered pairs on a $16 \times 16$ grid using 16 equidistant values in $[-2,2]$ in each coordinate. &
0.1 &
Randomly draw $\epsilon_i$s from $\mathcal{N}(2,4)$. \\\hline

$\varphi_5$ & 100 &
$x_{1i} = \sin(z_i)$ and $x_{2i} = \cos(z_i)$, where $z_i$s are $100$ equidistant values in $[0,\pi]$. &
0.01 &
Randomly draw $\epsilon_i$s from $\mathcal{N}(0,4)$. \\\hline

$\varphi_6$ & 200 &
Random sample of size 200 from a bivariate uniform distribution on $[-1,1]\times[-1,1]$. &
0.05 &
Replace covariates with random draws from a bivariate uniform distribution on $[-10,10]\times[-10,10]$, and responses with random draws from $\mathcal{N}(10,10)$. \\
\hline
$\varphi_7$ & 200 &
Randomly generated from Uniform distribution over $[0,1]^7$ &
1 &
Randomly draw $\epsilon_i$s from $\mathcal{N}(5,25)$. \\
\hline
\end{tabular}
\end{table}

\noindent\textbf{Specification of the NN architecture:} \\
A fully-connected MLP with $L$ hidden layers with the same activation function in each hidden layer  
and a linear output was trained on the artificially simulated datasets. For each function, except $\varphi_6$, 
the number of hidden layers and the number of nodes per hidden layer are selected based on the average (training data) 
TMSE over 100 replications.  For $\varphi_6$,  we  follow the MLP architecture used by \cite{sorek2024robust}. 
The final choices of activation functions and numbers of hidden nodes are as follows: 
%(covers both underparametrized and overparametrized NNs): 

\begin{table}[H]
    \centering
    \begin{tabular}{l|c|c|c|c|c|c|c}
        Function & $\varphi_1$ & $\varphi_2$ & $\varphi_3$ & $\varphi_4$ & $\varphi_5$ & $\varphi_6$ & $\varphi_7$\\ \hline
        {Activation function ($\phi_l$)} & ReLU & Sigmoid & ReLU & Sigmoid & ReLU & GELU & ReLU\\
        Number of hidden layers ($L$) & 1 & 1 & 5 & 1 & 1 & 1 & 3 \\
        {Nodes per hidden layer ($K_l$)} & 5 & 10 & 50 & 15 & 10 & 30 & 30\\ \hline
        Dimension of parameter $\bm{\theta}$ & 16 & 31 & 10351 & 61 & 41 & 121 & 2011 \\
    \end{tabular}
%    \caption{NN architecture considered for approximating different functions in our simulation study}
    \label{tab:nn-arc}
\end{table}

\noindent \textbf{Competing methods:}\\ 
We have compared the results obtained by our proposed rRNet 
with several existing robust NN learning algorithms discussed in Section \ref{intro}, 
namely the MAE, LTA, LTS, LMLS, and M-estimators with the Huber and Tukey's loss functions, along with the usual LSE-based training. 
For M-estimators, standard choices for tuning parameters are used 
-- $c = 1.345$ and $c = 4.685$ for the Huber and Tukey's loss, respectively. 
For LTS and LTA estimators, the trimming constant $h$ is selected adaptively following the recommendation of \cite{rusiecki2013robust}. 
For \ref{itm:f1} only, Tukey's method is employed with the number of hidden nodes being $7$ and $11$, respectively, 
for clean and contaminated datasets, due to its poorer performance at any lower number of hidden nodes. 
In all other cases, we have used the same NN architecture as mentioned above.\\

\noindent \textbf{Performance measures:} \\ 
For each method, we have used the estimated value $\widehat{\bm{\theta}}$ of the NN weights $\bm{\theta}$ 
to compute the fitted (predicted) response values at each feature values, i.e.,
  $\widehat{y}_i =  \mu(\bm{x}_i, \widehat{\bm{\theta}})$ for $i=1, \ldots, n$, 
and compared them with the actual response values $y_i$s. 
As a summary accuracy measure of the fit (function approximation), 
we have used the average (training data) TMSE  across 1000 replications,
with the trimming proportion being the same as the proportion of inserted contamination in each dataset. 

Additionally, to study the out-of-sample prediction accuracy of the fit, 
we have generated a fresh test set of $n$ observations $(\bm{x}_i^{test},y_i^{test}),\ i=1,\hdots,n$, from the same model \eqref{fa-model}, 
(without any contamination) and computed the test data MSE, given by 
$\frac{1}{n}\sum_{i=1}^n (y_i^{test} - \mu(\bm{x}_i^{test}, \widehat{\bm{\theta}}))^2$. 
Note that the NN weights $\widehat{\bm{\theta}}$ are still obtained from the NN model trained on possibly contaminated data, 
but we have now evaluated their performances on clean test data. \\

\noindent \textbf{Simulation results:}\\
The resulting average train TMSE and test MSE over 1000 replications are presented in Tables \ref{tab:f6}--\ref{tab:f5}
%\ref{tab:f1}-\ref{tab:f7} for each of the chosen functions, respectively. 
for $\varphi_3$, $\varphi_5$ and $\varphi_6$, respectively; 
the same for the remaining functions are provided in Appendix \ref{add-results}.
These results together show that, under clean data, all robust methods exhibit very similar, 
good performance in approximating the target functions.  
However, under contaminated data, our proposed rRNet with an appropriate value of $\beta$, outperforms most other competing methods.
Particularly, performances of the rRNet with larger $\beta>0$ remains remarkably stable  
under increasing proportion of data contamination.  Additionally, a higher value of tuning parameter $\beta$ is required 
to get the best results at a higher contamination proportion. 
Performances of the LSE and LTS methods clearly deteriorate significantly under stronger contamination,
while the Tukey's M-estimator seems to be the closest existing competitor to our proposed rRNet.

\begin{table}[!h]
	\centering
	\caption{Average train TMSE and test MSE obtained while approximating the function $\varphi_3$ 
		based on sampled data with contamination proportion $\delta$ (the minimum is highlighted in bold font)}
	\begin{tabular}{l|rrrr|rrrr}
		    \hline
		NN training & \multicolumn{4}{c}{Avg TMSE (on training data)} & \multicolumn{4}{|c}{Avg MSE (on test data)} \\
		\multicolumn{1}{c|}{methods } & $\delta \rightarrow~~$0\% & 10\% & 20\% & 30\% & $\delta \rightarrow~~$0\% & 10\% & 20\% & 30\% \\ \hline
		LSE & \textbf{0.0112} & 0.0605 & 0.1640 & 0.3042 & \textbf{0.0126} & 0.0702 & 0.2013 & 0.4102 \\ 
		\multicolumn{9}{l}{\underline{DPD loss with tuning parameter $\beta$}}\\
		$\beta = 0.1$ & \textbf{0.0112} & 0.0111 & 0.0930 & 0.2320 & 0.0127 & {0.0132} & 0.1156 & 0.3148 \\
		$\beta = 0.3$ & 0.0116 & 0.0112 & \textbf{0.0109} & \textbf{0.0107} & 0.0131 & 0.0133 & \textbf{0.0135} & \textbf{0.0140} \\
		$\beta = 0.5$ & 0.0123 & 0.0118 & 0.0113 & 0.0109 & 0.0138 & 0.0140 & 0.0141 & 0.0145 \\
		$\beta = 0.7$ & 0.0133 & 0.0124 & 0.0118 & 0.0112 & 0.0148 & 0.0149 & 0.0149 & 0.0150 \\
		$\beta = 1  $ & 0.0149 & 0.0135 & 0.0124 & 0.0117 & 0.0164 & 0.0162 & 0.0159 & 0.0159 \\
		\multicolumn{9}{l}{\underline{Existing robust losses}}\\
		MAE & 0.0123 & 0.0126 & 0.0138 & 0.0163 & 0.0137 & 0.0148 & 0.0173 & 0.0256 \\
		LTA & 0.0123 & 0.0126 & 0.0138 & 0.0163 & 0.0137 & 0.0148 & 0.0173 & 0.0256 \\
		LTS & \textbf{0.0112} & 0.0605 & 0.1640 & 0.3042 & 0.0126 & 0.0702 & 0.2013 & 0.4102 \\
		LMLS & \textbf{0.0112} & 0.0181 & 0.0368 & 0.0768 & 0.0127 & 0.0212 & 0.0463 & 0.1060 \\ 
		Huber's M & 0.0116 & 0.0127 & 0.0194 & 0.0478 & 0.0131 & 0.0149 & 0.0243 & 0.0667 \\
		Tukey's M & 0.0118 & \textbf{0.0110} & 0.0111 & 0.0138 & 0.0132 & \textbf{0.0130} & 0.0137 & 0.0181 \\
		\hline
	\end{tabular}
	\label{tab:f6}
\end{table}

\begin{table}[!h]
    \centering
    \caption{Average train TMSE and test MSE obtained while approximating the function $\varphi_5$ based on sampled data with contamination proportion $\delta$ (the minimum is highlighted in bold font)}
    \begin{tabular}{l|rrrr|rrrr}
    \hline
        NN training & \multicolumn{4}{c}{Avg TMSE (on training data)} & \multicolumn{4}{|c}{Avg MSE (on test data)} \\
        \multicolumn{1}{c|}{methods } & $\delta \rightarrow~~$0\% & 10\% & 20\% & 30\% & $\delta \rightarrow~~$0\% & 10\% & 20\% & 30\% \\ \hline
        LSE & 0.0110 & 0.0281 & 0.0383 & 0.0443 & 0.0121 & 0.0325 & 0.0487 & 0.0637 \\[.3em]
        \multicolumn{9}{l}{\underline{Proposed rRNet: DPD-loss with tuning parameter $\beta$}}\\
        $\beta$ = 0.1 & \textbf{0.0100} & 0.0097 & 0.0173 & 0.0295 & \textbf{0.0115} & 0.0120 & 0.0227 & 0.0426 \\
        $\beta$ = 0.3 & 0.0102 & \textbf{0.0095} & \textbf{0.0090} & 0.0087 & 0.0116 & \textbf{0.0118} & \textbf{0.0122} & \textbf{0.0130} \\
        $\beta$ = 0.5 & 0.0107 & 0.0099 & 0.0092 & \textbf{0.0087} & 0.0121 & 0.0123 & 0.0125 & 0.0136 \\
        $\beta$ = 0.7 & 0.0120 & 0.0108 & 0.0099 & 0.0090 & 0.0131 & 0.0133 & 0.0137 & 0.0146 \\
        $\beta$ = 1 & 0.0158 & 0.0133 & 0.0116 & 0.0100 & 0.0166 & 0.0164 & 0.0173 & 0.0175 \\
        \multicolumn{9}{l}{\underline{Existing robust losses}}\\
        MAE & 0.0132 & 0.0131 & 0.0133 & 0.0136 & 0.0143 & 0.0160 & 0.0192 & 0.0247 \\
        LTA & 0.0114 & 0.0109 & 0.0106 & 0.0104 & 0.0126 & 0.0132 & 0.0140 & 0.0156 \\
        LTS & 0.0110 & 0.0284 & 0.0384 & 0.0440 & 0.0121 & 0.0329 & 0.0490 & 0.0632 \\
        LMLS & 0.0107 & 0.0133 & 0.0155 & 0.0177 & 0.0120 & 0.0158 & 0.0200 & 0.0254 \\ 
        Huber's M & 0.0106 & 0.0108 & 0.0121 & 0.0147 & 0.0120 & 0.0131 & 0.0157 & 0.0210 \\
        Tukey's M & 0.0151 & 0.0275 & 0.0153 & 0.0152 & 0.0155 & 0.0639 & 0.0418 & 0.0487 \\
        \hline
    \end{tabular}
    \label{tab:f4}
\end{table}

\begin{table}[!h]
    \centering
    \caption{Average train TMSE and test MSE obtained while approximating the function $\varphi_6$ based on sampled data with contamination proportion $\delta$ (the minimum is highlighted in bold font)}
    \begin{tabular}{l|rrrr|rrrr}
    \hline
        NN training & \multicolumn{4}{c}{Avg TMSE (on training data)} & \multicolumn{4}{|c}{Avg MSE (on test data)} \\
        \multicolumn{1}{c|}{methods } & $\delta \rightarrow~~$0\% & 10\% & 20\% & 30\% & $\delta \rightarrow~~$0\% & 10\% & 20\% & 30\% \\ \hline
        LSE & 0.0039 & 0.0214 & 0.0448 & 0.0706 & 0.0048 & 0.0562 & 0.1335 & 0.2300 \\[.3em]
        \multicolumn{9}{l}{\underline{Proposed rRNet: DPD-loss with tuning parameter $\beta$}}\\
        $\beta$ = 0.1 & \textbf{0.0033} & \textbf{0.0035} & \textbf{0.0038} & 0.0471 & \textbf{0.0042} & \textbf{0.0045} & \textbf{0.0051} & 0.0913 \\
        $\beta$ = 0.3 & 0.0036 & 0.0040 & 0.0045 & \textbf{0.0050} & 0.0045 & 0.0053 & 0.0065 & \textbf{0.0080} \\
        $\beta$ = 0.5 & 0.0047 & 0.0055 & 0.0056 & 0.0059 & 0.0058 & 0.0072 & 0.0085 & 0.0098 \\
        $\beta$ = 0.7 & 0.0069 & 0.0072 & 0.0072 & 0.0073 & 0.0082 & 0.0092 & 0.0107 & 0.0130 \\
        $\beta$ = 1 & 0.0101 & 0.0097 & 0.0091 & 0.0092 & 0.0118 & 0.0120 & 0.0136 & 0.0188 \\
        \multicolumn{9}{l}{\underline{Existing robust losses}}\\
        MAE & 0.0042 & 0.0371 & 0.0660 & 0.0900 & 0.0052 & 0.0646 & 0.1371 & 0.2201 \\
        LTA & 0.0042 & 0.0371 & 0.0660 & 0.0900 & 0.0052 & 0.0646 & 0.1371 & 0.2201 \\
        LTS & 0.0039 & 0.1330 & 0.2628 & 0.3229 & 0.0048 & 0.2068 & 0.4548 & 0.6594 \\
        LMLS & 0.0038 & 0.0306 & 0.0571 & 0.0808 & 0.0047 & 0.0483 & 0.1020 & 0.1639 \\ 
        Huber's M & 0.0036 & 0.0264 & 0.0587 & 0.1002 & 0.0044 & 0.0482 & 0.1210 & 0.2244 \\
        Tukey's M & 0.0056 & 0.0036 & 0.0042 & 0.0075 & 0.0066 & 0.0046 & 0.0058 & 0.0126 \\
        \hline
    \end{tabular}
    \label{tab:f5}
\end{table}

In particular, for \ref{itm:f1} and \ref{itm:f2}, the LSE performs the best under clean training data ($\delta =0$). 
Under contaminated data, however, the rRNet with a higher value of $\beta>0$ is observed to provide the best performance 
with increasing contamination proportion. For \ref{itm:f6}, it can be observed that, 
under heavy contamination (20\% and 30\%), the rRNet with $\beta=0.3$ outperforms all other approaches. 
In contrast, under light (10\%) or no (0\%) contamination, the rRNet with $\beta=0.1$ yields the best performance among all $\beta\in[0,1]$;
they are also better than existing competitors, but Tukey's M-estimators gives slightly better (mostly comparable) 
results than the best performing rRNet.  
Also, for \ref{itm:f4}, the proposed rRNet performs best at $\beta=0.3$ in the presence of contamination, 
whereas, the rRNet with $\beta=0.1$ gives better results for uncontaminated data;
they are always significantly better than all competitors.  
Similarly, the rRNet exhibits the best performance for \ref{itm:f5} at $\beta=0.1$ when the contamination proportion is $\leq 20\%$ 
and at $\beta=0.3$ for 30\% contamination. 
It is impressive that the proposed rRNet can successfully handle such strong data contamination for \ref{itm:f5}.
For \ref{itm:f7}, we can observed that the rRNet with $\beta=0$ (LSE) and $\beta=0.1$ performs best, 
in terms of their performances on training data, under lighter data contamination (0\% and 10\%). 
However, on test data, the rRNet with $\beta=0.3$ uniformly performs the best for any positive contamination proportion ($\geq$ 10\%.)
In general, rRNet provides significant improvements over all competitors particularly at heavier contamination level
and on the test data MSE, suggesting it's better generalizability beyond training sample even under strong contamination.

% For the simulation results presented in this section, Steps 2 and 3 of Algorithm \ref{Alg-dpd-nn} are iterated for a predetermined fixed number of iterations. However, when analyzing the real datasets in the next section, the maximum number of iterations is determined heuristically by monitoring the trimmed mean squared error (TMSE) of the training data and selecting a point after its stabilization, i.e., the stabilization of the neural network fit.

\subsection{Performance in real-life prediction problems} \label{real-data}

We now illustrate the practical applicability of the proposed rRNet for predicting continuous responses 
based on a set of related covariates in real-life datasets. For this purpose, we analyze following three datasets: 
\begin{itemize}
    \item \textbf{Airfoil Self-Noise (ASN) data:} 
    This dataset, originally developed by the National Aeronautics and Space Administration (NASA), 
    contains aerodynamic and acoustic measurements related to airfoil self-noise generated from experiments on NACA 0012 airfoils 
    \citep{airfoil_self-noise_291}. Here we aim to predict the scaled sound pressure level (in decibels)
    based on five available input features, namely the frequency, angle of attack, chord length, free-stream velocity, 
    and suction-side displacement thickness.
    
    \item \textbf{Concrete Compressive Strength (CCS) data:} This dataset, obtained from \cite{yeh1998modeling},
    is used for modeling concrete compressive strength (response variable) based on eight physical factors (features).
    These features are the amounts (in kg/m$^3$) of cement, blast furnace slag, fly ash, water, superplasticizer, 
    coarse aggregate and fine aggregate, as well as the age of testing (in days).
    
    \item \textbf{Boston Housing Price (BHP) data:} This dataset, originally introduced by \cite{harrison1978hedonic}, 
    contains housing-related information of Boston, collected by the U.S. Census Service. 
    Here we target to predict the median value of owner-occupied homes (in \$1000's)
    based on 14 available quantitative features -- per capita crime rate in the town, 
    proportion of residential land zoned for lots over 25,000 sq.ft., proportion of non-retail business acres per town, 
    indicator of tract bounding Charles river, nitric oxides concentration (parts per 10 million), 
    average number of rooms per dwelling, proportion of owner-occupied units built before 1940, 
    weighted distances to five Boston employment centers, index of accessibility to radial highways, 
    full-value property-tax rate per \$10000, pupil-teacher ratio by town, percentage of lower status among the population
    and a custom measure of racial discrimination in town. 
\end{itemize}

For each dataset, we train a 3-layer MLP (input layer $\rightarrow$ a hidden layer $\rightarrow$ output layer) 
with ReLU activation in the hidden layer and a linear output for predicting the responses based on the available covariates/features. 
The numbers of hidden nodes are taken to be 14, 21, and 8 for the ASN, CCS, and BHP data, respectively.
The covariates are scaled to $[0,1]$, in order to ensure numerical stability and facilitate efficient training of the regression NNs 
with various loss functions (as in the simulation studies). The responses are also scaled to $(0, 1)$ for  the ASN and the BHP data.

We compare the performance of different loss functions using the (average) out-of-sample TMSE (with 20\% trimming) 
computed through $k$-fold cross-validation (CV). The results obtained with $k=10$ are reported in Table \ref{tab:real-data}; 
other choices of $k$ are seen to yield the same ranking of loss functions, and hence they are not repeated here for brevity. 
Significant and consistent reduction in the TMSE by robust methods, including our rRNet, relative to the traditional LSE 
suggests the presence of non-negligible outliers in these datasets with respect to the assumed noise model.

\begin{table}[!h]
\centering
\caption{10-fold CV TMSE (20\% trimming) obtained by different learning algorithms 
	for the three datasets (for readability, the results for ASN and BHP data are reported as 100 $\times$ original figures)}
\label{tab:real-data}
\begin{tabular}{l|ccc}
\hline
 & \multicolumn{3}{c}{Dataset}\\
Method & ASN & CCS & BHP \\ \hline
LSE & 0.145 & 14.042 & 0.172 \\
\multicolumn{4}{l}{\underline{DPD loss with tuning parameter $\beta$}} \\
$\beta = 0.1$ & 0.120 & 13.555 & 0.156 \\
$\beta = 0.3$ & \textbf{0.113} & 11.454 & \textbf{0.144} \\
$\beta = 0.5$ & 0.119 & \textbf{10.947} & 0.145 \\
$\beta = 0.7$ & 0.125 & 11.445 & 0.151 \\
$\beta = 1$   & 0.142 & 13.931 & 0.172 \\ \hline
\multicolumn{4}{l}{\underline{Competitor robust losses}} \\
MAE        & 0.126 & 11.629 & 0.153 \\
LTA        & 0.123 & 11.738 & 0.154 \\
LTS        & 0.145 & 14.042 & 0.172 \\
LMLS       & 0.145 & 14.234 & 0.180 \\ 
Huber's M  & 0.125 & 12.372 & 0.157 \\
Tukey's M  & 0.114 & 12.078 & 0.150 \\
\hline
\end{tabular}
\end{table}

Further, from Table \ref{tab:real-data}, it is clear that the lowest CV TMSE is achieved by the proposed rRNet 
with a suitable choice of $\beta$, indicating its strong robust performance, for all three datasets. 
In particular, the optimal values of the DPD tuning parameter $\beta$, yielding the best performance for the ASN, CCS, and BHP datasets, 
are $0.3$, $0.5$, and $0.3$, respectively.  In addition, Tukey’s method is observed to achieve a comparable performance to 
that of the rRNet for the ASN data, whereas the LTA method is found to deliver competitive results for the CCS dataset. 

We should also mention here that the validity of our assumption of additive Gaussian noise has been assessed empirically for each dataset 
by studying the residuals of the fitted NN model to the whole data. %, using both the LSE and the proposed rRNet algorithms. 
A few representative histograms of these residuals obtained from the LSE and rRNet, along with their scatter plots with the fitted values, 
are presented in Figures \ref{fig:Airfoil}-\ref{fig:Boston-medv} in Appendix \ref{add-results} for all three datasets. 
These indicate that the  residuals do not exhibit systematic structure (mostly random) and the histograms of the residuals are bell-shaped, 
suggesting the IID and normality assumptions on errors are reasonable.  
However, for the LSE-based method, the residual plots cannot separate out the outlying points, as expected due to its non-robust nature. 
These hidden outliers are easily detected by the residual plots of the proposed DPD-based rRNet, 
facilitating additional outlier detection power of our proposal.

\subsection{On the choice of $\beta$ in the rRNet}

Since the performance of our rRNet depends crucially on its tuning parameter $\beta\geq 0$, 
one must choose it properly for any real data application. 
We have observed throughout that the ideal choice of $\beta$ varies depending on the amount of contamination in data, 
with stronger contamination needing larger values of $\beta$. 
However, in practice, the contamination proportion and types are often unknown, 
making it hard to make informative theoretical suggestion on the choice of $\beta$. 
So, given Remark \ref{REM:1}, we suggest to select an appropriate (practically effective) value of $\beta\in[0,1]$ 
by minimizing the cross-validated out-of-sample TMSEs, as illustrated for data examples presented above. 
However, other alternative criteria can also be used to select this tuning parameter; 
see, e.g., \cite{fujisawa2006robust} and \cite{sugasawa2021selection} for two such approaches 
which can be routinely adopted in the present context of NN learning.
Considering the length of the current manuscript, we have deferred the detailed study of 
such data-driven selection of $\beta$ for a sequel work.

\section{Concluding remarks}\label{conclusion}

In this study, a novel robust loss function is introduced for robust training of regression NNs
based on the statistically grounded minimum DPD estimation procedure. 
Robustness properties of the proposed learning framework are rigorously established through a detailed analysis 
of the IFs of both the estimated network parameters and the resulting predicted responses, 
as well as their asymptotic  breakdown point analyses. 
The boundedness of the IFs demonstrates local (B-)robustness of the rRNet, 
while the breakdown point result characterizes its global robustness under data possible contamination.

A key strength of the proposed framework lies in its generality. 
The rRNet algorithm is formulated for broad classes of regression NN models  
and error distributions satisfying minimal interpretable assumptions, 
and both its convergence behavior and robustness properties are established within this unified setting. 
To the best of our knowledge, rRNet constitutes one of the first robust learning frameworks 
for general regression NNs that is accompanied by such rigorous theoretical guarantees for both local and global robustness. 
Practical advantages of the proposed approach are further demonstrated through extensive simulation studies 
and real-data analyses, where rRNet performs superior, or at least competitive, to existing robust learning methods 
under data contamination.

Although our empirical investigations focus on regression models with mean functions represented by fully connected MLPs, 
the generality of the underlying methodology allows its direct extension to more complex neural architectures 
using the same DPD-based loss formulation as noted in Appendix~\ref{APP:assumptionms}. 
In particular, it is practically important to study the rRNet for regression NNs with regularization 
such as dropout or elastic-net penalty, and multivariate outcomes (responses).  
It is also of interest to adapt the rRNet framework to classification problems via generalized regression formulations,  
as well as to more complex architectures such as convolutional or recurrent NNs. 
Along these lines, we have parallelly  developed a similar robust neural learning framework for classification problems
based on a general class of divergence measures, including the $\beta$-divergence, in \cite{jana2026rSDNet}. 
We hope to also address some other extensions in future work.

\bigskip\bigskip
\noindent\textbf{Acknowledgment:}\\
The authors thank Mr. Sourojyoti Barick, research scholar at the Indian Statistical Institute, 
for his valuable guidance on the implementation of rRNet in Python.

\newpage
\appendix
\setcounter{figure}{0}
\setcounter{table}{0}
\setcounter{equation}{0}
\renewcommand{\thefigure}{S\arabic{figure}}
\renewcommand{\thetable}{S\arabic{table}}
\renewcommand{\theequation}{S\arabic{equation}}
\renewcommand{\theHfigure}{S\arabic{figure}}
\renewcommand{\theHtable}{S\arabic{table}}
\renewcommand{\theHequation}{S\arabic{equation}}

\begin{center}
{\Huge Supplementary Appendices} 
\vspace{1cm}
\end{center}

\section{A brief background on the minimum DPD estimation} \label{bg}

%As we have noted in Section \ref{intro}, although the name $\beta$-divergence is more common in the literature of information theory and machine learning, we will keep on referring to this divergence as the DPD following the naming convention from its proposers \citep{basu1998robust}. 

For completeness, we briefly review the DPD-based robust estimation procedure 
under independent non-homogeneous (INH) setups following \cite{ghosh2013robust}. 
Recall that the form of the DPD is given in Eq.~\eqref{dpd-def},  which depends on a tuning parameter $\beta\geq 0$ 
(giving rise to the alternate name, $\beta$-divergence). 

Consider a given sample of INH observations $y_1,\hdots,y_n$, 
where each $y_i$ follows an unknown true distribution $G_i$ with density $g_i$, for $i=1,2,\hdots,n$. 
Suppose we model them by parametric families of densities 
$\mathcal{F}_i = \left\{f_{i,\bm{\theta}}: \bm{\theta}\in\Theta \subseteq \mathbb{R}^d\right\}$, 
where the model densities may differ for each $i$ but share the common parameter $\bm{\theta}$. 
For such a setup, \cite{ghosh2013robust} suggested obtaining the MDPDE of $\bm{\theta}$ as a minimizer of 
the average DPD measure between the estimated true densities and the model densities, given by
\[\frac{1}{n}\sum_{i=1}^n d_\beta(\widehat{g}_i, f_{i,\bm{\theta}}),\]
over the parameter space $\Theta$, where $\widehat{g}_i$ is an empirical estimate of $g_i$ based on the single observation $y_i$, 
(the Dirac delta function at $y_i$ in the present case) for each $i$. Then, using the definition of the DPD measure in \eqref{dpd-def},  
the MDPDE $\bm{\widehat{\theta}}_\beta$ of $\bm{\theta}$ with tuning parameter $\beta>0$, 
can be defined as a minimizer of the simpler objective (loss) function
\begin{equation}\label{obj-fn}
    H_{n,\beta}(\bm{\theta}) = \frac{1}{n}\sum_{i=1}^n \left[\int f_{i,\bm{\theta}}^{1+\beta}d\lambda - \left(1+\frac{1}{\beta}\right) f_{i,\bm{\theta}}^\beta(y_i) + \frac{1}{\beta}\right],
\end{equation}
over $\bm{\theta}\in\Theta$. In the above loss function, we have omitted the integral from the third term of $d_\beta(\cdot,\cdot)$, 
as it is independent of $\bm{\theta}$, and thus, does not contribute anything in the optimization process. 
The most appealing fact about the DPD is that we can avoid any non-parametric smoothing (e.g., kernel estimate) 
for $\widehat{g_i}$ in the construction of the above objective function. 
Further, it can be shown that, at $\beta = 0$, the above objective function reduces to (in limit) 
$-n^{-1}\sum_{i=1}^n \ln{f_{i,\bm{\theta}}(y_i)}$, and hence the MDPDE at $\beta=0$ is nothing but the MLE. 

The associated minimum DPD functional (MDPDF) of the parameter $\bm{\theta}$, 
at the vector of true distributions $\bm{G} = (G_1,\hdots,G_n)$ is then defined as
\begin{equation} \label{MDPDF}
    \bm{T}_\beta(\bm{G}) = \argmin_{\bm{\theta\in\Theta}} \frac{1}{n}\sum_{i=1}^n d_\beta\left(g_i, f_{i,\bm{\theta}}\right) = \argmin_{\bm{\theta\in\Theta}} H_{n,\beta}^*(\bm{\theta}),
\end{equation}
for $\beta>0$, where
\begin{equation}\label{mdpdf-obj}
H_{n,\beta}^*(\bm{\theta}) = E\left(H_{n,\beta}(\bm{\theta})\right)
= \frac{1}{n}\sum_{i=1}^n \left[\int f_{i,\bm{\theta}}^{1+\beta}d\lambda 
- \left(1+\frac{1}{\beta}\right) \int f_{i,\bm{\theta}}^\beta g_i d\lambda + \frac{1}{\beta}\right].
\end{equation}
It again coincides with the ML functional at $\beta=0$. 

It has been widely discussed in the literature that the tuning parameter $\beta$ regulates 
the balance between robustness and efficiency of the resulting MDPDEs; 
the robustness of the MDPDE increases with increasing values of $\beta$, 
at the cost of asymptotic (pure data) efficiency. 
But the loss in efficiency is not quite significant at small $\beta>0$ under common statistical models; 
see \cite{basu1998robust, basu2011statistical, basu2026} for details.

%\section{Scope, Flexibility, and Generalizability of the rRNet}
\section{Scope and flexibility of the assumed framework}
\label{APP:assumptionms}

Let us now  further examine and elaborate on the regulatory conditions assumed for the development and theoretical studies of the rRNet. 
We demonstrate that our framework is indeed designed with significant flexibility to accommodate 
a wide class of practical NN configurations, including those with non-smooth activation functions, and various noise distributions.

%\subsection{Architectural Flexibility of Allowed NN Models}
\subsection{Assumptions on the NN model architecture}
\label{APP:NN_assumptionms}

We have established all theoretical properties of the proposed rRNet under highly flexible assumptions on 
the NN model $\mu(\cdot, \bm{\theta})$ that hold for a wide class of practical NN architectures, well beyond simple MLPs. 
Recall that  (N0) is the base assumption justifying the identifiability of the rRNet target,
while Assumption (N1) ensures its algorithmic stability and convergence. 
Assumption (N2) is additionally required for non-smooth NN models to study local robustness of the rRNet 
via uniform smooth approximation, an approach that is also quite common  in the literature on non-smooth NN learning. 
These assumptions are close to minimal, and the following lemma provides explicit and verifiable sufficient conditions, 
at the layer level, under which a general NN model satisfies (N0)--(N2). 
The proof follows directly from standard results in NN theory 
\citep[see, e.g.,][]{hornik1989multilayer,goodfellow2016deep,kratsios2021universal,petersen2024mathematical}, 
and is therefore omitted.

\begin{lemma}\label{LEM:NN1}
Let $\mu(\bm{x}, \bm{\theta})$ be a general NN model obtained as a finite composition of layers 
drawn from the following primitives: affine maps, coordinatewise activations, pooling or aggregation operators 
(e.g., sum, average, max, softmax), residual connections, and attention mechanisms with continuous score functions.
%where each layer is one of the following primitives:
%\begin{itemize}
%	\item affine maps in inputs and parameters;
%	\item coordinatewise activation functions;
%	\item pooling or aggregation operators (e.g. sum, average, max, softmax);
%	\item residual or skip connections;
%	\item attention mechanisms with continuous score functions.
%\end{itemize}
Then, the following properties hold for $\mu$. 
\begin{itemize}
	\item[(i)]  If each primitive layer is Borel measurable in inputs and parameters,
	the regression function $\bm{x}\mapsto\mu(\bm{x}, \bm{\theta})$ is measurable for each $\bm{\theta}\in\Theta$.
	\item[(ii)] Suppose that two parameter values yield identical layer outputs almost surely only through explicit, 
	known layerwise symmetries, and that all global non-identifiabilities are compositions of these symmetries.
	Together with the measurability condition in Part (i), this implies (N0).
	
	\item[(iii)] If, for almost every $\bm{x}\in\mathcal{X}$, each primitive layer is locally Lipschitz continuous in its parameters,
	then $\mu$ satisfies Assumption (N1). 
	
	\item[(iv)] If each non-smooth primitive admits a sequence of $\mathcal{C}^2$ functions 
	converging uniformly on compact subsets of its domain, then  $\mu$ satisfies Assumption (N2). 	
\end{itemize}
\end{lemma}

The above lemma shows the broad applicability of rRNet for robust learning across almost all practically relevant NN architectures,
including MLPs, CNNs, residual networks, and also attention-based models. 
For the special case of MLPs, the primitives reduce to affine maps and layer-wise activation functions, 
and Lemma \ref{LEM:NN1} provides easily verifiable sufficient conditions on the chosen activations.
In fact, widely used activation functions, including linear/identity, sigmoid, tanh, ReLU, Leaky ReLU, ELU, GELU, and Softplus,  
satisfy the requirements of (N0) and (N1).  

Particularly, Part (ii) of Lemma \ref{LEM:NN1} formalizes the fact that standard NN parameterizations are non-identifiable 
only through explicit architectural symmetries, such as permutations of hidden units or channels, 
compensating rescalings across adjacent layers, and permutations of attention heads \citep{bloem2020probabilistic,ran2017parameter}. 
No additional unidentified equivalences are allowed. This condition excludes degenerate cases such as deep linear networks 
with continuous symmetries or hard parameter sharing beyond architectural constraints.

For NN models constructed using non-smooth activation functions,
Assumption (N2) requires that the activation be expressible as the uniform limit of a sequence of $\mathcal{C}^2$ functions $\sigma_m(z)$ 
such that the corresponding sequence of NN models ($\{\mu_m\}$) preserves the symmetry group $\mathcal{G}$ specified in (N0).
For the commonly used ReLU activation, which satisfies the local Lipschitz condition in (N1), 
a standard choice is the Softplus function \citep{zheng2015improving} given by  $\sigma_m(z) = \frac{1}{m} \log(1 + e^{mz})$ . 
As $m \to \infty$, $\sigma_m(z)$ converges uniformly to $\text{ReLU}(z)$. 
Since Softplus is $\mathcal{C}^\infty$ and strictly monotonic, 
it preserves sign and permutation symmetries of the weights \citep{bona2023parameter}, thereby satisfying both (N0) and (N2) simultaneously.

On another interesting note, if the symmetry group $\mathcal{G}$ in (N0) acts linearly on $\Theta$ 
(e.g., as a finite-dimensional representation), then (N0) and (N1) together imply (N2).  
In this case, the smoothing sequence $\{\mu_m\}$ may be constructed directly via convolution, 
as in the proof of Lemma \ref{LEM:S0}; see also \cite{chen1995smoothing}.

Finally, local robustness of rRNet against outliers/contamination in feature spaces requires uniform boundedness of 
the parameter gradient of the NN model $\mu$, which can be ensured under mild interpretable layer-level conditions as well. 
In particular, if each layer of the network is Lipschitz continuous in its parameters with uniformly bounded Lipschitz constants, 
and the parameter space is restricted to a bounded (or norm-controlled) subset, 
then the overall network, as a finite composition of such layers, admits a finite global Lipschitz constant with respect to its parameters. 
Consequently, the gradient $|\nabla_{\boldsymbol{\theta}}\mu(\boldsymbol{x},\boldsymbol{\theta})|$ 
is uniformly bounded on compact subsets of the input and parameter spaces. 
Such bounds are again standard in the analysis of NNs and follow from classical results relating Lipschitz continuity 
to gradient boundedness and robustness \citep{virmaux2018lipschitz,wang2023direct}.

%\subsection{On permissible error distributions}
\subsection{Assumptions on the error distribution}
\label{APP:error_assumptionms}

Throughout the paper our main assumption on the error distribution has been (A0),
which alone imposes required shape and tail constraints on the error density $f$ for the validity of rRNet. 
In the following Lemma, we list down important implications of (A0) on $f$,   
which has been used throughout our theoretical analyses.

\begin{lemma}[Implications of A0]\label{LEM:A0}
If the error density $f$ satisfies Assumption (A0), then the following results hold.
	\begin{enumerate}
	\item[(i)] $f$ is unimodal (with the mode being unique up to the null sets) and has at most exponential tails. 
	Specifically, there exist constants $a>0$ and $b\in\mathbb{R}$ such that $f(s) \le e^{-a|s| + b}$ for all $s\in\mathbb{R}$.

	\item[(ii)] $f$ has finite moments of all orders and a finite moment generating function in a neighborhood of zero.

	\item[(iii)] The score function $u(s) = f'(s)/f(s)$ exists a.e.~on $\mathbb{R}$ and is monotone non-increasing.

	\item[(iv)] $u(s)$ is bounded if and only if $f$ has asymptotically exponential tails. 
	However, the function $s u(s)$ is always unbounded.

	\item[(v)] $f$ is globally Lipschitz on $\mathbb{R}$. But, $\log f$ is so only when $u$ is bounded a.e.

	\item[(vi)] For any $\beta > 0$ and $i, j \in \{0, 1, 2, \ldots\}$,  
	$C^{(\beta)}_{i,j} := \int_{\mathbb{R}} s^i u^j(s) f^{1+\beta}(s) ds<\infty$. \\ 
	In particular, $C^{(\beta)}_{i,j} > 0$ if both $i, j$ are even, but 
	$$
	C^{(\beta)}_{i,1} = -\frac{i}{1+\beta} C^{(\beta)}_{i-1,0} ~~~\mbox{ for all }~i\geq 1.
	$$  
	Further, if $f$ is additionally symmetric, $C^{(\beta)}_{i,j} = 0$ whenever $(i+j)$ is odd.  
	
	\item[(vii)] $u'(x)$ exists a.e., and is non-positive. Further, %for  any $i= 0, 1, 2, \ldots$, and $\beta > 0$, 
	 we have 
	$$ 
	\int_{\mathbb{R}} s^i u'(s) f^{1+\beta}(s) ds = - i C^{(\beta)}_{i-1,1} - (1+\beta)C^{(\beta)}_{i,2},
	~~~\mbox{ for all }~ i\geq 0, ~~\beta>0.
	$$
	
	\item[(viii)] $f$ satisfies a Poincaré inequality: there exists a finite constant $C_f > 0$ such that 
	for any locally Lipschitz function $g: \mathbb{R} \to \mathbb{R}$, we have 
\begin{equation}
	\int_{\mathbb{R}} g(s)^2 f(s) ds - \left( \int_{\mathbb{R}} g(s) f(s) ds \right)^2 \le C_f \int_{\mathbb{R}} |g'(s)|^2 f(s) ds.
\end{equation}
	\end{enumerate}
\end{lemma}
\begin{proof}
The following proofs rely primarily on the properties of log-concave densities \citep[see,e.g.,][]{saumard2014log}, 
since $f$ is so under (A0).
%By (A0), $f$ is log-concave with $h = \log f$ being strictly concave a.e. 

\noindent
Part (i): \\
The concavity of $\log f$  implies that its superlevel sets $\{s : f(s) \geq c\}$ are intervals. 
Hence $f$ increases up to a mode and decreases thereafter. 
Further, strict log-concavity a.e. ensures the uniqueness of the mode up to null sets.

The second part follows from Theorem 5.1 of \cite{saumard2014log}. In particular, since $\log f$ is concave, 
it must be bounded above by its tangent lines, i.e., $\log f(s) \le -a|s| + b$ for some $a > 0, b \in \mathbb{R}$ and all $s\in\mathbb{R}$,
implying the desired result. 

\smallskip
\noindent
Part (ii):\\
The exponential tail decay established in Part (i) ensures that the integral $\int |\varepsilon|^k f(\varepsilon) d\varepsilon$ 
converges for all $k > 0$. Since the tails decay at least as fast as $e^{-a|\varepsilon|}$, 
the moment generating function $M(t) = E[e^{t\varepsilon}]$ exists and is finite for all $|t| < a$.

\smallskip
\noindent
Part (iii): \\
By the Rademacher’s theorem for concave functions, $\log f$ is differentiable a.e., and hence $u$ exists a.e.
Moreover, by defining properties of a concave function, the derivative $u$ of the concave function $\log f$ is monotone non-increasing.

\smallskip
\noindent
Part (iv): \\
Let us first assume that $|u(s)|\le M$ a.e.\ for some $M>0$. Then, for $s>s_0$, we have 
\[
\log f(s)-\log f(s_0)=\int_{s_0}^s u(t)\,dt,
~~\Rightarrow~~
-M(s-s_0)\le \log f(s)-\log f(s_0)\le 0,
\]
since $u(t)\le0$ eventually (for large enough $t$) by integrability of $f$. 
So, for large $s$, we get 
\begin{eqnarray}
c_1 e^{-M s} \le f(s) \le c_2 e^{-M s},
\label{Eq:a1}
\end{eqnarray}
for suitable constants $c_1, c_2>0$. Thus, $f$ has asymptotically exponential tails.

Conversely, if $f$ satisfies (\ref{Eq:a1}) for some $c_1, c_2, M >0$ and all $s\geq s_0$, 
then we have $|\log f(s) + M s |\le C$ for some constant $C>0$ and all $s\geq s_0$. 
Thus, the function $\psi(s) = \log f(s) + Ms$ is both bounded and concave on $[s_0, \infty)$, 
and hence it's derivative is bounded a.e. This, in turn, implies the boundedness of $u$.

For the second part, let us assume that, if possible, $|s u(s)|\le K$ a.e.
Then, we have $|u(s)| \le \frac{K}{|s|}$ for all large $|s|$. Integrating, we get 
\[
\left| \log f(s)-\log f(s_0) \right|
\le K \int_{s_0}^s \frac{dt}{t}
= K\log\frac{s}{s_0},
\]
which implies $f(s)\ge C s^{-K}$ for all large $s$ with some constant $C>0$. 
This contradicts the tail behavior of $f$ as proved in Part (i). Hence, $s u(s)$ must be unbounded.

\smallskip
\noindent
Part (v): \\
Since $f'(s)=u(s)f(s)$ a.e.\ and $f$ has exponential tails, $|f'(s)|$ is uniformly bounded, implying that $f$ is globally Lipschitz. 
The function $\log f$ is globally Lipschitz if and only if its derivative $u(s)$ is bounded a.e.

\smallskip
\noindent
Part (vi): \\
Since $f$ has at most exponential tails from Part~(i), so it $f^{\beta}$ for any $\beta>0$. 
Further, for any $j \ge 0$, the term $|u^j(s) f(s)|$ is bounded by some polynomial $P_j(|s|)$ of order at most $j$, 
because $\log f$ is concave. Thus, for any $i, j \geq 0$, we get 
\[
|s^i u^j(s) f^{1+\beta}(s)| = |s^i| |u^j(s) f(s)| | f^{\beta}(s)| \leq |s|^i P_j(|s|) |f^\beta(s)|
\le C_{i,j} (1+|s|)^{i+j} e^{-c|s|},
\]
for some $c>0$. Since polynomial terms are dominated by exponential decay, 
the integrand in the definition of $C^{(\beta)}_{i,j}$ is integrable over $\mathbb{R}$, yielding
$C^{(\beta)}_{i,j} < \infty.$

Next, if both $i$ and $j$ are even, then $s^i\ge0$, $u^j(s)\ge0$ for all $s$. These together with the fact that $f^{1+\beta}(s)>0$
for all $s\in\mathbb{R}$ implies $C^{(\beta)}_{i,j} > 0$.

Now, to prove the required identity for $j=1$, we note that 
\[
\frac{d}{ds}f^{1+\beta}(s)
= (1+\beta) f^{\beta}(s) f'(s)
= (1+\beta) u(s) f^{1+\beta}(s)
\quad \text{a.e.}
\]
Thus, for any $i\geq 1$, we get using integration by parts that
\begin{eqnarray}
C^{(\beta)}_{i,1}
= \int_{\mathbb{R}} s^i u(s) f^{1+\beta}(s)\,ds
&=& \frac{1}{1+\beta} \int_{\mathbb{R}} s^i d(f^{1+\beta}(s))
\nonumber\\
&=& \frac{1}{1+\beta} (-i) \int_{\mathbb{R}} s^{i-1} f^{1+\beta}(s)\,ds
= -\frac{i}{1+\beta} C^{(\beta)}_{i-1,0}.
\nonumber
\end{eqnarray}
Here integration by parts is justified since $s^i f^{1+\beta}(s)\to0$ as $|s|\to\infty$ by exponential tails of $f$ (as proved in Part (i)).

Finally, if $f$ is symmetric about $0$, then $u(s)$ is antisymmetric, i.e., $u(-s) = -u(s)$ a.e.~on $\mathbb{R}$.
Consequently, whenever $i+j$ is odd, the integrand $s^i u^j(s) f^{1+\beta}(s)$ is an odd function for any $\beta\geq 0$, 
implying $C^{(\beta)}_{i,j} = 0$.

\smallskip\noindent
Part (vii): \\
Since $\log f$ is concave, it is twice differentiable a.e.~by Alexandrov's Theorem, and is also non-positive by concavity. 
The recurrence relation follows by a similar argument as in the proof of Part (vi), using the integration by parts. 
% Use \frac{d}{ds}(s^i f^{1+\beta}(s) u(s))

\smallskip\noindent
Part (viii): \\
This is a one-dimensional case of the Brascamp–Lieb inequality; see Proposition 10.1 of \cite{saumard2014log}.
\end{proof}

We can see from the above lemma that Assumption (A0) is enough to support all theoretical analyses of the paper.
Most common error densities, such as Gaussian, Gompertz, logistic, etc. satisfies (A0).
All these densities are, in fact, strictly log-concave everywhere on $\mathbb{R}$. 
However, (A0) requires strictness only a.e., also allowing non-smooth error densities, e.g., Laplace. 
Additionally, all these common error densities are symmetric around zero;
this symmetry of $f$ is a sufficient addition to (A0) to ensure $C_{1,2}^{(\beta)} = 0$ for any $\beta\geq 0$ (Part (vi) of the above lemma)
which is required for explicitly computing the IF of the parameter estimates in rRNet.
This additional symmetry of $f$ also ensures $\bm{\theta}$-$\sigma$ decoupling, 
thereby justifying  the desired performance of the alternating minimization in rRNet algorithm.

\section{Proofs of the results from the main paper}\label{pf}

\subsection{Proof of Theorem \ref{THM:conv}}

\noindent
Part (a): \\
Fix a $k\geq 1$. From (SO2) and the sufficient descent property \eqref{EQ:SO1.1} in (SO1),
we get 
$$
\mathcal{L}_{n,\beta}\left(\widehat{\bm{\theta}}^{(k+1)},\widehat{\sigma}^{(k+1)}\right)
\leq \mathcal{L}_{n,\beta}\left(\widehat{\bm{\theta}}^{(k+1)},\widehat{\sigma}^{(k)}\right) 
\leq \mathcal{L}_{n,\beta}\left(\widehat{\bm{\theta}}^{(k)},\widehat{\sigma}^{(k)}\right)
- c \left|\left|\widehat{\bm{\theta}}^{(k+1)} - \widehat{\bm{\theta}}^{(k)}\right|\right|^2.
$$
Thus, the monotone descent property holds since $c> 0$.

\noindent
Part (b): \\
The sequence of DPD-loss values are monotone by Part (a). 
Additionally, the loss function $\mathcal{L}_{n,\beta}(\bm{\theta}, \sigma)$ is bounded below under (A0)
since $\sigma$ is bounded below. Thus, this sequence must have a finite limit, say $\mathcal{L}_{\infty}$. 

\noindent
Part (c): \\
First note that, from the proof of Part (a), for every $k\geq 1$, we have 
$$
c \left|\left|\widehat{\bm{\theta}}^{(k+1)} - \widehat{\bm{\theta}}^{(k)}\right|\right|^2
\leq \mathcal{L}_{n,\beta}\left(\widehat{\bm{\theta}}^{(k)},\widehat{\sigma}^{(k)}\right)
- \mathcal{L}_{n,\beta}\left(\widehat{\bm{\theta}}^{(k+1)},\widehat{\sigma}^{(k+1)}\right).
$$
Summing over $k = 0, 1, 2, \ldots$, and using Part (b), we get 
$$
c \sum_{k=1}^\infty \left|\left|\widehat{\bm{\theta}}^{(k+1)} - \widehat{\bm{\theta}}^{(k)}\right|\right|^2
\leq \mathcal{L}_{n,\beta}\left(\widehat{\bm{\theta}}^{(0)},\widehat{\sigma}^{(0)}\right)
- \mathcal{L}_{\infty} <\infty.
$$
Since $c>0$, this implies that $ \left|\left|\widehat{\bm{\theta}}^{(k+1)} - \widehat{\bm{\theta}}^{(k)}\right|\right| \rightarrow 0$
as $k\rightarrow\infty$. 
 
Now, consider a limit point $(\bm{\theta}_\infty, \sigma_\infty)$ of the sequence 
$\{(\widehat{\bm{\theta}}^{(k+1)},\widehat{\sigma}^{(k+1)})\}_{k\geq 1}$. 
Then, there exists a subsequence $\{(\widehat{\bm{\theta}}^{(k_j+1)},\widehat{\sigma}^{(k_j+1)})\}_{j\geq 1}$
such that 
$$\left(\widehat{\bm{\theta}}^{(k_j+1)},\widehat{\sigma}^{(k_j+1)}\right) \rightarrow (\bm{\theta}_\infty, \sigma_\infty)
~~~\mbox{  as} ~j \rightarrow\infty.
$$ 
Working along the subsequence in $\{k_j\}$,  we get from the condition \eqref{EQ:SO1.2} in (SO1): 
$$	
||\bm{g}_{k_j}|| \leq c_k' \left|\left|\widehat{\bm{\theta}}^{(k_j+1)} - \widehat{\bm{\theta}}^{(k_j)}\right|\right|
\rightarrow 0, ~~\mbox{ as } j\rightarrow\infty,  
$$
for $\bm{g}_{k_j} \in \partial_{\bm{\theta}}\mathcal{L}_{n,\beta}(\widehat{\bm{\theta}}^{(k_j)}, \widehat{\sigma}^{(k_j)})$ 
for all $j\geq 1$. Thus, we get $||\bm{g}_{k_j}|| \rightarrow 0$ and hence $\bm{g}_{k_j} \rightarrow 0$ as $j\rightarrow\infty$. 
Since $\mathcal{L}_{n, \beta}$ is locally Lipschitz in its arguments under (A0) and (N1), 
by Proposition 2.6.2 of \cite{clarke1990optimization}, we get 
$\bm{0} \in \partial_{\bm{\theta}}\mathcal{L}_{n,\beta}(\bm{\theta}_\infty, \sigma_\infty)$.

For the second part, note that (SO2) implies 
$0\in \partial_\sigma \mathcal{L}_{n,\beta}(\widehat{\bm{\theta}}^{(k)}, \widehat{\sigma}^{(k)})$ for each $k\geq1$. 
But, since $\mathcal{L}_{n, \beta}$ is locally Lipschitz in its arguments under (A0) and (N1), 
by Proposition 2.1.5 of \cite{clarke1990optimization}, $\partial_\sigma \mathcal{L}_{n,\beta}$ is upper semi-continuous. 
So, taking limit along the subsequence in $\{k_j\}$, we get  
$0 \in {\partial_\sigma}\mathcal{L}_{n,\beta}(\bm{\theta}_\infty, \sigma_\infty)$, completing the proof.

\subsection{Proof of Proposition \ref{PROP:theta-updtae-dg1}}

Let us fix any input to the algorithm and fix $k\geq 0$.
First note that, under Assumption (A0) and (N1), the DPD loss function, 
viewed as a function of $\bm{\theta}$ alone with $\sigma=\widehat{\sigma}^{(k)}$, 
is upper-$\mathcal{C}^2$ around $\widehat{\bm{\theta}}^{(k)}$, as per Definition 10.29 of \cite{rockafellar1998variational}. 
Then, by Proposition 3.2 of \cite{aragon2025nonmonotone}, we get $\bm{g}_{k} \in 
\partial_{\bm{\theta}}\mathcal{L}_{n,\beta}(\widehat{\bm{\theta}}^{(k)}, \widehat{\sigma}^{(k)})$ and a constant $C_\kappa\geq 0$ such that 
\begin{eqnarray}
	\mathcal{L}_{n,\beta}\left(\widehat{\bm{\theta}}^{(k+1)},\widehat{\sigma}^{(k)}\right) 
	\leq \mathcal{L}_{n,\beta}\left(\widehat{\bm{\theta}}^{(k)},\widehat{\sigma}^{(k)}\right)
	+ \big<\bm{g}_k,  \widehat{\bm{\theta}}^{(k+1)} - \widehat{\bm{\theta}}^{(k)}\big>
	+ C_\kappa \left|\left|\widehat{\bm{\theta}}^{(k+1)} - \widehat{\bm{\theta}}^{(k)}\right|\right|^2.
	\label{EQ:pf.1}
\end{eqnarray}
Now, by the sub-gradient descent rule (\ref{EQ:theta-update-gd1}), 
we get $\widehat{\bm{\theta}}^{(k+1)} - \widehat{\bm{\theta}}^{(k)} = -\alpha_k \bm{g}_k$.
Substituting it in (\ref{EQ:pf.1}) and simplifying,  we get
\begin{eqnarray}
	\mathcal{L}_{n,\beta}\left(\widehat{\bm{\theta}}^{(k+1)},\widehat{\sigma}^{(k)}\right) 
	\leq \mathcal{L}_{n,\beta}\left(\widehat{\bm{\theta}}^{(k)},\widehat{\sigma}^{(k)}\right)
	- c \left|\left|\widehat{\bm{\theta}}^{(k+1)} - \widehat{\bm{\theta}}^{(k)}\right|\right|^2,
\nonumber
\end{eqnarray}
where $c = \frac{1}{\alpha_k} - C_\kappa>0$ provided we choose the step-size $\alpha_k < C_\kappa^{-1}$. 
This completes the proof of the first part (\ref{EQ:SO1.1}) of (SO1). 

Further, by the update rule in (\ref{EQ:theta-update-gd1}), it immediately follows that 
$
||\bm{g}_k|| = \frac{1}{\alpha_k}||\widehat{\bm{\theta}}^{(k+1)} - \widehat{\bm{\theta}}^{(k)}||, 
$
and hence \eqref{EQ:SO1.2} of (SO1) holds for any choice of $0< c_k' < \frac{1}{\alpha_k}$.

\subsection{Proof of Theorem \ref{THM:IF-MDPDF}}

We start with the general formula for the IFs as given in (\ref{EQ:IF-MDPDF-gen})
and simplify each terms to get the desired formulas presented in the theorem. 
Let us recall the definition of $\mathcal{I}_\beta(y,\bm{x}|\bm{\eta})$ from \eqref{EQ:I},
which depends on $\psi_{1,\beta}(s)= u(s) f^\beta(s)$, $\psi_{2,\beta} = (1 + s u(s)) f^{\beta}(s)$.
Thus, whenever the model assumption is true, i.e., $G_i = F_{i,\bm{\theta}}$, 
the density of $Y$ given $\bm{x}_i$ is $f_{i, \bm{\theta}}(y) = \frac{1}{\sigma} f\left(\frac{y - \mu(\bm{x},\bm{\theta})}{\sigma}\right)$
for all $i=1, \ldots, n$. Thus, under (A0), we get 
\begin{eqnarray}
E_{F_{i,\bm{\theta}}}\left[\mathcal{I}_\beta(Y,\bm{x}_i|\bm{\eta}) \right] = \left(
%\begin{array}{c}
[\nabla_{\bm{\theta}}\mu(\bm{x}_i, \bm{\theta})]^\top C^{(\beta)}_{0,1} 
%\\\\
~~~
I_0
%\end{array}
\right)^\top,
%= \left(\begin{array}{c}
%	0\\ C^{(\beta)}_{0,0}
%\end{array}\right), 
\end{eqnarray}
where $C^{(\beta)}_{0,1}=0$ since $f$ has sub-exponential tails (Lemma \ref{LEM:A0}(i))
and the second integral is computed as (again using Lemma \ref{LEM:A0}(vi)) 
%$C^{(\beta)}_{0,1} = \int u(s)f^{1+\beta}(s)ds =0$ and 
$$
I_0 = \int (1 + s u(s)) f^{1+\beta}(s) ds = C^{(\beta)}_{0,0} + C^{(\beta)}_{1,1} 
= C^{(\beta)}_{0,0} - \frac{1}{1+\beta}C^{(\beta)}_{0,0} = \frac{\beta}{1+\beta}C^{(\beta)}_{0,0}>0.
$$
Next, in order to find $\boldsymbol{\Psi}_n(\boldsymbol{\eta})$, we note that (through standard differentiation and integration) 
\begin{eqnarray*}
&& E_{F_{i,\bm{\theta}}}\left[\nabla_{\bm{\eta}}\mathcal{I}_\beta(Y,\bm{x}_i|\bm{\eta}) \right]
\\
&=& - \frac{1}{\sigma} \begin{bmatrix}
	\begin{array}{cc}
		[\nabla_{\bm{\theta}}\mu(\bm{x}_i, \bm{\theta})]^\top [\nabla_{\bm{\theta}}\mu(\bm{x}_i, \bm{\theta})] I_1		
		- 	 \sigma[\nabla_{\bm{\theta}}^2 \mu(\bm{x}_i, \bm{\theta})]C^{(\beta)}_{0,1} 
 		&  [\nabla_{\bm{\theta}} \mu(\bm{x}_i, \bm{\theta})] I_2
 		\\\\
 		[\nabla_{\bm{\theta}} \mu(\bm{x}_i, \bm{\theta})]^\top I_2 
 		&  I_3
\end{array}
\end{bmatrix},
\end{eqnarray*}
where we again compute the integrals, under (A0), from Lemma \ref{LEM:A0} as follows:
\begin{eqnarray*}
C^{(\beta)}_{0,1} &=& \int u(s) f^{1+\beta}(s) ds = 0, ~~~~\mbox{ since $f$ has sub-exponential tails, }
\nonumber\\
I_{1} &=& \int (u'(s)+\beta u^2(s)) f^{1+\beta}(s) ds = -(1+\beta)  C^{(\beta)}_{0,2} + \beta C^{(\beta)}_{0,2} = - C^{(\beta)}_{0,2},
~~
\nonumber\\
I_{2} &=&  \int (s u'(s) + \beta s u^2(s))f^{1+\beta}(s) = - C^{(\beta)}_{0,1} - (1+\beta) C^{(\beta)}_{1,2} + \beta C^{(\beta)}_{1,2}
= - C^{(\beta)}_{1,2},
~~
\nonumber\\
I_3 &=& \int (\beta s u(s) + \beta s^2u^2(s) + s u(s) + s^2 u'(s))f^{1+\beta}(s) ds
\nonumber\\
&=& \beta C^{(\beta)}_{1,1} + \beta C^{(\beta)}_{2,2} + C^{(\beta)}_{1,1} - 2C^{(\beta)}_{1,1} - (1+\beta) C^{(\beta)}_{2,2}
\nonumber\\
&=& (\beta -1) C^{(\beta)}_{1,1} - C^{(\beta)}_{2,2} = - \left[C^{(\beta)}_{2,2} - \frac{1-\beta}{1+\beta} C^{(\beta)}_{0,0}\right]
= - \widetilde{C}^{(\beta)}_{2,2}.
\nonumber
\end{eqnarray*}

Now, under the assumption of the theorem, we have $C^{(\beta)}_{1,2}=0$, and hence we can simplify 
$\boldsymbol{\Psi}_n(\boldsymbol{\eta})$ to have the form
\begin{eqnarray*}
\boldsymbol{\Psi}_n(\boldsymbol{\eta})	&=&  \frac{1}{n\sigma} \begin{bmatrix}
		\begin{array}{cc}
			\bm{\dot{\mu}}_n(\bm{\theta})^\top \bm{\dot{\mu}}_n(\bm{\theta}) C^{(\beta)}_{0,2}
			&  \bm{0}
			\\\\
			\bm{0}^\top &  n \widetilde{C}^{(\beta)}_{2,2}
		\end{array}
	\end{bmatrix}.
\end{eqnarray*}
%with $\widetilde{C}^{(\beta)}_{2,2} = \left[C^{(\beta)}_{2,2} - \frac{1-\beta}{1+\beta} C^{(\beta)}_{0,0}\right]$.

Substituting all these in the expression of IF in (\ref{EQ:IF-MDPDF-gen}) and simplifying, we get 
\begin{eqnarray}
\frac{1}{n\sigma} [\bm{\dot{\mu}}_n(\bm{\theta})^\top \bm{\dot{\mu}}_n(\bm{\theta})] C^{(\beta)}_{0,2} 
~IF_i(t,\bm{T}_\beta^{\bm{\theta}}, \overline{F}_{n,\bm{\theta}})   
&=& -\frac{1}{n} \psi_{1,\beta}\left(\frac{t - \mu(\bm{x}_i,\bm{\theta})}{\sigma}\right)  \nabla_{\bm{\theta}}\mu(\bm{x}_i,\bm{\theta}),
\nonumber\\
\frac{1}{\sigma} \widetilde{C}^{(\beta)}_{2,2} ~IF_i(t,{T}_\beta^{\sigma}, \overline{F}_{n,\bm{\theta}})  
&=& - \frac{1}{n} \left[\psi_{2,\beta}\left(\frac{t - \mu(\bm{x}_i,\bm{\theta})}{\sigma}\right)  
- \frac{\beta C^{(\beta)}_{0,0}}{1+\beta}\right].
\nonumber
\end{eqnarray}
Then the theorem follows from the above two equations, by noting that $\bm{v}_n(\bm{\theta})$ must belongs to 
$\in Ker(\bm{\dot{\mu}}_n(\bm{\theta})^\top \bm{\dot{\mu}}_n(\bm{\theta})) = Ker(\bm{\dot{\mu}}_n(\bm{\theta}))$.

\subsection{Proof of Corollary \ref{COR:IF-predictor}}

Let us define $\mu_{n,\beta,\epsilon}^*(\boldsymbol{x}) =  \mu(\bm{x}, \bm{T}_\beta(\overline{G}_{n,\epsilon}))$ 
for any $\bm{x}\in\mathcal{X}$, so that we have 
\begin{eqnarray}
	IF_i(t, \mu_{n, \beta}^\ast(\bm{x}), \overline{F}_{n,\bm{\theta}}) 
	= \lim_{\epsilon\to 0} \frac{\mu_{n, \beta, \epsilon}^\ast(\bm{x})-\mu_{n, \beta, 0}^\ast(\bm{x})}{\epsilon}  
=  \frac{\partial}{\partial\epsilon} \mu_{n, \beta, \epsilon}^\ast(\bm{x})\bigg|_{\epsilon=0}.
	\nonumber
\end{eqnarray}
By the chain rule of differentiation, we then get 
\begin{eqnarray}
	IF_i(t, \mu_{n, \beta}^\ast(\bm{x}), \overline{F}_{n,\bm{\theta}}) 
	& = &  \frac{\partial}{\partial\epsilon} \mu_{n, \beta, \epsilon}^\ast(\bm{x})\bigg|_{\epsilon=0}
	 =   [\nabla_{\bm{\theta}}\mu(\bm{x}, \bm{T}_\beta(\overline{G}_{n}))]^\top 
	\frac{\partial}{\partial\epsilon} \bm{T}_\beta(\overline{G}_{n,\epsilon})\bigg|_{\epsilon=0}
	\nonumber\\
& = &  [\nabla_{\bm{\theta}}\mu(\bm{x}, \bm{\theta})]^\top 		 IF_i(t,\bm{T}_\beta^{\bm{\theta}}, \overline{F}_{n,\bm{\theta}}).
	\nonumber
\end{eqnarray}
Now, using the expression of IF of $\bm{T}_\beta^{\bm{\theta}}$ from Theorem \ref{THM:IF-MDPDF}, 
we get 
\begin{eqnarray}
	IF_i(t, \mu_{n, \beta}^\ast(\bm{x}), \overline{F}_{n,\bm{\theta}}) 
	=  - \frac{\sigma}{C^{(\beta)}_{0,2}} \psi_{1,\beta}\left(\frac{t - \mu(\bm{x}_i,\bm{\theta})}{\sigma}\right) 
	\mathcal{H}_n(\bm{x}, \bm{\theta}) + [\nabla_{\bm{\theta}}\mu(\bm{x},\bm{\theta})]^\top\bm{v}_n(\bm{\theta}),
\nonumber
\end{eqnarray}
for some (possibly non-unique) $\bm{v}_n(\bm{\theta})\in Ker(\bm{\dot{\mu}}_n(\bm{\theta}))$.
The proof of the corollary is then completed by noting that  
$[\nabla_{\bm{\theta}}\mu(\bm{x},\bm{\theta})]^\top\bm{v}_n(\bm{\theta}) = 0$ for any $\bm{x}\in\mathcal{A}$,
which is obvious from the definition of the admissible feature domain $\mathcal{A}$.

\subsection{Proof of Corollary \ref{COR:IF-bounded}}

The proof follows from the expression of the respective IFs 
by noting that both $\psi_{1,\beta}$ and $\psi_{2,\beta}$ are bounded for any error density $f$ satisfying (A0).
And this follows from the fact that, at any given $i, j\geq 0$ and $\beta>0$, (A0) implies
\[
|s^i u^j(s) f^{1+\beta}(s)| = |s^i| |u^j(s) f(s)| | f^{\beta}(s)| 
%\leq |s|^i P_j(|s|) |f^\beta(s)|
\le c (1+|s|)^{i+j} e^{-c|s|},
\]
for some $c>0$,   as shown in the proof of Lemma \ref{LEM:A0}(vi).

%\subsection{Proof of Corollary \ref{COR:IF-bounded0}}

\subsection{Proof of Lemma \ref{LEM:S0}}

This lemma follows from Proposition 5.7 of \cite{saumard2014log}. 
Here we present a constructive proof by defining $\{f_m\}$ explicitly for a given $f$ satisfying (A0). 

For this purpose, let us denote  $\psi_m(s) = \sqrt{\frac{m}{2\pi}} \exp(-\frac{ms^2}{2})$ 
be the density of a Gaussian distribution with mean zero and variance $1/m$ for each $m\geq1$. 
Then, denoting the convolution operator by $\circledast$, we define 
$$
f_m(s) = (f \circledast \psi_m)(s) = \int_{\mathbb{R}} f(s-t) \psi_m(t) dt,~~~m\geq 1.
$$

It is evident from the property of convolution that, since $\psi_m \in \mathcal{C}^\infty$, 
so is $f_m = f \circledast \psi_m$ for each $m\geq 1$, even if $f$ is non-differentiable at some points. 
Further, since $f$ is strictly log-concave a.e. by (A0) and $\psi_m$ is strictly log-concave everywhere,
their convolution $f_m$ is strongly log-concave on $\mathbb{R}$ by the preservation property of log-concavity \citep{saumard2014log}. 
Finally, under (A0), $f$ is continuous and vanishes at infinity (Lemma \ref{LEM:A0}),
and hence the sequence $f_m = f \circledast \psi_m$ converges uniformly to $f$ as the variance $1/m \to 0$, i.e., as $m\rightarrow\infty$.

For the second part, we note that $\psi_m$ is symmetric around its mean for all $m\geq 1$.
Thus, if $f$ is symmetric around zero, then so is its convolution $f_m$ for all $m\geq 1$,
since the convolution of two symmetric functions is known to be symmetric.

\subsection{Proof of Lemma \ref{LEM:smoothing_func}}\label{APP:smoothing_func}

Let us fix $n\geq 1$ and $\beta> 0$, and denote the compact parameter space for $\bm{\eta}$ as 
$\Theta_\eta = \Theta \times [\sigma_0, \sigma^0]$ with $\Theta$ being the compact parameter space for $\bm{\theta}$.    
Then, using the form of the DPD-loss function, we  get 
\begin{eqnarray}
&&\sup_{\bm{\eta}\in\Theta_\eta}\left|\mathcal{L}_{m,n,\beta}^*(\bm{\eta}) - \mathcal{L}_{n,\beta}^*(\bm{\eta}) \right|
%\leq  \frac{1}{n} \sum_{i=1}^n \int \left|V_{m, \beta}(y,\bm{x}_i|\bm{\eta}) - V_{\beta}(y,\bm{x}_i|\bm{\eta})\right|g_i(y)dy,
\nonumber\\
&\leq&  \frac{1}{\sigma_0^\beta}\Delta_{m, \beta} + \left(1 + \frac{1}{\beta}\right) \frac{1}{\sigma_0^\beta} 
\frac{1}{n} \sum_{i=1}^n \int  \sup_{\bm{\eta}\in\Theta_\eta}\left|f_m^\beta\left(r_{m,i}(y|\bm{\eta})\right)  - 
f^\beta\left(r_i(y|\bm{\eta})\right)\right|g_i(y)dy,~~~~~
\label{EQ:1}
\end{eqnarray}
where $r_i(y|\bm{\eta}) = \frac{y - \mu(\bm{x}_i,\bm{\theta})}{\sigma}$, 
$r_{m,i}(y|\bm{\eta}) = \frac{y - \mu_m(\bm{x}_i,\bm{\theta})}{\sigma}$ and 
$\Delta_{m, \beta} = \left|\int f_m^{1+\beta} - \int f^{1+\beta}\right|$ for each $m\geq 1$ and $i\geq 1$.
%$V_{m, \beta}(y,\bm{x}_i|\bm{\eta})$ is defined as $V_{\beta}(y,\bm{x}_i|\bm{\eta})$ with $f$ and $\mu$ being replaced by $f)m$ and $\mu_m$, respectively. 

Now, note that $\Delta_{m, \beta}\rightarrow 0 $ as $m\rightarrow\infty$ for any $\beta>0$ under (A0),
by an application of the dominated convergence theorem (DCT) along with Lemma \ref{LEM:S0}. 
Further, for the second term in (\ref{EQ:1}), we have 
\begin{eqnarray}
&&	\frac{1}{n} \sum_{i=1}^n \int  \sup_{\bm{\eta}\in\Theta_\eta}\left|f_m^\beta\left(r_{m,i}(y|\bm{\eta})\right)  - 
	f^\beta\left(r_i(y|\bm{\eta})\right)\right|g_i(y)dy
	\nonumber\\
&\leq&	\frac{1}{n} \sum_{i=1}^n \int  \sup_{\bm{\eta}\in\Theta_\eta}\left|f_m^\beta\left(r_{m,i}(y|\bm{\eta})\right)  - 
f^\beta\left(r_{m,i}(y|\bm{\eta})\right)\right|g_i(y)dy
\nonumber\\
&& ~~~~
+ \frac{1}{n} \sum_{i=1}^n \int  \sup_{\bm{\eta}\in\Theta_\eta}\left|f^\beta\left(r_{m,i}(y|\bm{\eta})\right)  - 
f^\beta\left(r_i(y|\bm{\eta})\right)\right|g_i(y)dy
\nonumber\\
&\leq& \sup\limits_{z\in \mathbb{R}}\left|f_m^\beta(z) - f^\beta\left(z\right)\right| +
 \int  \sup_{\bm{\eta}\in\Theta_\eta, i\geq 1}\left|f_m^\beta\left(r_{m,i}(y|\bm{\eta})\right)  - 
f^\beta\left(r_i(y|\bm{\eta})\right)\right|g_i(y)dy
\nonumber\\
&\leq& \sup\limits_{z\in \mathbb{R}}\left|f_m^\beta(z) - f^\beta\left(z\right)\right| +
\frac{K^*}{\sigma_0}  \sup_{\bm{\theta}\in\Theta, i\geq 1}\left|\mu_m(\bm{x}_i, \bm{\theta})  - \mu_m(\bm{x}_i, \bm{\theta})\right|,
	\label{EQ:2}
\end{eqnarray}
for some finite constant $K^*>0$, 
where this last step follows from local Lipschitz property of $f^\beta$ under (A0); see Lemma \ref{LEM:A0}.
Then, as $m \rightarrow \infty$, the first term in (\ref{EQ:2}) goes to zero by Lemma \ref{LEM:S0}
while the second term in (\ref{EQ:2}) goes to zero by Assumption (N2). 
Thus, combining it with \eqref{EQ:1}, we get the desired results on the uniform convergence of the DPD-loss functions, 
proving the first statement in Part (a). 

The proof for the case $\beta=0$ follows in a similar manner, and hence it is skipped. 

The second statement of Part (a), i.e., the convergence of the MDPDFs follows directly
from the uniform convergence of the DPD-loss function by the argmin continuity theorem.
It uses the facts that the parameter space is compact and, under (A0) and (N0), 
the MDPDFs exist uniquely up to symmetry groups by Theorem \ref{THM:identifiability}.

For Part (b), let us note that 
\begin{eqnarray}
\sup_{i\geq 1} |\mu_{m,n,\beta}^*(\bm{x}_i) - \mu_{n,\beta}^*(\bm{x}_i)|
&\leq & \sup_{i\geq 1} |\mu_m(\bm{x}_i, \bm{T}_{m,\beta}(\overline{G}_n)) - \mu_m(\bm{x}_i, \bm{T}_{\beta}(\overline{G}_n))|
\nonumber\\
&&+ \sup_{i\geq 1} |\mu_m(\bm{x}_i, \bm{T}_{\beta}(\overline{G}_n)) - \mu(\bm{x}_i, \bm{T}_{\beta}(\overline{G}_n))|.
~~
\label{EQ:3}
\end{eqnarray}
But, the second term in (\ref{EQ:3}) converges to zero as $m\rightarrow\infty$ by Assumption (N2). 
Also, by the Lipschitz continuity of $\mu_m$ by (N1)--(N2) and the convergence of the MDPDFs from Part (a),
we get that the first term in (\ref{EQ:3}) also goes to zero as $m\rightarrow\infty$. 
This completes the proof of the desired convergence result in (b).

\subsection{Proof of Theorem \ref{THM:BP}}

Let us first consider the case $\beta\in(0, 1]$, and define 
$D_\beta(g, f) = f^{1+\beta} + \left(1+\frac{1}{\beta}\right)f^\beta g + \frac{1}{\beta} g^{1+\beta}$ for any two densities $g$ and $f$, 
so that $d_\beta(g, f) = \int D_\beta(g, f) d\lambda$. 

By the definition of the MDPDFs from (\ref{EQ:MDPDF}), at any $m, n \geq 1$, we know that 
$\bm{\eta}_{m, \beta} =(\bm{\theta}_{m, \beta}^\top, \sigma_{m, \beta})^\top = \bm{T}_\beta(\overline{G}_{n,m, \epsilon})$ 
is a minimizer of the average DPD measure, given by $\frac{1}{n} \sum\limits_{i=1}^n d_\beta(g_{i,m, \epsilon}, f_{i, \bm{\eta}}),$
where $g_{i,m, \epsilon} = (1-\epsilon) f_{i, \bm{\eta}_*} + \epsilon k_{i,m}$ is the density of as $G_{i,m, \epsilon}$ for each $i\geq 1$.
It implies that, for any $m, n \geq 1$,  
\begin{eqnarray}
\frac{1}{n} \sum_{i=1}^n d_\beta(g_{i,m, \epsilon}, f_{i, \bm{\eta}_{m,\beta}}) \leq 
\frac{1}{n} \sum_{i=1}^n d_\beta(g_{i,m, \epsilon}, f_{i, \bm{\eta}_*}).
\label{EQ:bp1}
\end{eqnarray}
We now assume that such a sequence of MDPDFs $\{\bm{\eta}_{m, \beta}\}_{m\geq 1}$, satisfying (\ref{EQ:bp1}), 
exists for which the rRNet functional breaks down at a given $\epsilon>0$, 
i.e., $\{\bm{\eta}_{m, \beta}\}_{m\geq 1}$ satisfies (\ref{EQ:BP_def0}).
Then, our target is to derive a contradiction for $\epsilon<0.5$, proving that the breakdown cannot occur as long as $\epsilon<0.5$,
and hence $\epsilon_\beta^\ast=1/2$.

For this purpose, let us fix an $i\geq 1$ and study the limit of 
$$
d_\beta(g_{i,m, \epsilon}, f_{i, \bm{\eta}_{m,\beta}}) = \int_{A_{i,m,n}} D_\beta(g_{i,m, \epsilon}, f_{i, \bm{\eta}_{m,\beta}}) d\lambda
 + \int_{A_{i,m.n}^c} D_\beta(g_{i,m, \epsilon}, f_{i, \bm{\eta}_{m,\beta}})d\lambda, 
$$
where we define $A_{i,m,n} = \left\{ y: f_{i, \bm{\eta}_*} > \max(k_{i,m}(y), f_{i, \bm{\eta}_{m,\beta}}(y) )\right\}$.
Then, under Assumptions (A0) and (A1), using the DCT suitably, one can prove that \citep[see, e.g.,][]{jana2025bp}

\begin{eqnarray*}
\lim\limits_{m\rightarrow\infty} \left| \int_{A_{i,m,n}} D_\beta(g_{i,m, \epsilon}, f_{i, \bm{\eta}_{m,\beta}}) d\lambda 
- \frac{(1-\epsilon)^{1+\beta}}{\beta}\int f_{i, \bm{\eta}_*}^{1+\beta} d\lambda\right| &=& 0, ~~\mbox{ for all } n\geq N^*,
%\end{eqnarray*}
%and 
\\
%\begin{eqnarray*}
\lim\limits_{m\rightarrow\infty} \left| \int_{A_{i,m,n}^c} D_\beta(g_{i,m, \epsilon}, f_{i, \bm{\eta}_{m,\beta}}) d\lambda 
	- \int D_\beta(\epsilon k_{i,m}, f_{i, \bm{\eta}_{m,\beta}}) d\lambda \right| &=& 0, ~~\mbox{ for all } n\geq N^*,
\end{eqnarray*}
for some large enough $N^*\geq 1$. But, at any $i, m, n \geq1$, we get  
$\int f_{i, \bm{\eta}_*}^{1+\beta}d\lambda = \sigma_*^{-\beta} C^{(\beta)}_{0,0}, $
and 
\begin{eqnarray*}
\int D_\beta(\epsilon k_{i,m}, f_{i, \bm{\eta}_{m,\beta}}) d\lambda 
&=& \frac{1}{\sigma_{m, \beta}^\beta} C^{(\beta)}_{0,0} - \left(1+\frac{1}{\beta}\right) \epsilon 
\int  f_{i, \bm{\eta}_{m,\beta}}^\beta k_{i,m} d\lambda + \frac{\epsilon^{1+\beta}}{\beta} \int k_{i,m}^{1+\beta}d\lambda.
\end{eqnarray*}
Combining all these and averaging over $1=1, \ldots, n$, 
we get for sufficiently large $n\geq N^*$, 
\begin{eqnarray*}
\lim\limits_{m\rightarrow\infty} \frac{1}{n} \sum_{i=1}^n d_\beta(g_{i,m, \epsilon}, f_{i, \bm{\eta}_{m,\beta}})
&=& \frac{(1-\epsilon)^{1+\beta}}{\beta \sigma_*^\beta} C^{(\beta)}_{0,0} + 
\frac{1}{\sigma_{*, \beta}^\beta} C^{(\beta)}_{0,0} %- \left(1+\frac{1}{\beta}\right) \epsilon C_k
 + \frac{\epsilon^{1+\beta}}{\beta} C_{k,0},
\end{eqnarray*}
where $\sigma_{*,\beta} = \lim\limits_{m\rightarrow\infty} \sigma_{m, \beta}$ (may be infinite) and 
$C_{k, 0} = \lim\limits_{m\rightarrow\infty} \frac{1}{n} \sum_{i=1}^n \int k_{i,m}^{1+\beta}d\lambda$, 
which is finite by (A1). In the above, we have also used the fact that, under (A0)--(A1), 
for any fixed $n\geq 1$ and sequence $\{\bm{\eta}_{m,\beta}\}_{m\geq 1}$ satisfying (\ref{EQ:BP_def0}), we have
$$
\lim\limits_{m\rightarrow\infty} \frac{1}{n} \sum_{i=1}^n \int f_{i, \bm{\eta}_{m,\beta}}^\beta k_{i,m} d\lambda = 0.
$$
The proof of this last limit follows by breaking the integral over a suitable compact $B\subset \mathbb{R}$ and $B^c$, 
where each  of them goes to zero by (A1) and (A0), respectively.

Similarly, breaking the integrals over the events $B_{i,m,n} = \left\{ y : k_{i,m}(y) > f_{i, \bm{\eta}_*}(y)\right\}$ and its complement, 
and proceeding similarly, we get 
\begin{eqnarray*}
	\lim\limits_{m\rightarrow\infty} \frac{1}{n} \sum_{i=1}^n d_\beta(g_{i,m, \epsilon}, f_{i, \bm{\eta}_{*}})
	&=&  \left[\frac{(1-\epsilon)^{1+\beta}}{\beta} + 1 - (1-\epsilon) \left(1+\frac{1}{\beta}\right)\right] \frac{C^{(\beta)}_{0,0}}{\sigma_{*}^\beta} + \frac{\epsilon^{1+\beta}}{\beta} C_{k,0}.
\end{eqnarray*}

Therefore, we get a contradiction to (\ref{EQ:bp1}), i.e., breakdown cannot occur, if  
\begin{eqnarray}
	&&
	 \lim\limits_{m\rightarrow\infty} \frac{1}{n} \sum_{i=1}^n d_\beta(g_{i,m, \epsilon}, f_{i, \bm{\eta}_{m,\beta}})>
	\lim\limits_{m\rightarrow\infty} \frac{1}{n} \sum_{i=1}^n d_\beta(g_{i,m, \epsilon}, f_{i, \bm{\eta}_{*}})
	\nonumber\\
\Leftrightarrow && 
\frac{1}{\sigma_{*, \beta}^\beta} C^{(\beta)}_{0,0}  > 
\left[1 - (1-\epsilon) \left(1+\frac{1}{\beta}\right)\right] \frac{C^{(\beta)}_{0,0}}{\sigma_{*}^\beta}
	\nonumber\\
\Leftrightarrow && 
\frac{\sigma_*^\beta}{\sigma_{*, \beta}^\beta}  > 
\left[1 - (1-\epsilon) \left(1+\frac{1}{\beta}\right)\right]
\label{EQ:bp2}
\end{eqnarray}
But, since $0<\beta\leq 1$, we get $\left[1 - (1-\epsilon) \left(1+\frac{1}{\beta}\right)\right]<0$ for all $\epsilon<\frac{1}{2}$.
Then, since the LHS is always non-negative (it can be zero if $\sigma_{m, \beta} \rightarrow\infty$), \eqref{EQ:bp2} holds, 
producing a contadiction to (\ref{EQ:bp1}) for all $\epsilon<\frac{1}{2}$.
In other words, no sequence of rRNet functionals can have breakdown as per (\ref{EQ:BP_def0}) whenever $\epsilon<\frac{1}{2}$,
proving $\epsilon_\beta^* = \frac{1}{2}$ for all $\beta \in(0,1]$.

For the case $\beta=0$, let us note that 
$\bm{\eta}_{m, 0} =(\bm{\theta}_{m, 0}^\top, \sigma_{m, 0})^\top = \bm{T}_0(\overline{G}_{n,m, \epsilon})$ 
is a minimizer of $	\frac{1}{n} \sum_{i=1}^n d_0(g_{i,m, \epsilon}, f_{i, \bm{\eta}_{m,\beta}})$  at any $m, n \geq 1$.
By (\ref{dpd-def}), this minimization problem is equivalent to the minimization of   
\begin{eqnarray}
- \int \log f_{i, \bm{\eta}_{m,\beta}} g_{i,m, \epsilon}
= - (1-\epsilon)\int \log f_{i, \bm{\eta}_{m,\beta}} f_{i,\bm{\eta}_0} - \epsilon \int \log f_{i, \bm{\eta}_{m,\beta}} k_{i,m}.
\end{eqnarray}
Now, consider the contaminating sequence $\{K_{i,m}\}$ to be the sequence of degenerate distributions $\{\wedge_{y_m}\}$
with $y_m \rightarrow\infty$ as $m\rightarrow\infty$. 
Then, for this particular contaminating sequence and any $\epsilon>0$, we have 
\begin{eqnarray}
	- \int \log f_{i, \bm{\eta}_{m,\beta}} g_{i,m, \epsilon}
	= - (1-\epsilon)\int \log f_{i, \bm{\eta}_{m,\beta}} f_{i,\bm{\eta}_0} - \epsilon \log f_{i, \bm{\eta}_{m,\beta}}(y_m)
	\rightarrow - \infty, 
%	~~~\mbox{ as } m \rightarrow\infty.  
\end{eqnarray}
as $m \rightarrow\infty$. 
Thus, the resulting minimizer must break down for this choice of contaminating sequence whenever $\epsilon>0$, 
implying that its ABP is zero.

\section{Additional Figures and Tables}\label{add-results}

We now present all additional figures and tables showing further empirical illustrations, 
in the order they are referenced in the main paper.
\bigskip

\begin{figure}[!h]
	\centering
	\includegraphics[width=\textwidth]{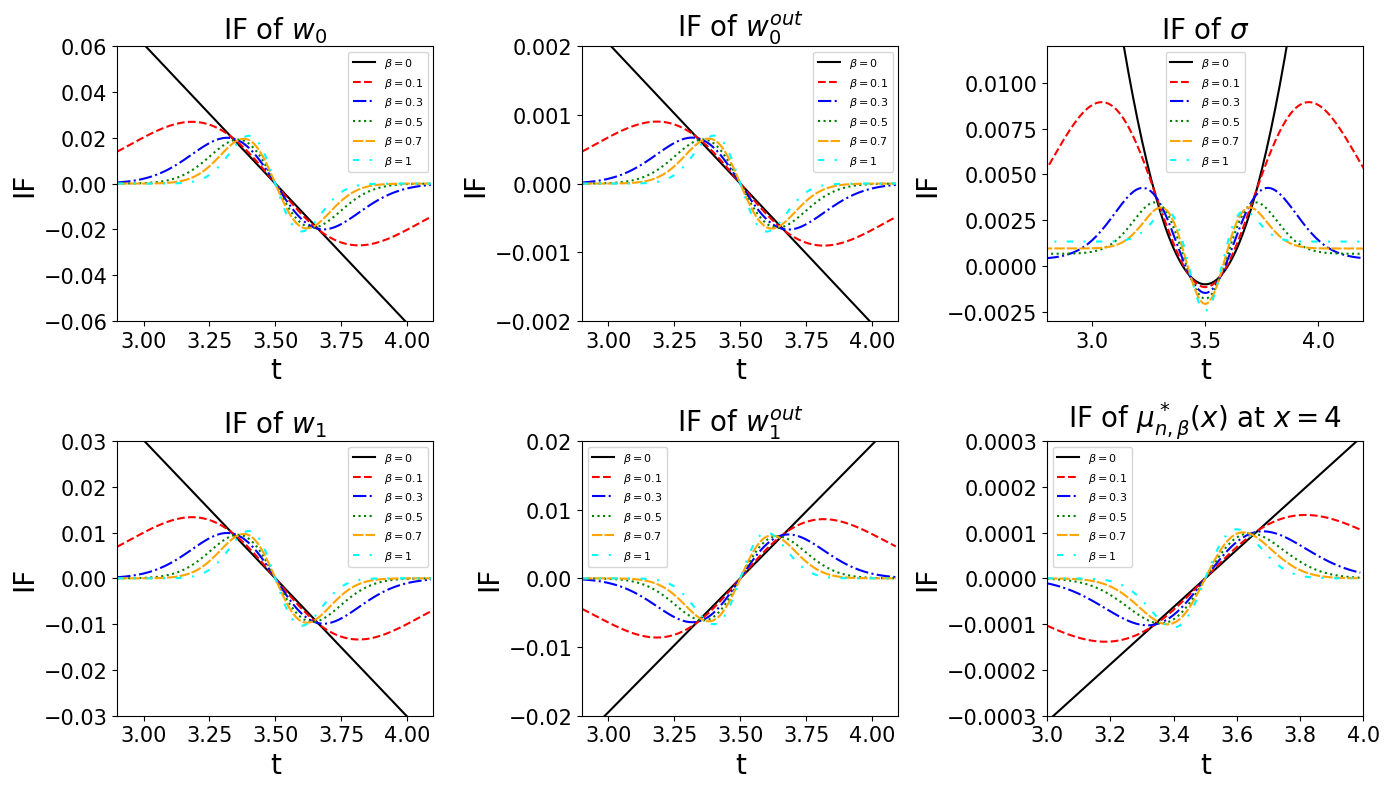}
	\caption{IFs of the MDPDFs and the rRNet predictors for a simple MLP, with sigmoid activation and Gaussian error, 
		under contamination in the 49-th observation [The case $\beta=0$ represents the standard LSE based training]}
	\label{fig:IF-NN-49}
\end{figure}

\begin{figure}
	\centering
	\includegraphics[width=\textwidth]{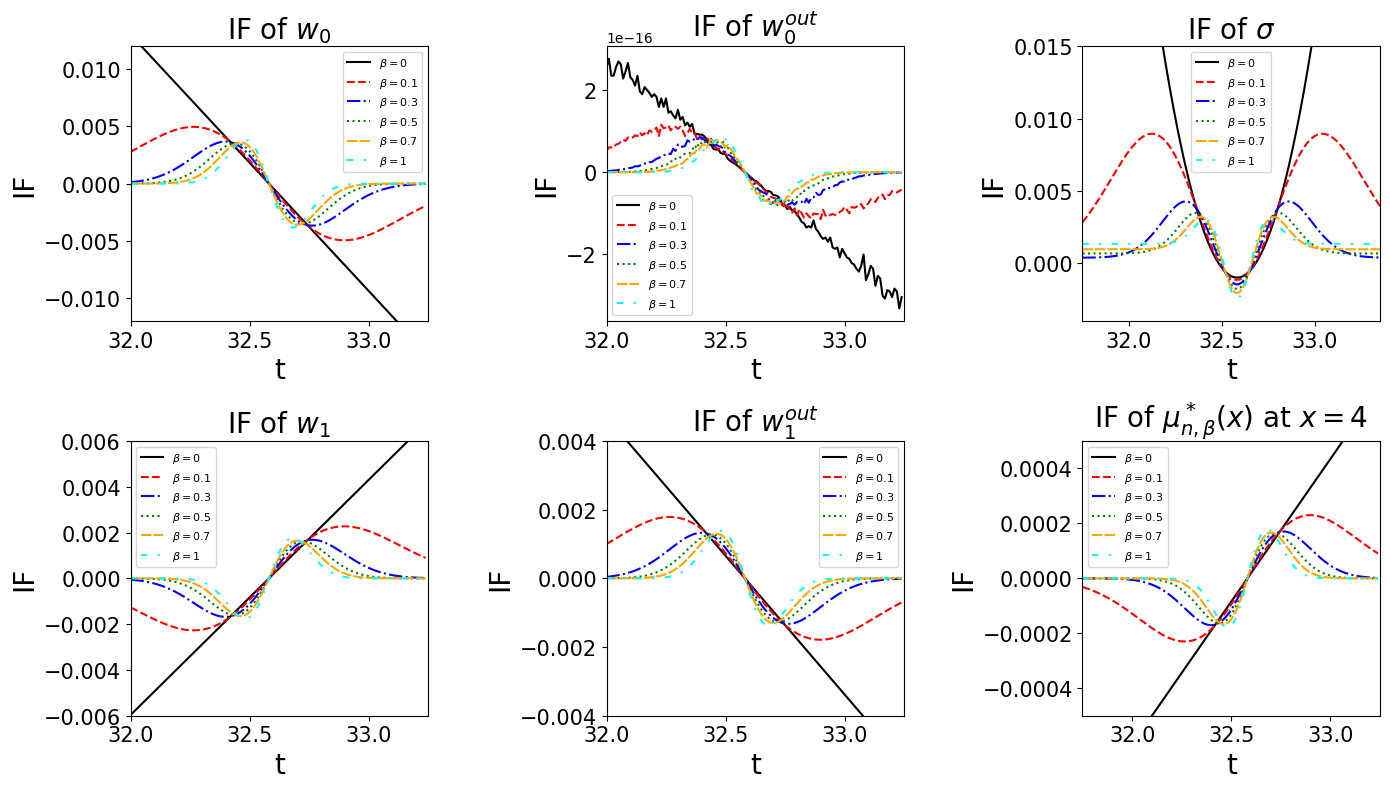}
	\caption{IFs of the MDPDFs and the rRNet predictors for a simple MLP, with ReLU activation and Gaussian error, 
		under contamination in the 49-th observation [The case $\beta=0$ represents the standard LSE based training]}
	\label{fig:IF-relu-10-49}
\end{figure}

\begin{table}
	\centering
	\caption{Average train TMSE and test MSE obtained while approximating the function $\varphi_1$ based on sampled data with contamination proportion $\delta$ (the minimum is highlighted in bold font)}
	\begin{tabular}{l|rrrr|rrrr}
		\hline
		NN training & \multicolumn{4}{c}{Avg TMSE (on training data)} & \multicolumn{4}{|c}{Avg MSE (on test data)} \\
		\multicolumn{1}{c|}{methods } & $\delta \rightarrow~~$0\% & 10\% & 20\% & 30\% & $\delta \rightarrow~~$0\% & 10\% & 20\% & 30\% \\ \hline
		LSE &\textbf{ 0.0114} & 0.0530 & 0.1626 & 0.3178 & \textbf{0.0118} & 0.0562 & 0.1810 & 0.3872 \\[.3em]
		\multicolumn{9}{l}{\underline{Proposed rRNet: DPD-loss with tuning parameter $\beta$}}\\
		$\beta$ = 0.1 & 0.0115 & \textbf{0.0114} & 0.0899 & 0.2425 & \textbf{0.0118} & \textbf{0.0121} & 0.0994 & 0.2914 \\
		$\beta$ = 0.3 & 0.0117 & 0.0120 & \textbf{0.0118} & 0.0191 & 0.0121 & 0.0126 & \textbf{0.0128} & 0.0218 \\
		$\beta$ = 0.5 & 0.0124 & 0.0122 & 0.0119 & \textbf{0.0115} & 0.0127 & 0.0128 & \textbf{0.0128} & \textbf{0.0128} \\
		$\beta$ = 0.7 & 0.0125 & 0.0122 & 0.0119 & 0.0117 & 0.0127 & 0.0128 & \textbf{0.0128} & 0.0130 \\
		$\beta$ = 1 & 0.0129 & 0.0130 & 0.0129 & 0.0129 & 0.0131 & 0.0136 & 0.0140 & 0.0145 \\
		\multicolumn{9}{l}{\underline{Existing robust losses}}\\ 
		MAE & 0.0121 & 0.0121 & 0.0128 & 0.0148 & 0.0124 & 0.0128 & 0.0140 & 0.0170 \\
		LTA & 0.0121 & 0.0121 & 0.0128 & 0.0149 & 0.0124 & 0.0128 & 0.0140 & 0.0171 \\
		LTS & 0.\textbf{0114} & 0.0530 & 0.1627 & 0.3181 & 0.0118 & 0.0561 & 0.1811 & 0.3876 \\
		LMLS & 0.0115 & 0.0162 & 0.0330 & 0.0727 & 0.0118 & 0.0172 & 0.0363 & 0.0841 \\ 
		Huber's M & 0.0117 & 0.0120 & 0.0154 & 0.0328 & 0.0120 & 0.0128 & 0.0169 & 0.0376 \\
		Tukey's M & 0.0126 & 0.0122 & 0.0119 & 0.0124 & 0.0128 & 0.0128 & \textbf{0.0128} & 0.0139 \\
		\hline
	\end{tabular}
	\label{tab:f1}
\end{table}

\begin{table}
	\centering
	\caption{Average train TMSE and test MSE obtained while approximating the function $\varphi_2$ based on sampled data with contamination proportion $\delta$ (the minimum is highlighted in bold font)}
	\begin{tabular}{l|rrrr|rrrr}
		\hline
		NN training & \multicolumn{4}{c}{Avg TMSE (on training data)} & \multicolumn{4}{|c}{Avg MSE (on test data)} \\
		\multicolumn{1}{c|}{methods } & $\delta \rightarrow~~$0\% & 10\% & 20\% & 30\% & $\delta \rightarrow~~$0\% & 10\% & 20\% & 30\% \\ \hline
		LSE & 0.0095 & 0.0729 & 0.1791 & 0.3120 & \textbf{0.0108} & 0.0903 & 0.2404 & 0.4632 \\[.3em]
		\multicolumn{9}{l}{\underline{Proposed rRNet: DPD-loss with tuning parameter $\beta$}}\\
		$\beta=0.1$ & \textbf{0.0093} & \textbf{0.0089} & 0.0531 & 0.1916 & 0.0109 & \textbf{0.0110} & 0.0726 & 0.2935 \\
		$\beta=0.3$ & 0.0094 & \textbf{0.0089} & \textbf{0.0085} & \textbf{0.0080} & 0.0109 & \textbf{0.0110} & \textbf{0.0112} & \textbf{0.0114} \\
		$\beta=0.5$ & 0.0095 & 0.0090 & \textbf{0.0085} & \textbf{0.0080} & 0.0110 & 0.0111 & 0.0113 & 0.0115 \\
		$\beta=0.7$ & 0.0096 & 0.0091 & 0.0086 & \textbf{0.0080} & 0.0111 & 0.0112 & 0.0114 & 0.0115 \\
		$\beta=1$ & 0.0099 & 0.0093 & 0.0087 & 0.0081 & 0.0114 & 0.0115 & 0.0115 & 0.0116 \\
		\multicolumn{9}{l}{\underline{Existing robust losses}}\\
		MAE & 0.0204 & 0.0195 & 0.0196 & 0.0212 & 0.0209 & 0.0214 & 0.0229 & 0.0270 \\
		LTA & 0.0098 & 0.0096 & 0.0097 & 0.0108 & 0.0113 & 0.0117 & 0.0126 & 0.0161 \\
		LTS & 0.0095 & 0.0729 & 0.1791 & 0.3118 & \textbf{0.0108} & 0.0903 & 0.2404 & 0.4629 \\
		LMLS & 0.0094 & 0.0131 & 0.0219 & 0.0383 & \textbf{0.0108} & 0.0157 & 0.0285 & 0.0584 \\ 
		Huber's M & 0.0094 & 0.0102 & 0.0174 & 0.0416 & 0.0109 & 0.0124 & 0.0224 & 0.0621 \\
		Tukey's M & 0.0202 & 0.0090 & 0.0094 & 0.0134 & 0.0217 & 0.0111 & 0.0122 & 0.0193 \\
		\hline
	\end{tabular}
	\label{tab:f2}
\end{table} 

\begin{table}
	\centering
	\caption{Average train TMSE and test MSE obtained while approximating the function $\varphi_4$ based on sampled data with contamination proportion $\delta$ (the minimum is highlighted in bold font)}
	\begin{tabular}{l|rrrr|rrrr}
		\hline
		NN training & \multicolumn{4}{c}{Avg TMSE (on training data)} & \multicolumn{4}{|c}{Avg MSE (on test data)} \\
		\multicolumn{1}{c|}{methods } & $\delta \rightarrow~~$0\% & 10\% & 20\% & 30\% & $\delta \rightarrow~~$0\% & 10\% & 20\% & 30\% \\ \hline
		LSE & 0.0156 & 0.0649 & 0.1717 & 0.3168 & 0.0169 & 0.0725 & 0.2016 & 0.4019 \\[.3em]
		\multicolumn{9}{l}{\underline{Proposed rRNet: DPD-loss with tuning parameter $\beta$}}\\
		$\beta$ = 0.1 & \textbf{0.0121} & \textbf{0.0119} & 0.0571 & 0.1902 & \textbf{0.0130} & \textbf{0.0136} & 0.0681 & 0.2404 \\
		$\beta$ = 0.3 & 0.0124 & 0.0120 & \textbf{0.0118} & \textbf{0.0120} & 0.0134 & 0.0138 & \textbf{0.0145} & \textbf{0.0160} \\
		$\beta$ = 0.5 & 0.0132 & 0.0127 & 0.0123 & 0.0121 & 0.0143 & 0.0147 & 0.0154 & 0.0163 \\
		$\beta$ = 0.7 & 0.0145 & 0.0138 & 0.0132 & 0.0126 & 0.0158 & 0.0161 & 0.0166 & 0.0172 \\
		$\beta$ = 1 & 0.0172 & 0.0157 & 0.0145 & 0.0134 & 0.0188 & 0.0185 & 0.0184 & 0.0185 \\
		\multicolumn{9}{l}{\underline{Existing robust losses}}\\
		MAE & 0.0158 & 0.0161 & 0.0172 & 0.0190 & 0.0171 & 0.0190 & 0.0219 & 0.0264 \\
		LTA & 0.0161 & 0.0167 & 0.0179 & 0.0195 & 0.0176 & 0.0197 & 0.0227 & 0.0272 \\
		LTS & 0.0156 & 0.0648 & 0.1717 & 0.3171 & 0.0169 & 0.0724 & 0.2016 & 0.4023 \\
		LMLS & 0.0139 & 0.0209 & 0.0283 & 0.0410 & 0.0151 & 0.0241 & 0.0348 & 0.0539 \\ 
		Huber's M & 0.0132 & 0.0149 & 0.0211 & 0.0337 & 0.0143 & 0.0174 & 0.0262 & 0.0447 \\
		Tukey's M & 0.0128 & 0.0126 & 0.0137 & 0.0164 & 0.0139 & 0.0146 & 0.0170 & 0.0218 \\
		\hline
	\end{tabular}
	\label{tab:f3}
\end{table}

\begin{table}
	\centering
	\caption{Average train TMSE and test MSE obtained while approximating the function $\varphi_7$ based on sampled data with contamination proportion $\delta$ (the minimum is highlighted in bold font)}
	\begin{tabular}{l|rrrr|rrrr}
		\hline
		NN training & \multicolumn{4}{c}{Avg TMSE (on training data)} & \multicolumn{4}{|c}{Avg MSE (on test data)} \\
		\multicolumn{1}{c|}{methods } & $\delta \rightarrow~~$0\% & 10\% & 20\% & 30\% & $\delta \rightarrow~~$0\% & 10\% & 20\% & 30\% \\ \hline
        LSE & \textbf{0.023} & 0.033 & 0.040 & 0.041 & 2.778 & 8.292 & 14.865 & 23.305 \\
\multicolumn{9}{l}{\underline{DPD loss with tuning parameter $\beta$}}\\
$\beta = 0.1$ & 0.105 & \textbf{0.025} & 0.031 & 0.043 & 2.552 & 3.121 & 9.856 & 19.205 \\
$\beta = 0.3$ & 0.590 & 0.296 & 0.174 & 0.138 & 2.156 & \textbf{2.293} & \textbf{2.602} & \textbf{3.098} \\
$\beta = 0.5$ & 1.196 & 0.817 & 0.904 & 0.526 & 2.377 & 2.615 & 3.437 & 4.002 \\
$\beta = 0.7$ & 1.443 & 1.248 & 1.032 & 0.695 & 2.559 & 2.999 & 3.478 & 3.928 \\
$\beta = 1  $ & 1.705 & 1.503 & 1.157 & 1.299 & 2.834 & 3.237 & 3.486 & 4.832 \\
\multicolumn{9}{l}{\underline{Existing robust losses}}\\
MAE & 0.152 & 0.037 & \textbf{0.027} & \textbf{0.025} & \textbf{2.124} & 4.441 & 8.679 & 14.414 \\
LTA & 0.152 & 0.037 & \textbf{0.027} & \textbf{0.025} & \textbf{2.124} & 4.441 & 8.679 & 14.414 \\
LTS & \textbf{0.023} & 0.033 & 0.040 & 0.041 & 2.778 & 8.292 & 14.865 & 23.305 \\
LMLS & 0.029 & 0.026 & 0.029 & 0.031 & 2.735 & 6.178 & 10.466 & 15.249 \\
Huber's M & 0.118 & 0.029 & 0.029 & 0.028 & 2.237 & 5.314 & 10.366 & 16.593 \\
Tukey's M & 13.681 & 9.775 & 7.404 & 4.851 & 13.674 & 10.876 & 9.290 & 7.170 \\ \hline
	\end{tabular}
	\label{tab:f7}
\end{table}

\begin{figure}
     \centering
     \begin{subfigure}[b]{0.32\textwidth}
         \centering
         \includegraphics[width=\textwidth]{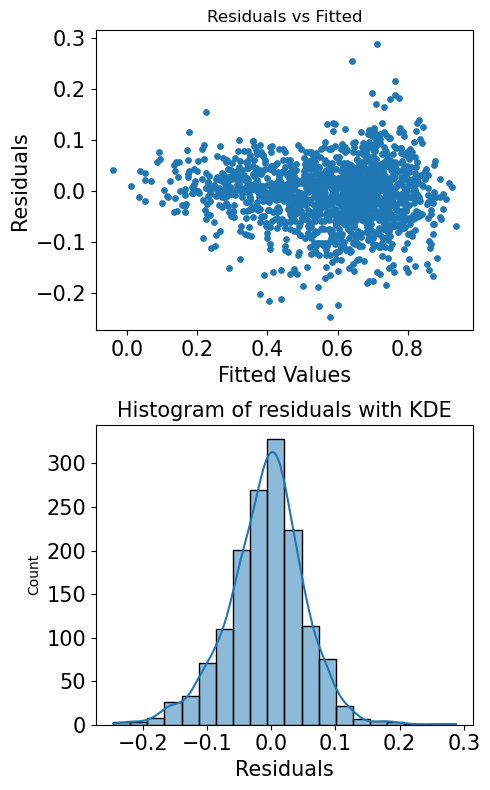}
         \caption{Standard LSE-based training}
         \label{fig:Airfoil-LSE}
     \end{subfigure}
%     \hfill
     \begin{subfigure}[b]{0.32\textwidth}
         \centering
         \includegraphics[width=\textwidth]{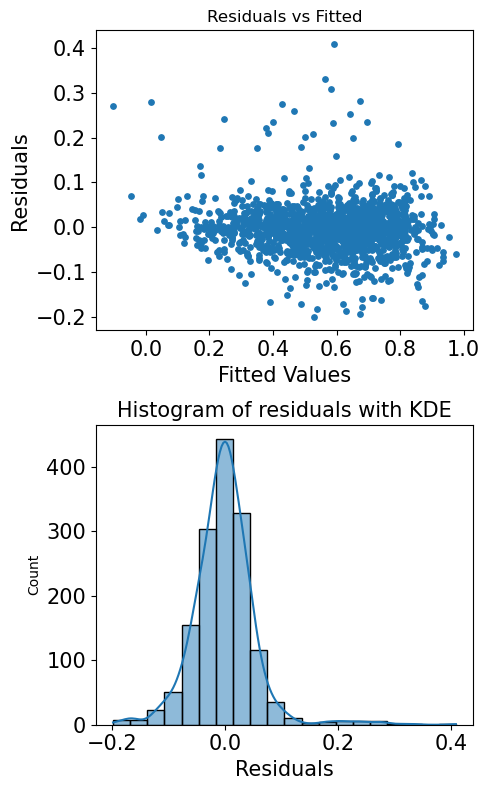}
         \caption{rRNet training with $\beta = 0.3$}
         \label{fig:Airfoil-DPD-0.3}
     \end{subfigure}
%     \hfill
     \begin{subfigure}[b]{0.32\textwidth}
         \centering
         \includegraphics[width=\textwidth]{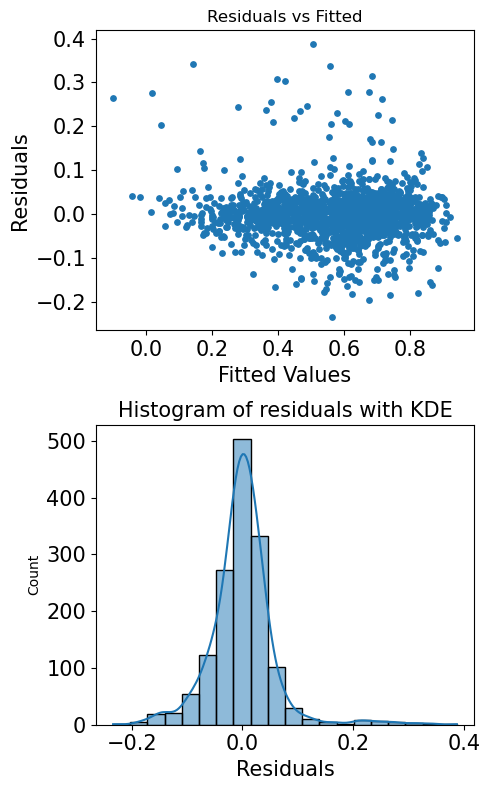}
         \caption{rRNet training with $\beta = 0.5$}
         \label{fig:Airfoil-DPD-0.5}
     \end{subfigure}
     \caption{Residual plots and histograms of the residuals obtained by different NN learning algorithms for the Airfoil Self-noise data}
     \label{fig:Airfoil}
\end{figure}

\begin{figure}
     \centering
     \begin{subfigure}[b]{0.32\textwidth}
         \centering
         \includegraphics[width=\textwidth]{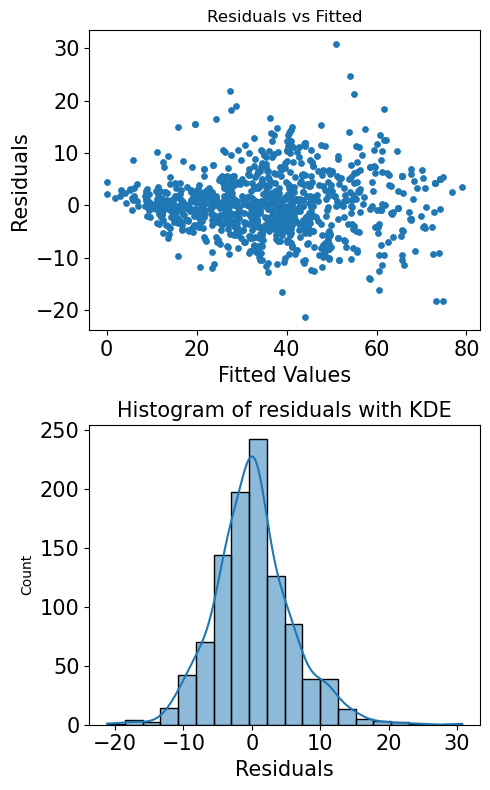}
         \caption{Standard LSE-based training}
         \label{fig:Concrete-LSE}
     \end{subfigure}
     \hfill
     \begin{subfigure}[b]{0.32\textwidth}
         \centering
         \includegraphics[width=\textwidth]{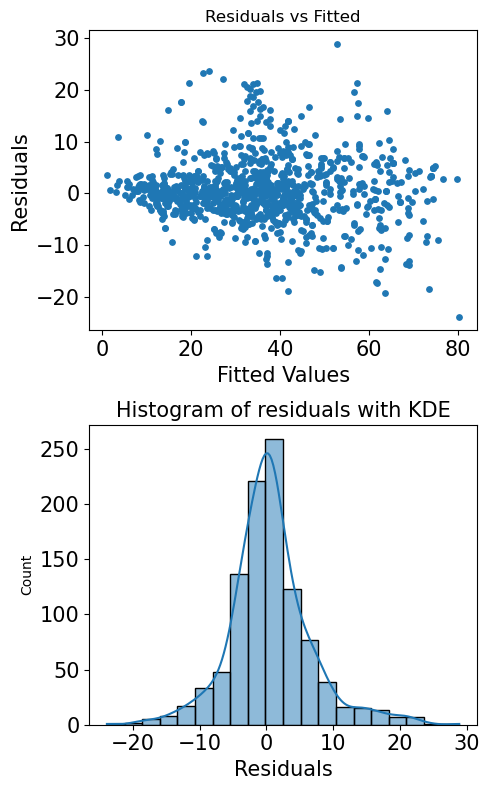}
         \caption{rRNet training with $\beta = 0.3$}
         \label{fig:Concrete-DPD-0.3}
     \end{subfigure}
     \hfill
     \begin{subfigure}[b]{0.32\textwidth}
         \centering
         \includegraphics[width=\textwidth]{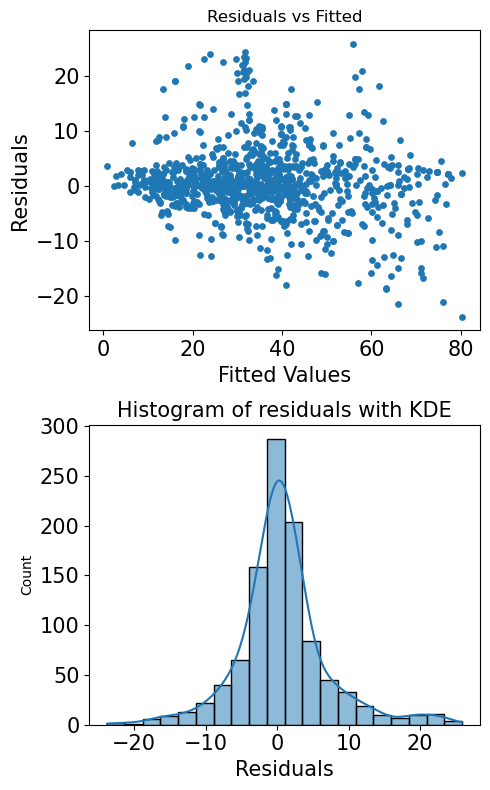}
         \caption{rRNet training with $\beta = 0.5$}
         \label{fig:Concrete-DPD-0.5}
     \end{subfigure}
     \caption{Residual plots and histograms of the residuals obtained by different NN learning algorithms for the Concrete Compressive Strength data}
     \label{fig:Concrete}
\end{figure}

\begin{figure}
     \centering
     \begin{subfigure}[b]{0.32\textwidth}
         \centering
         \includegraphics[width=\textwidth]{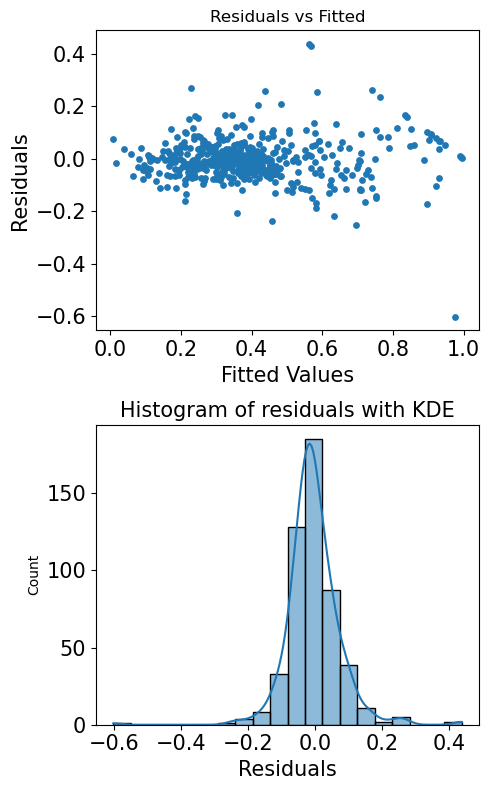}
         \caption{Standard LSE-based training}
         \label{fig:Boston-medv-LSE}
     \end{subfigure}
     \hfill
     \begin{subfigure}[b]{0.32\textwidth}
         \centering
         \includegraphics[width=\textwidth]{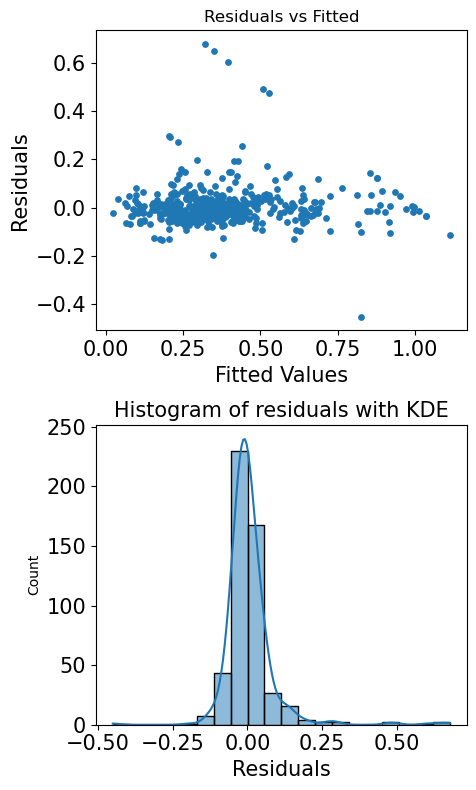}
         \caption{rRNet training with $\beta = 0.3$}
         \label{fig:Boston-medv-DPD-0.1}
     \end{subfigure}
     \hfill
     \begin{subfigure}[b]{0.32\textwidth}
         \centering
         \includegraphics[width=\textwidth]{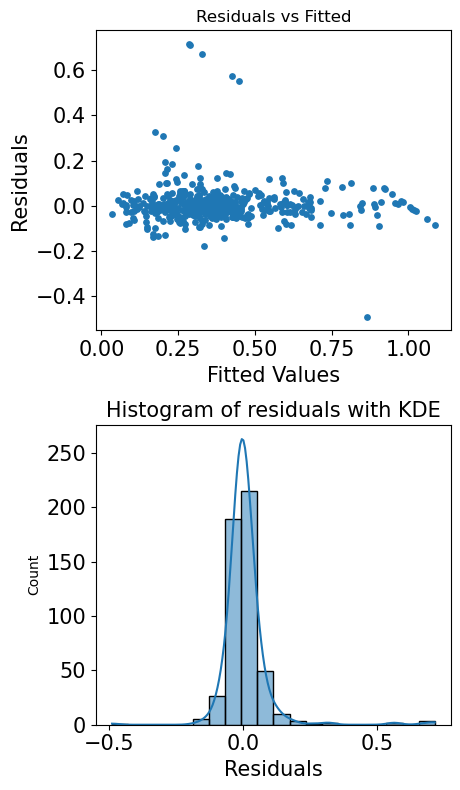}
         \caption{rRNet training with $\beta = 0.5$}
         \label{fig:Boston-medv-DPD-0.3}
     \end{subfigure}
     \caption{Residual plots and histograms of the residuals obtained by different NN learning algorithms for the Boston Housing data}
     \label{fig:Boston-medv}
\end{figure}

\clearpage
\newpage
\bibliography{Reference.bib}

\end{document}